\documentclass{article}

\usepackage{iclr2026_conference,times}

\iclrfinalcopy

\usepackage{amsmath,amssymb,amsthm,mathtools}
\usepackage{booktabs,multirow,array,tabularx}
\usepackage{graphicx}
\usepackage{microtype}
\usepackage{url}
\usepackage{hyperref}
\hypersetup{hidelinks,hypertexnames=false}
\usepackage{cleveref}
\usepackage{xcolor}
\usepackage{tikz}
\usetikzlibrary{arrows.meta,positioning,fit,calc,shapes.geometric,backgrounds}
\usepackage{enumitem}
\usepackage{algorithm}
\usepackage[noend]{algpseudocode}

\definecolor{methodblue}{RGB}{38,91,154}
\definecolor{targetgreen}{RGB}{52,125,82}
\definecolor{decodeorange}{RGB}{185,105,35}
\definecolor{softgray}{RGB}{242,244,247}

\newcommand{\G}{\mathcal{G}}
\newcommand{\Hh}{\mathcal{H}}
\newcommand{\Kset}{\mathcal{K}}

\newcommand{\Iset}{\mathcal{I}}
\newcommand{\Mset}{\mathcal{M}}
\newcommand{\R}{\mathbb{R}}
\newcommand{\SE}{\mathrm{SE}}
\newcommand{\SO}{\mathrm{SO}}
\newcommand{\stopgrad}{\operatorname{stopgrad}}
\newcommand{\cosdist}{d_{\mathrm{cos}}}
\newcommand{\dhier}{D_{\mathrm{hier}}}
\newcommand{\dspd}{d_{\mathrm{SPD}}}

\newcommand{\Exp}{\operatorname{Exp}}

\newcommand{\method}{\textsc{4DGS-JEPA}}

\newtheorem{proposition}{Proposition}

\title{4DGS-JEPA: Temporally Compositional Joint-Embedding Prediction for Dynamic Gaussian Splatting}
\author{
Yongchao Huang\thanks{Email: \texttt{yongchao.huang@abdn.ac.uk}}\\
22/07/2026
}

\begin{document}
\maketitle

\fancyhead{}
\renewcommand{\headrulewidth}{0pt}

\begin{abstract}
Dynamic Gaussian Splatting provides an explicit representation of evolving 3D scenes, but existing approaches are primarily optimized for reconstruction, future-state generation, or rendering rather than for learning reusable predictive dynamics. We propose 4DGS-JEPA, a Gaussian-native joint-embedding predictive architecture for causal multi-horizon prediction over dynamic Gaussian scenes. The model uses a hierarchical scene-, motion-group-, and Gaussian-level representation together with a horizon-conditioned transition operator that supports both direct prediction and recursive rollout. Its central principle is \emph{temporal composition}: different chronological transition paths reaching the same future endpoint should produce compatible predictive states. Endpoint and multi-horizon path supervision anchor these predictions to future target embeddings, while a selective geometry decoder and geometry-level composition ground the learned dynamics in consistent group motion and Gaussian geometry without requiring complete future appearance reconstruction. We further introduce a hybrid correspondence mechanism that combines persistent canonical identity with residual optimal-transport matching under reordering and topology change. We characterize zero-loss path agreement and finite-error rollout accumulation theoretically. Three controlled experiments provide mechanism-level evidence that temporal composition reduces latent path dependence while retaining predictive accuracy, geometry-level composition improves consistency of decoded motion, and hybrid correspondence preserves reliable identity while remaining robust when correspondence becomes ambiguous. Together, 4DGS-JEPA provides a predictive, temporally compositional formulation of dynamic Gaussian worlds.
\end{abstract}

\section{Introduction}

Three-dimensional Gaussian Splatting (3DGS) has become a powerful explicit scene representation because it combines differentiable optimization, high-quality novel-view synthesis, and real-time rendering \citep{kerbl2023gaussians}. Dynamic extensions, including 4D Gaussian Splatting (4D-GS), model time-dependent deformation and appearance in evolving scenes \citep{wu2024fourDGS}. These methods are highly effective as renderable scene representations, but they are primarily optimized to reconstruct or interpolate observations from a scene sequence. Accurate reconstruction at arbitrary timestamps therefore does not necessarily provide a reusable predictive state whose dynamics can be rolled forward beyond the observations used to construct the representation.

Joint-Embedding Predictive Architectures (JEPAs) offer a complementary principle: predict the representation of a missing or future target rather than reconstructing its raw observation \citep{lecun2022path,assran2023ijepa,bardes2024vjepa}. V-JEPA 2 extends this idea to large-scale video understanding and action-conditioned robotic planning, but its predictive state remains an image- or video-derived latent rather than an explicit 3D scene representation \citep{assran2025vjepa2}. This motivates a direct question: \emph{can a JEPA learn the dynamics of an explicit Gaussian world while retaining the abstraction benefits of latent prediction?}

We answer this question with 4DGS-JEPA, a Gaussian-native JEPA for causal multi-horizon prediction over dynamic Gaussian scenes. Given only a causal scene prefix, the model encodes time-indexed 3D Gaussian primitives into a hierarchical predictive state with scene-, motion-group-, and Gaussian-level tokens. A horizon-conditioned transition operator $\Phi_\phi$ maps this state back into the same predictive state space, allowing both direct prediction and recursive rollout.

The central principle is \emph{temporal composition}. A future state reached directly over a long horizon should agree with the state obtained by composing shorter chronological transitions. In the passive two-step case, this principle takes the form
\[
\Phi_\phi\!\left(Z_t,\delta_1+\delta_2\right)
\approx
\Phi_\phi\!\left(
\Phi_\phi(Z_t,\delta_1),
\delta_2
\right).
\label{eq:intro-composition}
\]
More generally, the model enforces consistency across admissible temporal decompositions while preserving the corresponding chronological action sequence when actions are present. Endpoint supervision anchors direct predictions to future target embeddings, multi-horizon path supervision anchors recurrent predictions along a rollout, and temporal composition explicitly suppresses residual dependence on the path used to reach the same future endpoint.

4DGS-JEPA also connects latent prediction back to explicit Gaussian dynamics without making complete future-scene reconstruction its principal objective. A \emph{selective} geometry decoder predicts only the Gaussian variables needed for geometric evolution, including group-level $\SE(3)$ motion, Gaussian-specific residual displacement, covariance deformation, and, in the full formulation, existence dynamics. Geometry-level composition further requires direct and recursively decoded motions to agree at both the motion-group and Gaussian levels. Because Gaussian representations are sets whose raw array ordering does not define identity, we additionally introduce a hybrid correspondence mechanism. Reliable persistent canonical identities are preserved when available, while residual optimal transport handles ambiguous correspondence caused by reordering or topology change.

This formulation differs from recent Gaussian world models in both its prediction target and its treatment of temporal structure. GWM combines a Gaussian VAE with a latent diffusion transformer to reconstruct action-conditioned future Gaussian scenes \citep{lu2025gwm}. GAF couples motion-aware Gaussian fields with current reconstruction, future-frame prediction, and action inference \citep{chai2025gaf}. MoGaF performs long-horizon 4D Gaussian forecasting using motion-aware grouping \citep{lee2026mogaf}. GaussianDream uses current and future Gaussian decoding as auxiliary supervision for robotic manipulation \citep{zhang2026gaussiandream}, while a Structured 4D Latent Predictive Model predicts future 3D structure in a general latent space that can be decoded to multiple 3D formats \citep{li2026structured4d}. In contrast, 4DGS-JEPA uses joint-embedding prediction rather than complete future-scene generation as its principal learning objective, operates natively on Gaussian scene states, and makes consistency across alternative temporal prediction paths an explicit learning principle.

Our contributions are:
\begin{enumerate}[leftmargin=*,itemsep=1pt]
    \item a Gaussian-native hierarchical JEPA for causal multi-horizon prediction over dynamic Gaussian scenes, with scene-, motion-group-, and Gaussian-level predictive states;
    \item a temporally compositional, horizon-conditioned transition framework combining endpoint supervision, recurrent multi-horizon path supervision, and explicit direct--composed consistency across temporal decompositions;
    \item a geometry-aware selective decoder and geometry-level composition objective that ground latent dynamics in group-level $\SE(3)$ motion and Gaussian-level deformation without requiring complete future appearance reconstruction;
    \item a hybrid Gaussian correspondence mechanism that preserves reliable persistent canonical identities and applies residual optimal transport only when correspondence becomes ambiguous under reordering or topology change;
    \item a zero-loss path-agreement result and a finite-error rollout bound, together with controlled experiments that separately evaluate latent temporal composition, decoded geometry composition, and hybrid correspondence.
\end{enumerate}

\section{Related Work}

\Cref{tab:related} summarizes representative dynamic Gaussian scene representations, Gaussian world models, and joint-embedding predictive models. The key distinction of 4DGS-JEPA is that it combines a Gaussian-native predictive state with joint-embedding supervision and explicit temporal path consistency, without making complete future-scene reconstruction the principal learning objective.

\subsection{Dynamic Gaussian scene representations}

3D Gaussian Splatting (3DGS) represents a static scene using a set of anisotropic Gaussian primitives,
\begin{equation}
\mathcal G
=
\{g_i\}_{i=1}^{N},
\qquad
g_i
=
(\mu_i,\Sigma_i,\alpha_i,c_i),
\end{equation}
where $\mu_i$ and $\Sigma_i$ denote the center and covariance of Gaussian $i$, while $\alpha_i$ and $c_i$ denote its opacity and spherical-harmonic appearance coefficients. These parameters are optimized through differentiable rasterization for high-quality novel-view synthesis \citep{kerbl2023gaussians}. The resulting representation describes a static scene but does not itself specify how that scene evolves over time.

Dynamic Gaussian representations extend this formulation to a time-indexed scene,
\begin{equation*}
\mathcal G_t
=
\{g_{i,t}\}_{i=1}^{N_t},
\end{equation*}
in which Gaussian positions, covariances, opacities, appearances, or existence states may evolve. For example, 4D-GS augments 3DGS with spatiotemporal features and a deformation predictor for reconstruction and rendering at different timestamps \citep{wu2024fourDGS}. Related dynamic neural scene representations, including D-NeRF and HyperNeRF, similarly use canonical-space deformation to model non-rigid motion and topology change \citep{pumarola2021dnerf,park2021hypernerf}.

Many such methods primarily optimize reconstruction or interpolation over an observed sequence: time indexes a scene representation fitted using observations from that sequence. 4DGS-JEPA addresses a different problem. Given only a causal Gaussian scene prefix $\mathcal G_{\leq t}$, it learns a predictive state $Z_t$ and forecasts its future representation,
\begin{equation}
\mathcal G_{\leq t}
\longrightarrow
Z_t
\longrightarrow
\widehat Z_{t+\Delta},
\end{equation}
without requiring reconstruction of every future pixel or Gaussian attribute. Thus, 3DGS represents a static scene, dynamic Gaussian methods represent time-varying scenes, and 4DGS-JEPA focuses on learning reusable dynamics for causal future prediction.

\subsection{Gaussian world models and forecasting}

Recent Gaussian world models move beyond time-conditioned reconstruction towards explicit future prediction. GWM reconstructs action-conditioned future Gaussian states using a 3D VAE and latent diffusion transformer and applies the learned model to imitation learning and model-based reinforcement learning \citep{lu2025gwm}. GAF augments Gaussians with motion attributes and jointly supports current-scene reconstruction, future prediction, and robot-action inference \citep{chai2025gaf}. MoGaF introduces motion-aware Gaussian grouping and forecasts long-horizon dynamic-scene evolution \citep{lee2026mogaf}. GaussianDream trains a feed-forward world-model representation using current and future Gaussian decoding as supervision and discards the decoders at test time for manipulation \citep{zhang2026gaussiandream}.

These approaches demonstrate the value of explicit 3D future prediction, but they rely substantially on explicit future-state generation or reconstruction supervision. 4DGS-JEPA instead makes prediction of future target embeddings its principal learning objective. Selective geometry decoding provides structured geometric grounding without requiring reconstruction of the complete future Gaussian appearance.

The Structured 4D Latent Predictive Model is particularly close in spirit: it forecasts future 3D scene structure in a general latent space conditioned on observations and language and decodes that representation into point clouds or Gaussians for planning \citep{li2026structured4d}. However, its predictive representation is not Gaussian-native and it does not explicitly constrain alternative temporal prediction paths to agree. In contrast, 4DGS-JEPA learns a Gaussian-native predictive representation and explicitly regularizes both latent and decoded dynamics to be consistent across alternative chronological decompositions of the same future evolution.

\subsection{Joint-embedding predictive learning}

JEPA proposes learning representations by predicting compatible target embeddings from context embeddings rather than reconstructing observations \citep{lecun2022path}. I-JEPA instantiates this principle for masked image regions \citep{assran2023ijepa}, while V-JEPA extends feature prediction to video without pixel reconstruction \citep{bardes2024vjepa}. V-JEPA 2 combines large-scale video pretraining with an action-conditioned latent dynamics model for robotic planning \citep{assran2025vjepa2}, and V-JEPA 2.1 further emphasizes dense, spatially grounded predictive features \citep{murlabadia2026vjepa21}.

Variational JEPA (VJEPA) extends deterministic joint-embedding prediction to a predictive distribution over future latent states, enabling uncertainty-aware world modelling without requiring an observation-space likelihood \citep{huang2026vjepa}. 4DGS-JEPA is complementary to this probabilistic direction. Its focus is instead on transferring joint-embedding predictive learning to explicit Gaussian scene states and enforcing temporal composition across direct and recursively composed future predictions.

\begin{table}[t]
\caption{Positioning relative to representative prior work. ``Full reconstruction'' indicates that decoding a complete future appearance or renderable state is central to the training objective. ``Path comp.'' denotes explicit consistency between alternative temporal prediction paths.}
\label{tab:related}
\centering
\small
\setlength{\tabcolsep}{3.1pt}
\begin{tabular}{lccccc}
\toprule
Method & Native state & Future target & JEPA & Full recon. & Path comp. \\
\midrule
4D-GS \citep{wu2024fourDGS}
& Gaussians & time query & no & yes & no \\
GWM \citep{lu2025gwm}
& Gaussians & future Gaussians & no & yes & no \\
GAF \citep{chai2025gaf}
& Gaussians & frame/action & no & yes & no \\
MoGaF \citep{lee2026mogaf}
& Gaussians & future Gaussians & no & yes & no \\
Structured 4D \citep{li2026structured4d}
& 3D latent & future 3D & no & yes & no \\
V-JEPA 2 \citep{assran2025vjepa2}
& video latent & future latent & yes & no & no \\
4DGS-JEPA
& Gaussian-native latent & latent/geometry & yes & no & yes \\
\bottomrule
\end{tabular}
\end{table}

\section{Method}
\label{sec:method}

This section develops 4DGS-JEPA from causal Gaussian observations to the complete learning objective. To learn future dynamics without leaking information from later observations, we first formalize the \textit{causal prediction setting} and construct context and target Gaussian histories under a \textit{forward-only protocol}. To represent scene evolution at multiple spatial scales, we then introduce a \textit{hierarchical student--teacher predictive state} together with a \textit{horizon-conditioned transition operator} that supports both \textit{direct prediction} and \textit{recursive rollout}. To make these predictions accurate and insensitive to how a future state is reached, we define a \textit{hierarchical latent distance} and combine \textit{endpoint supervision}, \textit{recurrent multi-horizon path supervision}, and \textit{temporal composition consistency}. We further ground the learned latent dynamics in explicit 3D motion through a \textit{selective geometry decoder}, and impose \textit{geometry-level composition} and \textit{structural priors} to encourage coherent Gaussian evolution. Finally, we introduce \textit{representation regularization} to prevent degenerate latent solutions and assemble all components into the complete training objective. An overview of the resulting architecture and information flow is shown in \cref{fig:overview}.

\subsection{Problem setup and causal Gaussian states}
\label{sec:problem-setup}

We consider \emph{causal multi-horizon prediction in dynamic Gaussian scenes}. Given Gaussian scene observations available only up to a context endpoint $t$, and optionally a specified future action sequence, the goal is to learn a predictive state from which representations of the scene at future horizons $\Delta>0$ can be predicted. The prediction branch must not access observations after $t$; future observations are used only to construct training targets. The objective is therefore to predict the future \emph{representation} of the Gaussian world rather than reconstruct its complete future appearance.

Let the dynamic scene at time $t$ be represented by
\begin{equation}
\G_t
=
\{g_{i,t}\}_{i=1}^{N_t},
\qquad
g_{i,t}
=
(\mu_{i,t},\Sigma_{i,t},\alpha_{i,t},c_{i,t},f_{i,t},e_{i,t}),
\label{eq:gaussian-state}
\end{equation}
where the center $\mu_{i,t}\in\R^3$ and covariance $\Sigma_{i,t}\in\mathbb{S}_{++}^{3}$ describe the geometry of Gaussian $i$, while $\alpha_{i,t}$ is its opacity, $c_{i,t}$ contains appearance coefficients, $f_{i,t}$ is an optional semantic feature, and $e_{i,t}\in[0,1]$ is its existence variable.

An important requirement is that the Gaussian state available to the predictor be constructed \textit{causally}. For a context endpoint $t$, an online Gaussian front-end $B_{\mathrm{on}}$ may use only observations up to $t$:
\begin{equation}
\G_{1:t}
=
B_{\mathrm{on}}(X_{1:t}).
\label{eq:causal-context-builder}
\end{equation}
Starting from the Gaussian state available at $t$, a target-side tracker is rolled forward only with observations available up to each target endpoint,
\begin{equation}
\G_{t+1:t+\Delta}^{+}
=
B_{\mathrm{on}}^{+}(\G_t,X_{t+1:t+\Delta}),
\label{eq:causal-target-builder}
\end{equation}
and is not allowed to revise any state at or before $t$. A Gaussian representation fitted jointly using the complete sequence would generally violate this protocol because its representation at time $t$ could depend on future observations. The controlled experiments instantiate this causal protocol directly in Gaussian space, where the generated state at each time is available without fitting a reconstruction model; evaluation with a learned online Gaussian front-end is left to future work.

For a context length $L$, define the causal context history
\begin{equation}
\Hh_t
=
(\G_{t-L+1},\ldots,\G_t).
\label{eq:context-history}
\end{equation}
For every target horizon $\Delta$, the target branch receives a history of the same temporal extent,
\begin{equation}
\Hh_{t+\Delta}^{+}
=
(\G_{t+\Delta-L+1}^{+},\ldots,\G_{t+\Delta}^{+}),
\label{eq:target-history}
\end{equation}
where states at or before the context endpoint are inherited unchanged, $\G_k^{+}=\G_k$ for $k\leq t$, while states after $t$ are obtained through forward-only target tracking. The goal is to learn a predictive state that is sufficient for forecasting future scene evolution but need not preserve unpredictable or unnecessary appearance details \citep{huang2026vjepa}.

\subsection{Hierarchical student--teacher predictive state}
\label{sec:hierarchical-state}

4DGS-JEPA uses a student--teacher joint-embedding architecture. The context, or \emph{student}, encoder $E_\theta$ maps the causal history $\Hh_t$ to a predictive state $Z_t$, and the transition model $\Phi_\phi$ predicts its representation at a future horizon $\Delta$. During training, a target, or \emph{teacher}, encoder $E_{\bar\theta}$ receives the corresponding future Gaussian history $\Hh_{t+\Delta}^{+}$ and provides the target representation $Z_{t+\Delta}^{+}$. The predictive problem is summarized as
\begin{equation}
\Hh_t
\xrightarrow{\;E_\theta\;}
Z_t
\xrightarrow{\;\Phi_\phi(\cdot,\Delta)\;}
\widehat Z_{t+\Delta}
\approx
Z_{t+\Delta}^{+}
\xleftarrow{\;E_{\bar\theta}\;}
\Hh_{t+\Delta}^{+}.
\label{eq:problem-overview}
\end{equation}
When actions are available, $\Phi_\phi$ is additionally conditioned on the corresponding action sequence $A_{t:t+\Delta-1}$. Future observations are therefore available only to the target branch during training and are never supplied to the student encoder or transition model.

The student encoder represents the Gaussian world hierarchically,
\begin{equation}
Z_t
=
E_\theta(\Hh_t)
=
\left(
z_t^{S},
\{z_{m,t}^{O}\}_{m=1}^{M_t},
\{z_{i,t}^{G}\}_{i\in\Iset_t}
\right),
\label{eq:hier-state}
\end{equation}
where the superscripts $S$, $O$, and $G$ denote the scene, motion-group, and Gaussian levels, respectively. Here, the scene token $z_t^{S}$ summarizes the global predictive state, the motion-group token $z_{m,t}^{O}$ represents group $m$, and the Gaussian token $z_{i,t}^{G}$ represents primitive $i$. The quantity $M_t$ is the number of motion groups at time $t$, while $\Iset_t\subseteq\{1,\ldots,N_t\}$ indexes the Gaussian primitives represented explicitly by local tokens and may contain all active primitives or a sampled subset for computational efficiency.

The encoder is permutation equivariant over Gaussian primitives. It may use canonical coordinates, motion-group membership, temporal information, and persistent identity embeddings when available. Pooling follows the same hierarchy: Gaussian-level tokens first aggregate into motion-group tokens, which subsequently aggregate into the global scene token. By default, the predictive representation emphasizes geometry, opacity, existence, semantics, and low-order appearance cues. Detailed view-dependent appearance is excluded or independently perturbed so that joint-embedding prediction does not reduce to reproducing exact appearance.

The target encoder $E_{\bar\theta}$ has the same temporal and hierarchical architecture as the student encoder and is updated by exponential moving average (EMA),
\begin{equation}
\bar\theta
\leftarrow
\tau\bar\theta+(1-\tau)\theta,
\label{eq:ema-target}
\end{equation}
producing
\begin{equation}
Z_{t+\Delta}^{+}
=
E_{\bar\theta}(\Hh_{t+\Delta}^{+}),
\label{eq:target-state}
\end{equation}
where $\tau\in[0,1)$ is the EMA momentum and the superscript $+$ denotes target-branch quantities. The teacher is not updated directly by gradient descent and is not required for future prediction at inference time. We use equal context and target history lengths so that the student and teacher representations have access to comparable temporal information; for example, a multi-frame teacher can represent motion information that would be unavailable to a single-frame target encoder.

Having defined the causal hierarchical predictive state $Z_t$, we next introduce the horizon-conditioned transition and the objectives used to make its dynamics accurate, recursively stable, temporally compositional, and geometrically coherent. \Cref{fig:overview} summarizes how these components interact within the complete 4DGS-JEPA architecture.

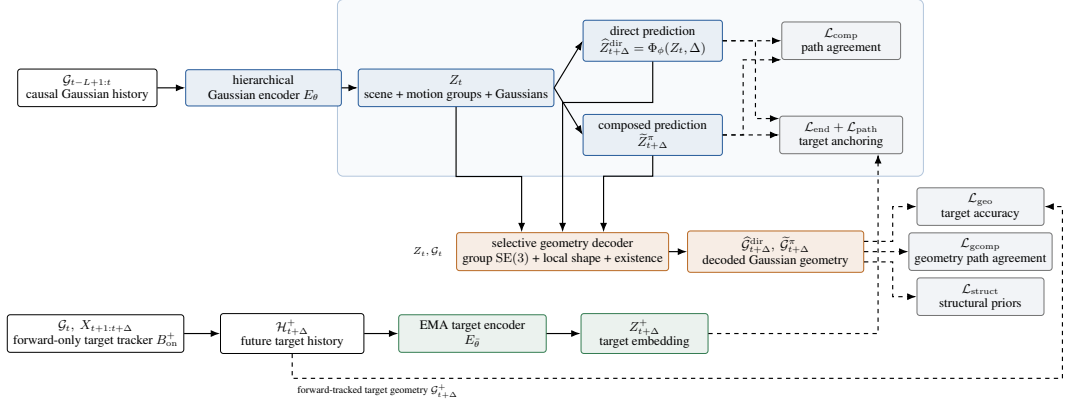
\begin{figure}[t]
\centering
\resizebox{\textwidth}{!}{%
\begin{tikzpicture}[
    box/.style={
        draw,
        rounded corners=2pt,
        minimum height=8mm,
        align=center,
        inner sep=4pt,
        font=\small
    },
    enc/.style={box, fill=methodblue!10, draw=methodblue},
    target/.style={box, fill=targetgreen!10, draw=targetgreen},
    dec/.style={box, fill=decodeorange!12, draw=decodeorange},
    loss/.style={box, fill=softgray, draw=black!55},
    arr/.style={-Latex, thick},
    darr/.style={-Latex, thick, dashed}
]


\node[box, minimum width=34mm] (hist) at (0,0)
{$\G_{t-L+1:t}$\\causal Gaussian history};

\node[enc, minimum width=38mm] (encoder) at (4.3,0)
{hierarchical\\Gaussian encoder $E_\theta$};

\node[enc, minimum width=42mm] (state) at (9.0,0)
{$Z_t$\\scene + motion groups + Gaussians};

\node[enc, minimum width=34mm] (direct) at (13.8,1.15)
{direct prediction\\
$\widehat Z_{t+\Delta}^{\mathrm{dir}}
=
\Phi_\phi(Z_t,\Delta)$};

\node[enc, minimum width=34mm] (path) at (13.8,-1.15)
{composed prediction\\
$\widetilde Z_{t+\Delta}^{\pi}$};

\node[loss, minimum width=29mm] (comploss) at (18.4,1.15)
{$\mathcal L_{\mathrm{comp}}$\\path agreement};

\node[loss, minimum width=30mm] (latentloss) at (18.4,-1.15)
{$\mathcal L_{\mathrm{end}}+\mathcal L_{\mathrm{path}}$\\target anchoring};


\node[dec, minimum width=46mm] (decoder) at (11.6,-4.0)
{selective geometry decoder\\
group $\SE(3)$ + local shape + existence};

\node[font=\scriptsize, anchor=east]
at ([xshift=-2mm]decoder.west)
{$Z_t,\G_t$};

\node[dec, minimum width=43mm] (futuregeom) at (16.8,-4.0)
{$\widehat{\G}_{t+\Delta}^{\mathrm{dir}},
\ \widetilde{\G}_{t+\Delta}^{\pi}$\\
decoded Gaussian geometry};

\node[loss, minimum width=31mm] (geotarget) at (21.8,-2.9)
{$\mathcal L_{\mathrm{geo}}$\\target accuracy};

\node[loss, minimum width=31mm] (gcomploss) at (21.8,-4.0)
{$\mathcal L_{\mathrm{gcomp}}$\\geometry path agreement};

\node[loss, minimum width=31mm] (structloss) at (21.8,-5.1)
{$\mathcal L_{\mathrm{struct}}$\\structural priors};


\node[box, minimum width=38mm] (targetsrc) at (0.2,-6.0)
{$\G_t,\ X_{t+1:t+\Delta}$\\forward-only target tracker $B_{\mathrm{on}}^{+}$};

\node[box, minimum width=35mm] (future) at (5.0,-6.0)
{$\Hh_{t+\Delta}^{+}$\\future target history};

\node[target, minimum width=36mm] (targetenc) at (9.4,-6.0)
{EMA target encoder\\$E_{\bar\theta}$};

\node[target, minimum width=31mm] (targetstate) at (13.6,-6.0)
{$Z_{t+\Delta}^{+}$\\target embedding};


\draw[arr] (hist) -- (encoder);
\draw[arr] (encoder) -- (state);

\draw[arr] (state.east) -- (direct.west);
\draw[arr] (state.east) -- (path.west);

\draw[arr]
(state.south)
-- ++(0,-17mm)
-| ([xshift=-10mm]decoder.north);

\draw[arr]
(direct.south)
-- ++(0,-10mm)
-| ([xshift=0mm]decoder.north);

\draw[arr]
(path.south)
-- ++(0,-5mm)
-| ([xshift=10mm]decoder.north);

\draw[arr] (decoder) -- (futuregeom);

\draw[arr] (targetsrc) -- (future);
\draw[arr] (future) -- (targetenc);
\draw[arr] (targetenc) -- (targetstate);


\draw[darr]
(direct.east)
-- (comploss.west);

\draw[darr]
(path.east)
-- ++(5mm,0)
|- (comploss.south west);

\draw[darr]
(direct.east)
-- ++(8mm,0)
|- ([yshift=4mm]latentloss.west);

\draw[darr]
(path.east)
-- (latentloss.west);

\draw[darr]
(targetstate.east)
-- ++(4mm,0)
-| ([xshift=9mm]latentloss.south);


\draw[darr]
([yshift=2.5mm]futuregeom.east)
-- ++(7mm,0)
|- (geotarget.west);

\draw[darr]
(futuregeom.east)
-- ++(9mm,0)
|- (gcomploss.west);

\draw[darr]
([yshift=-2.5mm]futuregeom.east)
-- ++(7mm,0)
|- (structloss.west);

\draw[darr]
(future.south)
-- ++(0,-6mm)
node[below right, font=\scriptsize]
{forward-tracked target geometry $\G_{t+\Delta}^{+}$}
-- ++(18.8,0)
|- ([xshift=4mm]geotarget.east)
-- (geotarget.east);


\begin{scope}[on background layer]
\node[
    draw=methodblue!35,
    rounded corners=4pt,
    fit=(state)(direct)(path)(comploss)(latentloss),
    inner sep=5mm,
    fill=methodblue!3
] {};
\end{scope}

\end{tikzpicture}%
}
\caption{Overview of 4DGS-JEPA. A causal Gaussian history is encoded into a hierarchical predictive state with scene-, motion-group-, and Gaussian-level tokens. A shared horizon-conditioned transition produces direct and recursively composed future predictions. The future target history is constructed by forward-only tracking from the causal Gaussian state and subsequent observations, and an EMA target encoder provides the corresponding target embedding. Target embeddings provide endpoint and path supervision, while direct--composed agreement enforces temporal composition. A selective geometry decoder maps the predicted latent dynamics to explicit Gaussian motion, deformation, and existence variables. The decoded states are constrained by target-relative geometry accuracy, geometry-level composition, and source-relative structural priors. Solid arrows denote computational flow and dashed arrows denote supervision or consistency constraints.}
\label{fig:overview}
\end{figure}

\subsection{Horizon-conditioned latent transition}
\label{sec:latent-transition}

Given the hierarchical predictive state $Z_t$, we learn a horizon-conditioned transition operator
\begin{equation}
\widehat Z_{t+\Delta}
=
\Phi_\phi(Z_t,\Delta,A_{t:t+\Delta-1}),
\label{eq:transition}
\end{equation}
where $A_{t:t+\Delta-1}$ is an optional action sequence and is omitted for passive scene forecasting. The transition is \emph{closed} in the predictive state space: it consumes and returns the same hierarchical state type as \cref{eq:hier-state}. Consequently, a predicted state can itself be supplied to $\Phi_\phi$ as the input to a subsequent transition, enabling recursive rollout and temporal composition. The relative horizon $\Delta$ is mapped to a temporal embedding that conditions each transition block, specifying how far forward the operator should propagate the current state. When the dynamics are explicitly non-autonomous, an absolute-time embedding of the transition start time $t$ is additionally provided.

\subsection{Hierarchical latent distance}
\label{sec:hierarchical-distance}

Prediction and composition are evaluated jointly at the three levels of the hierarchical predictive state introduced in \cref{eq:hier-state}: scene, motion-group, and Gaussian. For two such hierarchical latent states,
\begin{equation}
Z
=
\left(
z^{S},
\{z_m^{O}\},
\{z_i^{G}\}
\right),
\qquad
Z'
=
\left(
z'^{S},
\{z_n'^{O}\},
\{z_j'^{G}\}
\right),
\end{equation}
let $\Pi^{O}$ and $\Pi^{G}$ denote group- and Gaussian-level correspondence matrices, respectively. Because Gaussian primitives are not intrinsically index-aligned across time, these matrices may encode persistent correspondences or soft matches; their construction is described in \cref{sec:identity}.

Define the matched masses
\begin{equation}
m_O
=
\sum_{m,n}\Pi^{O}_{mn},
\qquad
m_G
=
\sum_{i,j}\Pi^{G}_{ij}.
\end{equation}
We use the hierarchical latent distance
{\small
\begin{equation}
\dhier(Z,Z')
=
\underbrace{
\lambda_S\cosdist(z^S,z'^S)
}_{\text{scene-level distance}}
+
\underbrace{
\frac{\lambda_O}{m_O+\varepsilon_{\mathrm{match}}}
\sum_{m,n}\Pi^O_{mn}\cosdist(z_m^O,z_n'^O)
}_{\text{group-level distance}}
+
\underbrace{
\frac{\lambda_G}{m_G+\varepsilon_{\mathrm{match}}}
\sum_{i,j}\Pi^G_{ij}\rho\!\left(z_i^G-z_j'^G\right)
}_{\text{Gaussian-level distance}},
\label{eq:hier-distance}
\end{equation}
}
where $\cosdist(u,v)=1-\frac{u^\top v}{\|u\|_2\|v\|_2}$ denotes cosine distance. Here, $m,n$ index motion groups in $Z$ and $Z'$, respectively, while $i,j$ index their Gaussian primitives; the superscripts $S$, $O$, and $G$ denote the scene, motion-group, and Gaussian levels. The matrices $\Pi^O$ and $\Pi^G$ provide the corresponding cross-state matches, $m_O$ and $m_G$ are their matched masses, and $\lambda_S,\lambda_O,\lambda_G\geq0$ weight the three hierarchy levels. The function $\rho$ is a robust Huber penalty, and $\varepsilon_{\mathrm{match}}>0$ prevents division by zero. The three terms respectively capture global scene evolution, coherent group motion, and local Gaussian deformation.

\subsection{Direct multi-horizon supervision}
\label{sec:direct-supervision}

The first learning objective anchors direct predictions to future teacher representations. For a sampled set of target horizons $\Kset_t$, we define
\begin{equation}
\mathcal{L}_{\mathrm{end}}
=
\sum_{\Delta\in\Kset_t}
w_\Delta
\dhier\left(
\widehat Z_{t+\Delta},
\stopgrad(Z_{t+\Delta}^{+})
\right),
\label{eq:endpoint-loss}
\end{equation}
where $\widehat Z_{t+\Delta}$ is given by \cref{eq:transition}, $w_\Delta\geq0$, and
$
\sum_{\Delta\in\Kset_t}w_\Delta=1.
$
For discrete frame horizons, a natural choice is the normalized temporal discount
\begin{equation}
w_\Delta
=
\frac{\gamma^{\Delta-1}}
{\sum_{\delta\in\Kset_t}\gamma^{\delta-1}},
\qquad
\gamma\in(0,1].
\label{eq:horizon-discount}
\end{equation}
Here, $\gamma=1$ weights all sampled horizons equally, whereas $\gamma<1$ progressively emphasizes nearer and typically more predictable future states. The weighting in \cref{eq:horizon-discount} therefore provides a general mechanism for emphasizing different temporal ranges. In the controlled experiments below, however, we use one terminal horizon per training step, cycling uniformly through the supervised horizon set, and do not separately ablate $\gamma$. Unlike pixel- or Gaussian-reconstruction objectives, $\mathcal L_{\mathrm{end}}$ requires the transition to predict only the target \textit{representation} of the future Gaussian world.

\subsection{Multi-horizon path supervision}
\label{sec:path-supervision}

Direct multi-horizon supervision does not guarantee that the learned transition remains accurate when recursively applied to its own predictions. For each terminal horizon $\Delta$, we therefore sample a temporal partition
\begin{equation}
\pi=(\delta_1,\ldots,\delta_J),
\qquad
\delta_j>0,
\qquad
\sum_{j=1}^{J}\delta_j=\Delta.
\end{equation}
Let
\begin{equation}
s_0=0,
\qquad
s_j=\sum_{q=1}^{j}\delta_q,
\qquad
s_J=\Delta
\end{equation}
denote the corresponding cumulative horizons. Starting from
$\widetilde Z_t^{\pi}=Z_t$, the recursively composed path is
\begin{equation}
\widetilde Z_{t+s_j}^{\pi}
=
\Phi_\phi\left(
\widetilde Z_{t+s_{j-1}}^{\pi},
\delta_j,
A_{t+s_{j-1}:t+s_j-1}
\right).
\label{eq:path-rollout}
\end{equation}

We anchor every intermediate composed state to the corresponding future teacher state:
\begin{equation}
\mathcal{L}_{\mathrm{path}}
=
\mathbb{E}_{\Delta,\pi}
\left[
\sum_{j=1}^{J}
\omega_j
\dhier\left(
\widetilde Z_{t+s_j}^{\pi},
\stopgrad(Z_{t+s_j}^{+})
\right)
\right],
\label{eq:path-loss}
\end{equation}
where $\omega_j\geq0$ controls the contribution of the $j$th cumulative horizon and is normalized over the sampled path.

Unlike direct endpoint prediction, the composed path repeatedly applies $\Phi_\phi$ to its own predicted states. Supervising the intermediate states therefore exposes the transition to the state distribution encountered during rollout and reduces the mismatch between direct training and recursive inference.

\subsection{Temporal composition consistency}
\label{sec:temporal-composition}

Endpoint and path supervision anchor direct and recurrent predictions to future teacher states, but they do not explicitly require two valid prediction paths leading to the same future time to agree. \emph{Temporal composition} imposes this additional constraint.

For each cumulative horizon $s_j$, there are two predictions of the same future state. The \emph{direct} prediction reaches $t+s_j$ in a single transition,
\begin{equation}
\widehat Z_{t+s_j}^{\mathrm{dir}}=\Phi_\phi\!\left(Z_t,s_j,A_{t:t+s_j-1}\right),
\label{eq:direct-cumulative-prediction}
\end{equation}
whereas the \emph{composed} prediction $\widetilde Z_{t+s_j}^{\pi}$ is obtained recursively through a sequence of shorter transitions along the temporal partition $\pi=(\delta_1,\ldots,\delta_J)$ according to \cref{eq:path-rollout}. Thus, $\widehat Z_{t+s_j}^{\mathrm{dir}}$ represents one direct transition to the future endpoint, while $\widetilde Z_{t+s_j}^{\pi}$ represents a composition of shorter transitions reaching the same endpoint. Temporal composition requires these two predictions to agree.

The temporal composition loss\footnote{The intuition behind temporal composition consistency is explained in \cref{app:temporal-composition-intuition}.} is
{\small
\begin{equation}
\mathcal{L}_{\mathrm{comp}}
=
\mathbb{E}_{\Delta,\pi}\!\left[\sum_{j=1}^{J}u_j\,\dhier\!\left(\widehat Z_{t+s_j}^{\mathrm{dir}},\stopgrad(\widetilde Z_{t+s_j}^{\pi})\right)\right]
+
\mathbb{E}_{\Delta,\pi}\!\left[\sum_{j=1}^{J}\nu_j\,\dhier\!\left(\widetilde Z_{t+s_j}^{\pi},\stopgrad(\widehat Z_{t+s_j}^{\mathrm{dir}})\right)\right].
\label{eq:composition-loss}
\end{equation}
}
Here, $u_j,\nu_j\geq0$ weight the two optimization directions. In the first term, the composed prediction is stop-gradient and therefore serves as a fixed target for the direct prediction. In the second term, the direct prediction is stop-gradient and serves as a fixed target for the composed prediction. The bidirectional construction therefore encourages agreement while allowing both prediction paths to receive consistency gradients \citep{huang2026bijepa}.

For two consecutive intervals, the objective encourages the controlled composition law
\begin{equation}
\Phi_\phi\!\left(Z_t,\delta_1+\delta_2,A_{t:t+\delta_1+\delta_2-1}\right)
\approx
\Phi_\phi\!\left(\Phi_\phi\!\left(Z_t,\delta_1,A_{t:t+\delta_1-1}\right),\delta_2,A_{t+\delta_1:t+\delta_1+\delta_2-1}\right).
\label{eq:controlled-composition}
\end{equation}
We refer to this property as \emph{temporal composition} rather than an exact one-parameter semigroup property because the transition may depend on control inputs and absolute time and may include discrete state changes such as Gaussian birth or death. The relevant requirement is therefore consistency across admissible decompositions of the same controlled temporal evolution. In passive forecasting, the action arguments are omitted.

The three latent objectives have complementary roles: $\mathcal L_{\mathrm{end}}$ anchors direct predictions to future targets, $\mathcal L_{\mathrm{path}}$ anchors recursively generated states, and $\mathcal L_{\mathrm{comp}}$ reduces dependence on the temporal path used to reach a future state. Composition alone is insufficient because direct and composed paths may agree while both being inaccurate or collapsed.

\paragraph{Theoretical characterization.}

The distinction between target anchoring and path consistency can be stated formally. Because $\dhier$ contains correspondence-weighted terms, equality below is understood on the supervised matched support. We assume that the compared scene- and motion-group-level embeddings are normalized, that the Gaussian-level discrepancy vanishes only at equality, and that the active correspondence plans cover the components whose equality is asserted.

For a fixed cumulative horizon $s_j$ and temporal partition $\pi$, define the corresponding local contributions of \cref{eq:path-loss,eq:composition-loss} as
\begin{equation*}
\ell_{\mathrm{path},j}^{\pi}
=
\omega_j
\dhier\!\left(
\widetilde Z_{t+s_j}^{\pi},
\stopgrad(Z_{t+s_j}^{+})
\right),
\end{equation*}
and
\begin{equation*}
\ell_{\mathrm{comp},j}^{\pi}
=
u_j
\dhier\!\left(
\widehat Z_{t+s_j}^{\mathrm{dir}},
\stopgrad(\widetilde Z_{t+s_j}^{\pi})
\right)
+
\nu_j
\dhier\!\left(
\widetilde Z_{t+s_j}^{\pi},
\stopgrad(\widehat Z_{t+s_j}^{\mathrm{dir}})
\right).
\end{equation*}

\begin{proposition}[Zero-loss path agreement]
\label{prop:path-agreement}
Fix a supervised cumulative horizon $s_j$ and temporal partition $\pi$. Assume $\omega_j>0$ and $u_j+\nu_j>0$. If
\begin{equation}
\ell_{\mathrm{path},j}^{\pi}
=
0,
\qquad
\ell_{\mathrm{comp},j}^{\pi}
=
0,
\label{eq:zero-local-losses}
\end{equation}
then
\begin{equation}
\widehat Z_{t+s_j}^{\mathrm{dir}}
=
\widetilde Z_{t+s_j}^{\pi}
=
Z_{t+s_j}^{+}
\label{eq:zero-loss-path-agreement}
\end{equation}
on the supervised matched support.
\end{proposition}

\Cref{prop:path-agreement} makes explicit why composition alone is insufficient: $\mathcal L_{\mathrm{comp}}=0$ establishes agreement between prediction routes, whereas target anchoring through $\mathcal L_{\mathrm{path}}$ is what connects that agreement to the future target representation. The proof is given in \cref{app:path-agreement-proof}.

Away from the zero-loss limit, recursive prediction can accumulate local transition errors. For a temporal partition $\pi=(\delta_1,\ldots,\delta_J)$, let the transition over segment $j$ be locally $L_j$-Lipschitz and let $\epsilon_j$ bound its local approximation error along a reference future trajectory. If
\[
e_0
=
\|Z_t-Z_0^\star\|
\]
denotes the initial mismatch between the rollout state and that reference trajectory, then
\begin{equation}
\left\|
\widetilde Z_{t+s_J}^{\pi}
-
Z_{t+s_J}^{+}
\right\|
\leq
\left(
\prod_{q=1}^{J}L_q
\right)e_0
+
\sum_{j=1}^{J}
\left(
\prod_{q=j+1}^{J}L_q
\right)
\epsilon_j.
\label{eq:rollout-bound}
\end{equation}
When the rollout and reference trajectory share the same initial predictive state, $e_0=0$. The bound separates local transition error from its amplification through recursive application of $\Phi_\phi$: errors introduced early in a rollout are propagated through the Lipschitz factors of all subsequent transitions. A derivation and precise construction of the reference trajectory are provided in \cref{app:theory}.

Together, \cref{prop:path-agreement,eq:rollout-bound} formalize the complementary roles of the two recurrent objectives. Path supervision controls target-relative errors along realized recurrent trajectories, while temporal composition constrains disagreement between recurrent and direct routes to the same future state. Neither property alone guarantees accurate long-horizon prediction.

\subsection{Geometry-aware selective decoder}
\label{sec:geometry-decoder}

The latent objectives above define the predictive representation without requiring observation reconstruction. Latent agreement alone, however, does not guarantee that a predicted state corresponds to coherent motion in 3D. We therefore introduce a \textit{selective geometry decoder} \footnote{Here, \emph{selective} means selecting which Gaussian attributes are decoded: the decoder predicts the variables needed for geometric dynamics rather than reconstructing the complete future Gaussian state or appearance. It doesn't mean the decoder is \textit{optional}.} that grounds the latent transition in a small set of interpretable Gaussian updates. Importantly, the decoder does not reconstruct the complete future Gaussian appearance; it predicts only the variables required to describe geometric evolution.

For any temporal interval $[a,b]$ with $a<b$, the decoder receives the latent state at the start of the interval, $Z_a$, the predicted latent state at its endpoint, $Z_b$, the source Gaussian scene $\G_a$, and the elapsed time $b-a$. Its role is to translate the latent transition from $a$ to $b$ into explicit motion of the scene's motion groups and Gaussian primitives.

Let $\Iset_a^{\mathrm{dec}}$ denote the decoder-side canonical slots at time $a$, including represented active primitives and any reserved inactive slots eligible for activation. The shared decoder $Q_\psi$ produces
\begin{equation}
\left(
\{\widehat\xi_{m,a\rightarrow b}\}_{m=1}^{M_a},
\{
\widehat r_{i,a\rightarrow b},
\widehat S_{i,a\rightarrow b},
\widehat e_{i,b}
\}_{i\in\Iset_a^{\mathrm{dec}}}
\right)
=
Q_\psi(Z_a,Z_b,\G_a,b-a),
\label{eq:decoder-outputs}
\end{equation}
where $a$ and $b$ denote the start and end times of the decoded interval, $m\in\{1,\ldots,M_a\}$ indexes motion groups present at time $a$, and $i\in\Iset_a^{\mathrm{dec}}$ indexes decoder-side Gaussian slots. The notation $a\rightarrow b$ indicates that a quantity describes the change accumulated over the interval from $a$ to $b$.

At the group level, $\widehat\xi_{m,a\rightarrow b}\in\mathfrak{se}(3)$ is the predicted rigid-motion increment for motion group $m$. At the Gaussian level, $\widehat r_{i,a\rightarrow b}\in\R^3$ is a local non-rigid displacement residual, $\widehat S_{i,a\rightarrow b}\in\mathbb{S}^{3}$ is a symmetric log-deformation controlling covariance change, and $\widehat e_{i,b}\in[0,1]$ is the predicted existence probability of slot $i$ at the endpoint $b$. Thus, the decoder decomposes scene evolution into coherent group motion, local primitive deformation, covariance evolution, and birth/death dynamics.
During forecasting, $Z_b$ is always a predicted endpoint state, either direct or composed; the teacher representation is not supplied to the decoder.

For a Gaussian $i$ assigned to motion group $m(i)$, the predicted group-level motion increment $\widehat\xi_{m,a\rightarrow b}\in\mathfrak{se}(3)$ is converted into a finite rigid transformation through the exponential map\footnote{$\SO(3)$ is the Lie group of proper 3D rotations, i.e., $3\times3$ orthogonal matrices $R$ with $\det(R)=1$. The special Euclidean group $\SE(3)$ augments such a rotation with a translation $d\in\R^3$ and therefore represents a complete rigid-body transformation $(R,d)$. In contrast, $\mathfrak{se}(3)$ is the Lie algebra associated with $\SE(3)$: it is the tangent space of infinitesimal rigid motions around the identity and can be parameterized by rotational and translational increments. The exponential map $\Exp:\mathfrak{se}(3)\rightarrow\SE(3)$ converts such an infinitesimal-coordinate motion representation into a finite rigid transformation.},
{\small
\begin{equation}
\widehat T_{m,a\rightarrow b}
=
\Exp(\widehat\xi_{m,a\rightarrow b})
=
\left(
\widehat R_{m,a\rightarrow b},
\widehat d_{m,a\rightarrow b}
\right)
\in\SE(3),
\label{eq:decoded-group-transform}
\end{equation}
}
where $\mathfrak{se}(3)$ is the Lie algebra of 3D rigid motions, $\widehat R_{m,a\rightarrow b}\in\SO(3)$ is the predicted rotation of group $m$ from time $a$ to $b$, and $\widehat d_{m,a\rightarrow b}\in\R^3$ is its predicted translation. Thus, $\Exp:\mathfrak{se}(3)\rightarrow\SE(3)$ maps the group-level motion increment to a rigid transformation composed of rotation and translation.

At the Gaussian level, the symmetric log-deformation is mapped through the matrix exponential to the local deformation matrix
\begin{equation}
\widehat U_{i,a\rightarrow b}
=
\exp(\widehat S_{i,a\rightarrow b}),
\label{eq:decoded-local-deformation}
\end{equation}
where $\widehat U_{i,a\rightarrow b}\in\mathbb{S}_{++}^{3}$ is a positive-definite local deformation, or stretch, matrix controlling Gaussian-specific changes in shape and scale. Thus, $\widehat T_{m,a\rightarrow b}$ captures the coherent rigid motion shared by a motion group, while $\widehat U_{i,a\rightarrow b}$ and $\widehat r_{i,a\rightarrow b}$ capture local deviations from that group motion. Here, $\Exp(\cdot)$ denotes the rigid-motion exponential map from $\mathfrak{se}(3)$ to $\SE(3)$, whereas $\exp(\cdot)$ in \cref{eq:decoded-local-deformation} denotes the matrix exponential.

The corresponding geometric updates for the center and covariance of Gaussian $i$ are
{\small
\begin{align}
\widehat\mu_{i,b}
&=
\widehat R_{m(i),a\rightarrow b}\mu_{i,a}
+
\widehat d_{m(i),a\rightarrow b}
+
\widehat r_{i,a\rightarrow b},
\label{eq:decoded-center}\\
\widehat\Sigma_{i,b}
&=
\widehat R_{m(i),a\rightarrow b}
\widehat U_{i,a\rightarrow b}
\Sigma_{i,a}
\widehat U_{i,a\rightarrow b}^{\top}
\widehat R_{m(i),a\rightarrow b}^{\top}.
\label{eq:decoded-cov}
\end{align}
}
The center update therefore combines the shared group rotation and translation with a Gaussian-specific residual displacement. The covariance update first applies the local shape/scale deformation and then transports the resulting covariance through the group rotation. Its congruence form preserves positive definiteness while allowing contraction, expansion, and anisotropic deformation.

Together with the predicted existence probability $\widehat e_{i,b}$ and the transported source attributes, these updates instantiate the decoded Gaussian at time $b$,
{\small
\begin{equation}
\widehat g_{i,b}
=
\left(
\widehat\mu_{i,b},
\widehat\Sigma_{i,b},
\alpha_{i,a},
c_{i,a},
f_{i,a},
\widehat e_{i,b}
\right),
\qquad
\widehat\G_b
=
\{\widehat g_{i,b}\}_{i\in\Iset_a^{\mathrm{dec}}},
\label{eq:decoded-gaussian-state}
\end{equation}
}
with inactive slots interpreted according to their predicted existence probabilities. Thus, the decoder outputs are converted into an explicit predicted Gaussian scene $\widehat\G_b$. This scene is subsequently used for target-anchored geometric supervision and geometry-level composition; for a composed rollout, it additionally becomes the source Gaussian state for decoding the next temporal segment.

The full decoding chain can be summarized as
{\small
\begin{equation}
(Z_a,Z_b,\G_a,b-a)
\xrightarrow{\;Q_\psi\;}
\left(
\{\widehat\xi_{m,a\rightarrow b}\},
\{\widehat r_{i,a\rightarrow b},\widehat S_{i,a\rightarrow b},\widehat e_{i,b}\}
\right)
\xrightarrow{\;\text{geometric update}\;}
\widehat\G_b
\xrightarrow{\;\text{losses / next rollout step}\;}
\cdots
\label{eq:decoder-chain}
\end{equation}
}
That is, the decoder $Q_\psi$ first predicts motion and deformation variables, these variables are then converted into an explicit Gaussian state $\widehat\G_b$ through \cref{eq:decoded-center,eq:decoded-cov}, and the resulting state is subsequently used either for geometric supervision or as the source state for the next segment in a composed rollout.

The default decoder does not predict high-frequency appearance. Appearance coefficients $c_{i,a}$ and semantic features $f_{i,a}$ are transported from the source state where defined, while the predicted existence probability $\widehat e_{i,b}$ controls activation and deactivation of canonical slot $i$ at the endpoint $b$. Opacity $\alpha_{i,a}$ is likewise transported for surviving primitives rather than independently reconstructed. A small appearance-residual head may be trained separately for rendering evaluation, but appearance reconstruction is excluded from the default JEPA objective.

For a cumulative horizon $s_j$, direct decoding applies the shared decoder once,
{\small
\begin{equation}
\widehat\G_{t+s_j}^{\mathrm{dir}}
=
Q_\psi^{\G}\!\left(
Z_t,\widehat Z_{t+s_j}^{\mathrm{dir}},\G_t,s_j
\right),
\label{eq:direct-geometry-decoding}
\end{equation}
}
where $Q_\psi^{\G}$ denotes application of the decoder outputs in \cref{eq:decoder-outputs} through the geometric updates in \cref{eq:decoded-center,eq:decoded-cov}. For a composed path, we initialize $\widetilde\G_t^{\pi}=\G_t$ and recursively decode each interval,
{\small
\begin{equation}
\widetilde\G_{t+s_q}^{\pi}
=
Q_\psi^{\G}\!\left(
\widetilde Z_{t+s_{q-1}}^{\pi},
\widetilde Z_{t+s_q}^{\pi},
\widetilde\G_{t+s_{q-1}}^{\pi},
\delta_q
\right),
\qquad q=1,\ldots,j.
\label{eq:composed-geometry-decoding}
\end{equation}
}
Thus, group transformations, local residuals, covariance changes, and existence states are rolled forward along the composed path rather than decoded only at the terminal endpoint.

\paragraph{Target-anchored geometric supervision.}

The selective decoder is trained against the forward-tracked target Gaussian states. The purpose of this supervision is to ask whether the decoded prediction places the correct Gaussian geometry at the future endpoint: predicted primitives should have the correct positions, shapes, and existence states.

Let $\Iset_{\mathrm{pred}}$ denote the predicted decoder slots under comparison and let $\Iset^{+}$ denote the real target primitives, excluding the dustbin. Because predicted and target Gaussians are not necessarily index-aligned, we compare them through the Gaussian transport plan $\Pi^G$, where $\Pi^G_{ij}$ measures the correspondence mass between predicted slot $i$ and target primitive $j$. Given a predicted Gaussian state $\widehat\G$, a target state $\G^{+}$, and their transport plan $\Pi^G$, we define
{\scriptsize
\begin{equation}
\ell_{\mathrm{geo}}\!\left(\widehat\G,\G^{+};\Pi^G\right)
=
\sum_{i\in\Iset_{\mathrm{pred}}}\sum_{j\in\Iset^{+}}
\Pi^G_{ij}
\left[
\lambda_\mu\rho_\delta(\widehat\mu_i-\mu_j^{+})
+
\lambda_\Sigma\dspd(\widehat\Sigma_i,\Sigma_j^{+})
\right]
+
\lambda_e\sum_{i\in\Iset_{\mathrm{pred}}}
\mathrm{BCE}\!\left(\widehat e_i,\overline e_i^{+}\right).
\label{eq:geometry-loss-local}
\end{equation}
}
The loss measures three complementary geometric discrepancies. The center term $\rho_\delta(\widehat\mu_i-\mu_j^{+})$ measures the positional error between predicted Gaussian $i$ and target Gaussian $j$ using the robust Huber penalty\footnote{The Huber penalty is defined as $\rho_\delta(x)=\frac{1}{2}x^2$ for $|x|\leq\delta$ and $\rho_\delta(x)=\delta\left(|x|-\frac{1}{2}\delta\right)$ otherwise, where $\delta>0$ controls the transition from quadratic to linear penalization. For Gaussian-center errors, we apply it to the Euclidean displacement, $\rho_\delta(\widehat\mu_i-\mu_j^{+})\equiv\rho_\delta(\|\widehat\mu_i-\mu_j^{+}\|_2)$.} $\rho_\delta$. The covariance term $\dspd(\widehat\Sigma_i,\Sigma_j^{+})$ measures the discrepancy between their Gaussian covariance ellipsoids, and therefore penalizes errors in shape, scale, and orientation. Both terms are weighted by the correspondence mass $\Pi^G_{ij}$, so strongly matched Gaussian pairs contribute more strongly to the loss. The coefficients $\lambda_\mu,\lambda_\Sigma\geq0$ control the relative strengths of center and covariance supervision.

The final term supervises whether each predicted Gaussian slot should exist at the target endpoint. Here, $\mathrm{BCE}(p,y)=-y\log p-(1-y)\log(1-p)$ denotes binary cross-entropy between the predicted existence probability $\widehat e_i$ and a soft target existence value $\overline e_i^{+}$. We define
{\small
\begin{equation}
\overline e_i^{+}
=
\frac{\sum_{j\in\Iset^{+}}\Pi^G_{ij}}
{\sum_{j\in\Iset^{+}\cup\{\varnothing\}}\Pi^G_{ij}+\varepsilon_{\mathrm{match}}},
\label{eq:target-existence}
\end{equation}
}
where $\varnothing$ denotes the target dustbin and $\varepsilon_{\mathrm{match}}>0$ prevents division by zero. The numerator is the transport mass from predicted slot $i$ assigned to real target geometry, while the denominator additionally includes mass assigned to the dustbin. Consequently, $\overline e_i^{+}\approx1$ indicates that slot $i$ should remain or become active, whereas $\overline e_i^{+}\approx0$ indicates that it should disappear or remain inactive. Intermediate values naturally arise from soft or ambiguous correspondences. Transport to the dustbin therefore supervises disappearance or death, whereas transport from a reserved inactive slot to a real target primitive supervises birth and the geometry of the newly activated slot.

The target-anchored geometry objective applies this comparison to both prediction routes at every cumulative endpoint:
{\scriptsize
\begin{equation}
\mathcal{L}_{\mathrm{geo}}
=
\mathbb{E}_{\Delta,\pi}
\left[
\sum_{j=1}^{J}
\eta_j
\left(
\ell_{\mathrm{geo}}\!\left(
\widehat\G_{t+s_j}^{\mathrm{dir}},
\G_{t+s_j}^{+};
\Pi_{j,\mathrm{dir}}^G
\right)
+
\ell_{\mathrm{geo}}\!\left(
\widetilde\G_{t+s_j}^{\pi},
\G_{t+s_j}^{+};
\Pi_{j,\pi}^G
\right)
\right)
\right].
\label{eq:geometry-loss}
\end{equation}
}
Thus, at each cumulative horizon $s_j$, the direct prediction $\widehat\G_{t+s_j}^{\mathrm{dir}}$ and the composed prediction $\widetilde\G_{t+s_j}^{\pi}$ are each compared independently with the same forward-tracked target geometry $\G_{t+s_j}^{+}$. The corresponding transport plans $\Pi_{j,\mathrm{dir}}^G$ and $\Pi_{j,\pi}^G$ account for potentially different Gaussian correspondences along the two prediction routes. The weights $\eta_j\geq0$ control the contribution of each cumulative endpoint and are normalized over the sampled path, while the expectation averages over sampled terminal horizons $\Delta$ and temporal partitions $\pi$.

Importantly, $\mathcal L_{\mathrm{geo}}$ measures \emph{target accuracy}: it asks whether each prediction route produces the correct future geometry. It does not directly compare the direct and composed routes with one another; their mutual geometric agreement is enforced separately by the geometry-level composition objective in \cref{sec:geometry-regularization}. We use the affine-invariant SPD distance for $\dspd$ in the full implementation and a stable log-Cholesky approximation for large batches \footnote{Here, SPD denotes symmetric positive definite. Since each Gaussian covariance $\Sigma\in\mathbb S_{++}^{3}$ is SPD, the affine-invariant distance $d_{\mathrm{SPD}}(\Sigma_1,\Sigma_2)=\|\log(\Sigma_1^{-1/2}\Sigma_2\Sigma_1^{-1/2})\|_F$ measures differences in covariance shape, scale, anisotropy, and orientation while respecting the geometry of the SPD manifold. For computationally large batches, we use a cheaper log-Cholesky surrogate: writing $\Sigma=LL^\top$ with $L$ lower triangular and positive diagonal, the covariance is represented by the logarithms of the diagonal entries of $L$ together with its unconstrained off-diagonal entries, and distances are computed in these coordinates.}. In this way, the latent objectives determine what future information should be predicted, while the selective decoder encourages those predictions to correspond to coherent evolution of explicit Gaussian geometry.

\subsection{Geometry-level composition and structural regularization}
\label{sec:geometry-regularization}

The latent composition loss $\mathcal L_{\mathrm{comp}}$ in \cref{sec:temporal-composition} requires direct and composed \emph{latent} predictions to agree. This does not by itself guarantee that decoding the two latent paths yields the same explicit 3D motion. We therefore impose an analogous consistency constraint in Gaussian geometry. Whereas the target-anchored geometry loss $\mathcal L_{\mathrm{geo}}$ asks whether each decoded path agrees with the future target geometry, the geometry-level composition loss introduced below asks whether the \emph{direct and composed decoded motions agree with each other}.

Recall from \cref{eq:decoded-group-transform} that $\widehat T_{m,a\rightarrow b}\in\SE(3)$ denotes the predicted rigid transformation of motion group $m$ over an interval $[a,b]$. For a terminal horizon $\Delta$ and temporal partition $\pi=(\delta_1,\ldots,\delta_J)$ with cumulative horizons $s_J=\Delta$, let $\widehat T_{m,t\rightarrow t+\Delta}^{\mathrm{dir}}$ denote the transformation obtained by decoding the direct prediction from $t$ to $t+\Delta$. The corresponding composed transformation is obtained by multiplying the shorter segment transformations in chronological order,\footnote{The decoder represents each segment motion by an increment $\widehat\xi\in\mathfrak{se}(3)$ and maps it through $\Exp$ to a finite rigid transformation $\widehat T=\Exp(\widehat\xi)\in\SE(3)$, as in \cref{eq:decoded-group-transform}. Finite rigid transformations compose through the group operation of $\SE(3)$, i.e., matrix multiplication, so successive segment motions are combined as $\widehat T_J\cdots\widehat T_1$, with the rightmost transformation applied first. In Lie-algebra coordinates, this composition is described by the Baker--Campbell--Hausdorff relation and reduces approximately to addition of the motion increments only for sufficiently small or commuting motions.}
{\small
\begin{equation}
\widehat T_{m,t\rightarrow t+\Delta}^{\pi}
=
\widehat T_{m,t+s_{J-1}\rightarrow t+s_J}^{\pi}
\cdots
\widehat T_{m,t+s_1\rightarrow t+s_2}^{\pi}
\widehat T_{m,t\rightarrow t+s_1}^{\pi}.
\label{eq:composed-group-transform}
\end{equation}
}
The rightmost transformation is applied first. Thus, $\widehat T_{m,t\rightarrow t+\Delta}^{\mathrm{dir}}$ represents one decoded rigid motion over the full horizon, whereas $\widehat T_{m,t\rightarrow t+\Delta}^{\pi}$ represents the same-horizon motion obtained by composing the shorter decoded transformations along path $\pi$.

To compare two rigid transformations, we use their relative transformation. For $T_1,T_2\in\SE(3)$, define
{\small
\begin{equation}
d_{\SE(3)}(T_1,T_2)
=
\left\|
\operatorname{Log}\!\left(T_1^{-1}T_2\right)
\right\|_2,
\label{eq:se3-distance}
\end{equation}
}
where $\operatorname{Log}:\SE(3)\rightarrow\mathfrak{se}(3)$ is the Lie-group logarithm. The relative transform $T_1^{-1}T_2$ describes the residual rigid motion required to move from $T_1$ to $T_2$, and $d_{\SE(3)}(T_1,T_2)=0$ when the transformations coincide. Thus, this distance measures disagreement in group-level rotation and translation.

We additionally compare the Gaussian-level geometry produced by the two paths. Let $\Pi^{G,\mathrm{dc}}$ denote a direct--composed Gaussian correspondence plan, where $\Pi_{ij}^{G,\mathrm{dc}}$ measures the correspondence mass between Gaussian $i$ in the direct prediction and Gaussian $j$ in the composed prediction. Analogously to the target-anchored geometry loss in \cref{eq:geometry-loss-local}, we define
{\scriptsize
\begin{equation}
D_{\mathrm{match}}\!\left(\widehat\G^{\mathrm{dir}},\widetilde\G^{\pi};\Pi^{G,\mathrm{dc}}\right)
=
\sum_{i,j}\Pi_{ij}^{G,\mathrm{dc}}
\left[
\lambda_\mu\rho_\delta\!\left(\widehat\mu_i^{\mathrm{dir}}-\widetilde\mu_j^{\pi}\right)
+
\lambda_\Sigma\dspd\!\left(\widehat\Sigma_i^{\mathrm{dir}},\widetilde\Sigma_j^{\pi}\right)
+
\lambda_e\left(\widehat e_i^{\mathrm{dir}}-\widetilde e_j^{\pi}\right)^2
\right].
\label{eq:direct-composed-geometry-distance}
\end{equation}
}
Here, $\widehat\mu_i^{\mathrm{dir}}$, $\widehat\Sigma_i^{\mathrm{dir}}$, and $\widehat e_i^{\mathrm{dir}}$ denote the center, covariance, and existence probability produced by direct decoding, while $\widetilde\mu_j^{\pi}$, $\widetilde\Sigma_j^{\pi}$, and $\widetilde e_j^{\pi}$ denote their composed-path counterparts. The three terms therefore measure disagreement in position, covariance geometry, and existence, respectively. Unlike $\ell_{\mathrm{geo}}$ in \cref{eq:geometry-loss-local}, which compares a prediction with a future target and uses binary cross-entropy for target existence supervision, $D_{\mathrm{match}}$ compares two predictions and therefore uses a symmetric squared difference for their existence probabilities.

Let $\Mset_{\Delta,\pi}$ denote the set of motion groups for which valid direct--composed correspondence is available at horizon $\Delta$, and let $\Pi_{\Delta,\pi}^{G,\mathrm{dc}}$ denote the corresponding Gaussian-level direct--composed transport plan. We define the geometry-level composition loss
{\scriptsize
\begin{equation}
\mathcal{L}_{\mathrm{gcomp}}
=
\mathbb{E}_{\Delta,\pi}
\left[
\sum_{m\in\Mset_{\Delta,\pi}}
d_{\SE(3)}\!\left(
\widehat T_{m,t\rightarrow t+\Delta}^{\mathrm{dir}},
\widehat T_{m,t\rightarrow t+\Delta}^{\pi}
\right)
+
\lambda_{\mathrm{gc}}
D_{\mathrm{match}}\!\left(
\widehat\G_{t+\Delta}^{\mathrm{dir}},
\widetilde\G_{t+\Delta}^{\pi};
\Pi_{\Delta,\pi}^{G,\mathrm{dc}}
\right)
\right].
\label{eq:geometry-composition-loss}
\end{equation}
}
The first term compares the \emph{motion-group-level transformations}: it penalizes disagreement between the rigid motion decoded directly over the full horizon and the motion obtained by composing shorter decoded transformations. The second term compares the resulting \emph{Gaussian-level scene geometry}: it penalizes disagreement in matched centers, covariances, and existence probabilities. The coefficient $\lambda_{\mathrm{gc}}\geq0$ controls the relative strength of this Gaussian-level consistency term.

Thus, $\mathcal L_{\mathrm{gcomp}}$ in \cref{eq:geometry-composition-loss} is the geometric counterpart of the latent composition loss $\mathcal L_{\mathrm{comp}}$ in \cref{eq:composition-loss}: $\mathcal L_{\mathrm{comp}}$ aligns direct and composed paths in predictive latent space, whereas $\mathcal L_{\mathrm{gcomp}}$ encourages their decoded motion representations and Gaussian geometry to be path-consistent. This should be distinguished from the target-anchored geometry loss $\mathcal L_{\mathrm{geo}}$ in \cref{eq:geometry-loss}, which pulls each prediction route independently towards the future target geometry.

In the idealized zero-error case, if both the direct and composed decoded states exactly match the same target geometry under unambiguous correspondence, their Gaussian-level endpoint geometries necessarily agree, and the corresponding $D_{\mathrm{match}}$ term in \cref{eq:direct-composed-geometry-distance} becomes redundant. Its role is therefore not to provide an additional notion of target correctness, but to regularize the finite-error regime encountered in practice, where imperfect prediction, soft correspondence, noisy target tracking, and finite model capacity can leave two individually accurate paths with different residual errors\footnote{For example, consider a single Gaussian center with target $\mu^{+}=0$ in one dimension. A direct prediction $\widehat\mu^{\mathrm{dir}}=+\varepsilon$ and a composed prediction $\widetilde\mu^{\pi}=-\varepsilon$ have the same target error, $|\widehat\mu^{\mathrm{dir}}-\mu^{+}|=|\widetilde\mu^{\pi}-\mu^{+}|=\varepsilon$, yet disagree with each other by $|\widehat\mu^{\mathrm{dir}}-\widetilde\mu^{\pi}|=2\varepsilon$. Thus, small and even equal target-relative errors do not imply path agreement away from the zero-error limit.}. Direct--composed supervision explicitly suppresses this path-dependent discrepancy.

Moreover, agreement in endpoint Gaussian geometry does not necessarily imply agreement in the decoder's structured motion decomposition. Because the same endpoint geometry may be produced by different combinations of group-level $\SE(3)$ motion and Gaussian-specific residual deformation, the transformation distance $d_{\SE(3)}$ in \cref{eq:se3-distance}, as used in \cref{eq:geometry-composition-loss}, additionally encourages the underlying group transformations themselves to compose consistently across temporal paths. The target-anchored losses $\mathcal L_{\mathrm{end}}$, $\mathcal L_{\mathrm{path}}$, and $\mathcal L_{\mathrm{geo}}$ in \cref{eq:endpoint-loss,eq:path-loss,eq:geometry-loss}, respectively, therefore establish predictive accuracy, while $\mathcal L_{\mathrm{comp}}$ and $\mathcal L_{\mathrm{gcomp}}$ in \cref{eq:composition-loss,eq:geometry-composition-loss} act as complementary latent- and geometry-level path-consistency regularizers.

\paragraph{Structural dynamics priors.}

The selective decoder is additionally regularized by soft structural priors that encode simple properties expected of many dynamic scenes. Unlike the target-anchored geometry loss $\mathcal L_{\mathrm{geo}}$ in \cref{eq:geometry-loss}, these terms do not compare a prediction with future target geometry. Instead, \textit{they constrain how the decoded geometry is allowed to evolve relative to the causal source state}. They are evaluated for both direct decoded intervals $[t,t+s_j]$ and composed segments $[t+s_{j-1},t+s_j]$, and are then averaged over sampled horizons and paths.

For a decoded interval $[a,b]$, let $\Mset_{\mathrm{rigid},a}$ denote motion groups identified from the causal prefix as approximately rigid, and let $\mathcal E_{m,a}$ contain selected neighboring pairs of Gaussian primitives within rigid group $m$. A rigid group should preserve the distances between its constituent Gaussians, so we define
{\small
\begin{equation}
\ell_{\mathrm{rigid}}(a,b)
=
\sum_{m\in\Mset_{\mathrm{rigid},a}}
\sum_{(i,j)\in\mathcal E_{m,a}}
\left(
\|\widehat\mu_{i,b}-\widehat\mu_{j,b}\|_2
-
\|\mu_{i,a}-\mu_{j,a}\|_2
\right)^2.
\label{eq:rigid-loss}
\end{equation}
}
Here, $\mu_{i,a}$ and $\mu_{j,a}$ are the source centers, while $\widehat\mu_{i,b}$ and $\widehat\mu_{j,b}$ are the decoded centers obtained through the geometric center update by \cref{eq:decoded-center}. The rigidity loss $\ell_{\mathrm{rigid}}$ in \cref{eq:rigid-loss} therefore compares each pairwise distance \emph{before and after prediction}; it is small when the group moves without artificial stretching or compression. This prior complements the group-level $\SE(3)$ transformation defined in \cref{eq:decoded-group-transform} by encouraging \textit{approximately rigid groups to remain geometrically coherent after the Gaussian-specific residual displacement} in \cref{eq:decoded-center} is added.

The decoder also permits Gaussian-specific non-rigid residuals $\widehat r_{i,a\rightarrow b}$ through the decoder outputs in \cref{eq:decoder-outputs} and the center update in \cref{eq:decoded-center}. To prevent neighboring primitives from receiving unnecessarily irregular residual motions, let $\mathcal E_a$ denote a set of compatible neighboring Gaussian pairs at time $a$. We define
{\small
\begin{equation}
\ell_{\mathrm{smooth}}(a,b)
=
\sum_{(i,j)\in\mathcal E_a}
\kappa_{ij}
\left\|
\widehat r_{i,a\rightarrow b}
-
\widehat r_{j,a\rightarrow b}
\right\|_2^2,
\label{eq:smooth-loss}
\end{equation}
}
where $\kappa_{ij}\geq0$ is a compatibility weight derived from geometric proximity, semantic similarity, or estimated motion similarity. Hence, the smoothness loss $\ell_{\mathrm{smooth}}$ in \cref{eq:smooth-loss} encourages strongly compatible neighboring Gaussians to have similar local residual displacements, while allowing weakly related neighbors to deform more independently. Importantly, \cref{eq:smooth-loss} regularizes only the \emph{non-rigid residual} $\widehat r_{i,a\rightarrow b}$ rather than the complete motion, so the coherent group-level rigid transformation in \cref{eq:decoded-group-transform} is not penalized.

Finally, let $\Iset_{\mathrm{static},a}$ denote Gaussian primitives identified from the causal prefix as belonging to approximately static regions. Their predicted centers, obtained from \cref{eq:decoded-center}, should remain close to their source locations, giving
{\small
\begin{equation}
\ell_{\mathrm{static}}(a,b)
=
\sum_{i\in\Iset_{\mathrm{static},a}}
\left\|
\widehat\mu_{i,b}-\mu_{i,a}
\right\|_2^2.
\label{eq:static-loss}
\end{equation}
}
Thus, the static-region loss $\ell_{\mathrm{static}}$ in \cref{eq:static-loss} directly penalizes drift of geometry that the available causal evidence indicates should remain stationary.

Let $\mathcal L_{\mathrm{rigid}}$, $\mathcal L_{\mathrm{smooth}}$, and $\mathcal L_{\mathrm{static}}$ denote the corresponding averages of \cref{eq:rigid-loss,eq:smooth-loss,eq:static-loss}, respectively, over the sampled direct intervals and composed segments. Their weighted combination is
{\small
\begin{equation}
\mathcal L_{\mathrm{struct}}
=
\lambda_{\mathrm{rigid}}\mathcal L_{\mathrm{rigid}}
+
\lambda_{\mathrm{smooth}}\mathcal L_{\mathrm{smooth}}
+
\lambda_{\mathrm{static}}\mathcal L_{\mathrm{static}},
\label{eq:structural-objective}
\end{equation}
}
where $\lambda_{\mathrm{rigid}},\lambda_{\mathrm{smooth}},\lambda_{\mathrm{static}}\geq0$ control the strengths of rigidity preservation, local residual smoothness, and static-region stability, respectively. In summary, $\mathcal L_{\mathrm{rigid}}$ derived from \cref{eq:rigid-loss} discourages deformation within approximately rigid groups, $\mathcal L_{\mathrm{smooth}}$ derived from \cref{eq:smooth-loss} discourages spatially irregular non-rigid residual motion, and $\mathcal L_{\mathrm{static}}$ derived from \cref{eq:static-loss} discourages drift of approximately stationary geometry. These three terms are combined into the structural regularizer $\mathcal L_{\mathrm{struct}}$ in \cref{eq:structural-objective}.

These terms are soft priors rather than universal physical constraints. The sets $\Mset_{\mathrm{rigid},a}$ and $\Iset_{\mathrm{static},a}$, as well as the compatibility weights $\kappa_{ij}$ appearing in \cref{eq:rigid-loss,eq:static-loss,eq:smooth-loss}, are determined only from information available in the causal prefix. Rigidity and static penalties in \cref{eq:rigid-loss,eq:static-loss} are weakened or disabled when prefix-derived evidence indicates non-rigid deformation, independent motion, or topology change. Thus, the structural regularizer in \cref{eq:structural-objective} constrains decoded dynamics without introducing future information into the prediction branch.

\subsection{Representation regularization and full objective}
\label{sec:full-objective}

An EMA target does not by itself preclude all degenerate representation solutions. We therefore apply variance and covariance regularization to scene-, motion-group-, and Gaussian-level student embeddings, following the general principle of VICReg \citep{bardes2022vicreg}. We write
\begin{equation}
\mathcal{L}_{\mathrm{rep}}
=
\mathcal{L}_{\mathrm{var}}
+
\lambda_{\mathrm{cov}}
\mathcal{L}_{\mathrm{cov}},
\label{eq:representation-regularization}
\end{equation}
where $\mathcal L_{\mathrm{var}}$ prevents representation collapse by encouraging each embedding dimension to maintain non-trivial variation across the training batch, while $\mathcal L_{\mathrm{cov}}$ discourages different embedding dimensions from carrying redundant information by penalizing off-diagonal covariance. The coefficient $\lambda_{\mathrm{cov}}\geq0$ controls the relative strength of the covariance penalty. These regularizers are applied separately at the scene, motion-group, and Gaussian levels of the hierarchical state in \cref{eq:hier-state}. Unlike the predictive losses above, $\mathcal L_{\mathrm{rep}}$ does not specify what future state should be predicted; rather, it constrains the representation itself so that predictive agreement cannot be achieved through a collapsed or highly redundant latent space. The precise hierarchy-level definitions are given in \cref{app:regularization}.

The \textit{complete training objective} is
\begin{align}
\mathcal{L}
&=
\mathcal{L}_{\mathrm{end}}
+
\lambda_{\mathrm{path}}
\mathcal{L}_{\mathrm{path}}
+
\lambda_{\mathrm{comp}}
\mathcal{L}_{\mathrm{comp}}
+
\lambda_{\mathrm{geo}}
\mathcal{L}_{\mathrm{geo}}
\nonumber\\
&\quad+
\lambda_{\mathrm{gcomp}}
\mathcal{L}_{\mathrm{gcomp}}
+
\mathcal{L}_{\mathrm{struct}}
+
\lambda_{\mathrm{rep}}
\mathcal{L}_{\mathrm{rep}},
\label{eq:full-objective}
\end{align}
where all loss weights are non-negative. The default model contains no pixel-rendering loss and no objective requiring reconstruction of complete future appearance. A rendering-only head may be trained separately for evaluation to measure how much future visual information remains recoverable from the learned predictive state.

\paragraph{Optimization.}

The generic training procedure is summarized in Algorithm~\ref{alg:training}. The controlled experiments in \cref{sec:experiments} instantiate the corresponding ablation objectives and optimize them directly from initialization rather than using a staged training curriculum. Exact model, optimization, regularization, checkpoint-selection, and experiment-specific objective settings are provided in \cref{app:training-details}.

\section{Gaussian Identity and Hybrid Correspondence}
\label{sec:identity}

The losses introduced above require correspondences between Gaussian primitives or motion groups. In particular, the hierarchical latent distance $\dhier$ in \cref{eq:hier-distance} uses the group- and Gaussian-level correspondence matrices $\Pi^O$ and $\Pi^G$, the target-anchored geometry loss $\mathcal L_{\mathrm{geo}}$ in \cref{eq:geometry-loss-local} uses $\Pi^G$ to compare predicted and target primitives, and the direct--composed geometry discrepancy $D_{\mathrm{match}}$ in \cref{eq:direct-composed-geometry-distance} uses an analogous direct--composed plan $\Pi^{G,\mathrm{dc}}$. These correspondences cannot in general be obtained by simply matching array indices.

\paragraph{Why Gaussian correspondence is non-trivial.}

Recall from \cref{eq:gaussian-state} that a Gaussian scene is represented as a set of primitives. Its raw storage order does not define canonical identity or correspondence. For example, the same three persistent primitives may be stored as
\begin{equation}
[A,B,C]
\qquad\longrightarrow\qquad
[C,A,B],
\label{eq:identity-permutation-example}
\end{equation}
without any change in the represented scene. Naive index matching would nevertheless compare $A$ with $C$, $B$ with $A$, and $C$ with $B$. We therefore distinguish a primitive's \emph{raw array position} from its \emph{persistent canonical identity}:
\begin{equation}
\text{raw array position}
\;\neq\;
\text{persistent canonical identity}.
\label{eq:index-vs-identity}
\end{equation}

This distinction directly affects predictive learning. If two equivalent Gaussian states use different primitive orderings, an index-wise latent or geometry loss may penalize an otherwise correct prediction. More generally, a Gaussian set is not uniquely parameterized: densification may split one primitive into several, pruning may remove primitives, and independent reconstructions may reorder, merge, or relocate primitives while representing nearly the same scene.

The correspondence problem becomes still more general under topology-changing evolution. A primitive may split, several primitives may merge, a primitive may disappear, or newly represented geometry may appear:
\begin{equation*}
\begin{aligned}
A &\longrightarrow \{A_1,A_2\},
&&\text{one-to-many split},
\\
\{A,B\} &\longrightarrow C,
&&\text{many-to-one merge},
\\
A &\longrightarrow \varnothing,
&&\text{death / disappearance},
\\
\varnothing &\longrightarrow F,
&&\text{birth / appearance}.
\end{aligned}
\end{equation*}
Correspondence is therefore not always one-to-one: it may be one-to-many, many-to-one, one-to-zero, or zero-to-one. A raw index-wise comparison between arbitrary dynamic Gaussian states is consequently not well defined.

We address this with a \emph{hybrid correspondence mechanism}. Persistent canonical identities are used whenever they remain valid, motion-group correspondence provides a more stable intermediate constraint, and optimal transport (OT) handles residual ambiguous or topology-changing geometry. The design principle is
\begin{equation}
\begin{aligned}
\text{reliable persistent canonical identity}
&\;\Rightarrow\;
\text{fixed correspondence},
\\
\text{uncertain or invalid canonical identity}
&\;\Rightarrow\;
\text{residual soft matching}.
\end{aligned}
\label{eq:hybrid-identity-principle}
\end{equation}

\paragraph{Prefix-anchored canonical bank.}

Our primary representation uses a prefix-anchored canonical Gaussian bank
\begin{equation}
\G_t^{\mathrm{can}}
=
\{g_{i,t}^{\mathrm{can}}\}_{i=1}^{N},
\label{eq:canonical-bank}
\end{equation}
constructed using only observations available up to the context endpoint $t$. Here, $i\in\{1,\ldots,N\}$ denotes a \emph{persistent canonical slot identifier}, and $N$ is the fixed slot capacity of the bank, including both active primitives and reserved inactive slots. This differs from both the raw array position discussed in \cref{eq:index-vs-identity} and $N_t$ in \cref{eq:gaussian-state}, which denotes the number of active scene primitives at time $t$. Persistent canonical identity therefore means that a slot identifier is deliberately preserved across time while its correspondence remains valid; it does not require the Gaussian geometry associated with that slot to remain fixed. When this canonical-bank representation is used, the active Gaussian scene $\G_t$ in \cref{eq:gaussian-state} is the active subset of $\G_t^{\mathrm{can}}$, while the remaining slots are reserved inactive slots available for future activation.

For target construction, the canonical bank $\G_t^{\mathrm{can}}$ is cloned at the context endpoint $t$ and tracked forward using only observations available up to each future target time, following the causal protocol in \cref{eq:causal-target-builder}. Its parameters at or before $t$ are never retroactively revised using future observations. Within each causal prefix, any densification or refinement used to construct the canonical bank must therefore be completed before target tracking; future observations are not allowed to alter the bank or its states at or before $t$. Ordinary future evolution is then represented through deformation and existence of persistent slots rather than uncontrolled changes in indexing.

Newly represented geometry may be handled by a reserved pool of inactive canonical slots, consistent with the decoder-side set $\Iset_a^{\mathrm{dec}}$ introduced in \cref{eq:decoder-outputs}. Reserved slots may carry motion-group or spatial anchors that provide an initial prior for future activation. When activated, the predicted local residual $\widehat r_{i,a\rightarrow b}$, deformation $\widehat U_{i,a\rightarrow b}$, and existence probability $\widehat e_{i,b}$ determine the slot's future geometry through \cref{eq:decoded-center,eq:decoded-cov}. The forward target tracker may similarly associate emerging target geometry with available inactive slots using causal group, spatial, or semantic affinity before invoking residual OT. Persistent canonical identities therefore provide a stable causal bookkeeping mechanism without relying on identities obtained by fitting the complete trajectory.

A persistent match is treated as valid only when the two states carry the same persistent canonical identity and that identity remains unique and valid in both states. For two Gaussian states $\G$ and $\G'$, let
\begin{equation}
\mathcal P_{\mathrm{pers}}(\G,\G')
=
\left\{
(i,j):
\begin{array}{l}
\text{slots $i$ and $j$ share the same persistent}\\
\text{canonical identity and the match remains unique}
\end{array}
\right\}.
\label{eq:valid-persistent-pairs}
\end{equation}
For prediction--target comparison, validity is supplied by the forward target tracker; for direct--composed comparison, it is inherited from the common prefix-anchored canonical bank. A split, merge, canonical reassignment, or otherwise uncertain correspondence invalidates the corresponding hard one-to-one match, and that mass is instead passed to residual matching. Temporary occlusion need not invalidate canonical identity when the causal tracker can maintain the association through the occluded interval.

\paragraph{Motion-group correspondence.}

Primitive-level identities may become uncertain under local refinement, occlusion, split, or merge events. Motion groups therefore provide a more stable intermediate unit of correspondence. Recall that $m(i)$ denotes the motion group associated with Gaussian $i$. In the general formulation, groups may be formed and tracked causally using geometry, motion, semantics, or their combination.

The group-level correspondence matrix $\Pi^O$ used in \cref{eq:hier-distance} records correspondence between motion groups in the two states under comparison. Persistent group identities give hard correspondences when available, while uncertain cases may receive soft correspondence mass. Retaining a motion-group identity even when its supporting Gaussian primitives change allows the group-level representation to remain stable through local split, merge, birth, and death events.

To expose this information to Gaussian-level matching, let
\begin{equation}
\kappa^O_{mn}\in[0,1]
\label{eq:group-compatibility}
\end{equation}
denote the compatibility between group $m$ in the first state and group $n$ in the second. Persistent group identity gives $\kappa^O_{mn}\in\{0,1\}$ as a hard special case, while uncertain group correspondence may produce intermediate values derived from normalized correspondence mass in $\Pi^O$. Gaussian primitives belonging to strongly corresponding groups are therefore preferred without requiring group compatibility itself to be binary.

\paragraph{Hybrid Gaussian correspondence.}

At the Gaussian level, valid persistent canonical identities are applied first. Let $\Pi^{G,\mathrm{pers}}$ denote a partial transport plan supported only on the valid persistent pairs in \cref{eq:valid-persistent-pairs}. For source and target marginal masses $\mathbf p$ and $\mathbf q$, respectively, the persistent plan satisfies
\begin{equation}
\Pi_{ij}^{G,\mathrm{pers}}=0
\quad\text{for}\quad
(i,j)\notin\mathcal P_{\mathrm{pers}},
\qquad
\Pi^{G,\mathrm{pers}}\mathbf 1\leq\mathbf p,
\qquad
(\Pi^{G,\mathrm{pers}})^\top\mathbf 1\leq\mathbf q.
\label{eq:persistent-partial-plan}
\end{equation}

Only the remaining unmatched mass is passed to OT. Define the residual marginals
\begin{equation}
\mathbf p_{\mathrm{res}}
=
\mathbf p-\Pi^{G,\mathrm{pers}}\mathbf 1,
\qquad
\mathbf q_{\mathrm{res}}
=
\mathbf q-(\Pi^{G,\mathrm{pers}})^\top\mathbf 1.
\label{eq:residual-marginals}
\end{equation}
The residual correspondence plan $\Pi^{G,\mathrm{res}}$ operates only on this remaining mass. The complete Gaussian correspondence is therefore
\begin{equation}
\Pi^G
=
\Pi^{G,\mathrm{pers}}
+
\Pi^{G,\mathrm{res}},
\label{eq:hybrid-correspondence}
\end{equation}
where persistent and residual correspondence operate on disjoint mass. Equivalently, the matching hierarchy is
\begin{equation}
\text{persistent canonical identity}
\;\longrightarrow\;
\text{motion-group compatibility}
\;\longrightarrow\;
\text{OT for residual ambiguity}.
\label{eq:hybrid-matching-hierarchy}
\end{equation}

To define the residual OT stage independently of the particular two Gaussian states being compared, consider arbitrary states
\begin{equation*}
\G=\{g_i\}_{i=1}^{N},
\qquad
\G'=\{g_j'\}_{j=1}^{N'}.
\end{equation*}
Let $m(i)$ and $m'(j)$ denote the motion groups of primitive $i$ in $\G$ and primitive $j$ in $\G'$, respectively. For a residual source primitive $i$ and residual target primitive $j$, we use the generic cost
{\small
\begin{equation}
C_{ij}(\G,\G')
=
\beta_\mu\|\mu_i-\mu_j'\|_2^2
+
\beta_f\cosdist(f_i,f_j')
+
\beta_\Sigma\dspd(\Sigma_i,\Sigma_j')
+
\beta_O\!\left(
1-\kappa^O_{m(i),m'(j)}
\right),
\label{eq:ot-cost}
\end{equation}
}
where $\mu_i$ and $\mu_j'$ are the Gaussian centers, $\Sigma_i$ and $\Sigma_j'$ are their covariances, and $f_i$ and $f_j'$ are semantic features when available. The semantic term is omitted when such features are unavailable. The function $\cosdist$ is the cosine distance defined in \cref{sec:hierarchical-distance}, while $\dspd$ is the SPD covariance distance used in \cref{eq:geometry-loss-local}. The non-negative coefficients $\beta_\mu,\beta_f,\beta_\Sigma,\beta_O$ control the contributions of position, semantics, covariance geometry, and motion-group compatibility, respectively.

Thus, $C_{ij}(\G,\G')$ is small when two primitives are spatially close, semantically compatible, geometrically similar, and associated with strongly corresponding motion groups. Hard group identity is recovered when $\kappa^O$ is binary, whereas soft group correspondence yields a graded compatibility penalty.

Rather than forcing a single lowest-cost assignment, we obtain a soft residual correspondence through entropically regularized optimal transport following the Lagrangian Sinkhorn formulation\footnote{Cuturi writes the regularized objective as $\langle P,M\rangle-\lambda^{-1}h(P)$ with $h(P)=-\sum_{ij}P_{ij}\log P_{ij}$. Setting $\varepsilon_{\mathrm{OT}}=\lambda^{-1}$ gives the form used here up to $-\varepsilon_{\mathrm{OT}}\sum_{ij}P_{ij}$, which is constant under fixed transport marginals and therefore does not change the optimizer.} of \citet[Sec.~4, Eq.~(2)]{cuturi2013sinkhorn}.
{\small
\begin{equation}
\Pi^{G,\mathrm{res}}
=
\arg\min_{\Pi\in\mathcal U(\mathbf p_{\mathrm{res}},\mathbf q_{\mathrm{res}})}
\left\{
\langle \Pi,C\rangle
+
\varepsilon_{\mathrm{OT}}
\sum_{i,j}
\Pi_{ij}\bigl(\log\Pi_{ij}-1\bigr)
\right\},
\label{eq:residual-ot}
\end{equation}
}
where $\mathcal U(\mathbf p_{\mathrm{res}},\mathbf q_{\mathrm{res}})$ is the set of non-negative transport plans with residual marginals $\mathbf p_{\mathrm{res}}$ and $\mathbf q_{\mathrm{res}}$, and $\varepsilon_{\mathrm{OT}}>0$ controls entropy regularization. The first term favors low-cost geometric, semantic, and group-compatible matches, while the entropy term permits soft correspondence when identity is ambiguous. We solve \cref{eq:residual-ot} with Sinkhorn iterations after fixing all valid persistent correspondences.

A dustbin row and column augment the residual transport problem in \cref{eq:residual-ot} to absorb genuinely unmatched mass. When dustbins are present, $\mathbf p_{\mathrm{res}}$ and $\mathbf q_{\mathrm{res}}$ in \cref{eq:residual-ot} denote the corresponding augmented residual marginals, including their dustbin masses. The target-side dustbin column represents predicted primitives that cannot be matched to represented future geometry, corresponding to disappearance, deactivation, or otherwise unmatched predicted slots; this is the same dustbin $\varnothing$ used to construct the target existence value in \cref{eq:target-existence}. Conversely, the predicted-side dustbin row absorbs target geometry that cannot be assigned to an existing or reserved inactive canonical slot. Newly represented geometry is first allowed to match an available inactive slot, so the dustbin acts as a fallback for genuinely unmatched mass rather than the primary birth mechanism.

Because $\Pi^{G,\mathrm{res}}$ in \cref{eq:residual-ot} is soft, it permits fractional mass assignments compatible with non-one-to-one correspondence. One predicted primitive may distribute residual mass over several target primitives during a split, while several predicted primitives may place residual mass on the same target lineage during a merge. This flexibility does not by itself identify the true lineage; rather, it avoids imposing an incorrect one-to-one correspondence when the available geometric, semantic, and group evidence is ambiguous. Persistent matches remain fixed, and only residual unmatched mass participates in this soft matching.

The experiment-specific instantiation of the residual transport problem, including its count-aware marginals, dustbin cost, Sinkhorn settings, and event definitions, is given in \cref{app:matching}.

\paragraph{Correspondence for latent and decoded states.}

The correspondence mechanism is defined from the associated Gaussian slots, canonical identities, motion-group information, and Gaussian geometry; it is not inferred from the latent discrepancy that it is subsequently used to evaluate. Consequently, when the Gaussian-level latent distance in \cref{eq:hier-distance} compares tokens $z_i^G$ and $z_j'^G$, the correspondence matrix $\Pi^G$ is inherited from the associated Gaussian states or canonical slot assignments. This avoids a circular construction in which latent similarity would determine the same correspondence under which latent similarity is measured.

The same hybrid construction is used whenever two Gaussian sets require correspondence. For prediction--target matching,
\begin{equation}
(\G,\G')
=
\left(
\widehat\G,
\G^{+}
\right),
\label{eq:prediction-target-correspondence-pair}
\end{equation}
yielding $\Pi^G$ as used in \cref{eq:hier-distance,eq:geometry-loss-local}. For direct--composed matching,
\begin{equation}
(\G,\G')
=
\left(
\widehat\G^{\mathrm{dir}},
\widetilde\G^{\pi}
\right),
\label{eq:direct-composed-correspondence-pair}
\end{equation}
yielding the analogous plan $\Pi^{G,\mathrm{dc}}$ used in \cref{eq:direct-composed-geometry-distance}. The correspondence machinery is therefore shared; only the two Gaussian states supplied to it differ.

By default, the resulting correspondence plans $\Pi^G$ and $\Pi^{G,\mathrm{dc}}$ are treated as matching assignments when used inside the predictive and geometric objectives: gradients from those losses are not propagated through the construction of the correspondence plans. This separates optimization of the predictive model from optimization of the matching rule and prevents the model from reducing a prediction loss merely by altering the correspondence itself. A differentiable Sinkhorn matcher could instead be used when end-to-end optimization of correspondence is desired.

\paragraph{Correspondence-aware comparison.}

Persistent canonical identity provides useful correspondence when it remains valid, but Gaussian predictions should not be rewarded or penalized for reproducing an arbitrary storage ordering. The resulting comparison is therefore invariant to raw array permutation while remaining sensitive to reliable persistent canonical correspondence. When such identity becomes uncertain or invalid, the hybrid mechanism in \cref{eq:hybrid-identity-principle} progressively falls back to residual OT rather than imposing an incorrect index-wise match.

\section{Experiments}
\label{sec:experiments}

The initial experiments provide controlled proof-of-concept tests of the mechanisms most specific to 4DGS-JEPA. Rather than attempting a large-scale comparison across dynamic-scene reconstruction, video world modelling, representation probing, and robotic planning, we focus on three questions directly implied by the Method: (i) \textit{does multi-horizon path supervision reduce recursive rollout drift, and does temporal composition reduce dependence on the particular rollout path}; (ii) \textit{does latent path consistency transfer to explicit Gaussian motion, and is this transfer strengthened by geometry-level composition}; and (iii) \textit{does hybrid correspondence remain reliable under primitive reordering and local topology changes}?

These questions isolate the roles of $\mathcal L_{\mathrm{path}}$, $\mathcal L_{\mathrm{comp}}$, $\mathcal L_{\mathrm{geo}}$, $\mathcal L_{\mathrm{gcomp}}$, and the hybrid correspondence mechanism introduced in \cref{sec:path-supervision,sec:temporal-composition,sec:geometry-decoder,sec:geometry-regularization,sec:identity}. We intentionally use a compact controlled Gaussian-world benchmark and small models that can be trained on a single GPU. Large-scale dynamic-scene benchmarks, comparisons with image- and video-space world models, representation probes, and action-conditioned planning are left to future evaluation.

\subsection{Controlled dynamic Gaussian worlds}
\label{sec:controlled-gaussian-data}

We procedurally generate trajectories of explicit anisotropic 3D Gaussian scenes. This provides exact future geometry, motion-group identity, and underlying transformations while avoiding the confounding effects and computational cost of fitting a dynamic Gaussian reconstruction model before evaluating predictive dynamics.

Each scene contains a small number of motion groups together with optional static background Gaussians, with state representation given by \cref{eq:gaussian-state}. Each motion group contains anisotropic Gaussians sampled around a compact 3D structure. Its evolution combines a group-level rigid transformation with Gaussian-specific local deformation,
\begin{equation}
\mu_{i,t}
=
R_{m(i),t}\mu_{i,0}
+
d_{m(i),t}
+
r_{i,t},
\label{eq:synthetic-center-dynamics}
\end{equation}
and
\begin{equation}
\Sigma_{i,t}
=
R_{m(i),t}
U_{i,t}
\Sigma_{i,0}
U_{i,t}^{\top}
R_{m(i),t}^{\top},
\label{eq:synthetic-cov-dynamics}
\end{equation}
mirroring the structured decoder in \cref{eq:decoded-center,eq:decoded-cov}. The group motions include translations and rotations with randomly sampled velocities and accelerations, while $r_{i,t}$ and $U_{i,t}$ introduce smooth Gaussian-specific non-rigid deformation. Some groups are designated static. Dynamics parameters, initial Gaussian configurations, and group motions are sampled independently for the training, validation, and test sets, so test trajectories are not continuations of training trajectories.

Experiments~1 and~2 use geometry and known motion-group identities, with persistent canonical identities fixed and all active primitives retained throughout each trajectory. Existence is therefore fixed to one in these forecasting experiments. Split, merge, birth, death, and permutation events are introduced separately in Experiment~3, allowing prediction and composition to be evaluated independently of correspondence ambiguity.

Because the trajectories are generated directly in Gaussian space, the online context builder and forward target tracker in \cref{eq:causal-context-builder,eq:causal-target-builder} reduce to causal access to the generated Gaussian states. No future state is used to modify the context representation.

Experiments~1 and~2 use a context history of $L=4$ states. Let
\[
\Kset_{\mathrm{train}}
=
\{1,2,4,8\}
\]
denote the set of terminal prediction horizons used during training, where $\Delta\in\Kset_{\mathrm{train}}$ is the temporal offset, measured in future time steps, from the context endpoint $t$ to the prediction target at $t+\Delta$. Experiment~1 additionally evaluates interpolation at unseen terminal horizons $\{3,6\}$ and tests held-out temporal decompositions containing either seen or unseen segment lengths.

\paragraph{Small-model implementation.}

The proof-of-concept model instantiates the components of \cref{sec:method} at deliberately small scale. Gaussian primitives are encoded into local tokens, pooled into motion-group tokens, and then into a scene token according to \cref{eq:hier-state}. The encoder and horizon-conditioned transition $\Phi_\phi$ use low-dimensional embeddings and shallow attention blocks. The controlled selective decoder $Q_\psi^\G$ predicts the group-level rigid increment $\widehat\xi_{m,a\rightarrow b}$, local residual displacement $\widehat r_{i,a\rightarrow b}$, and covariance deformation through $\widehat S_{i,a\rightarrow b}$. Because all active primitives persist in Experiments~1 and~2, existence is fixed to one and the existence head $\widehat e_{i,b}$ of the full formulation in \cref{eq:decoder-outputs} is not exercised in these forecasting experiments. Representation regularization $\mathcal L_{\mathrm{rep}}$ in \cref{eq:representation-regularization} is retained in every learned variant.

Directly comparable variants use the same model capacity, data-generation protocol, optimizer, horizon schedule, temporal partitions, and seed set, while each model is independently initialized and trained. Exact architecture, optimization, regularization, checkpoint-selection, and runtime details are provided in \cref{app:training-details}.

\subsection{Experiment 1: path supervision and temporal composition}
\label{sec:experiment-temporal-composition}

Experiment~1 isolates the two temporal-learning mechanisms introduced in \cref{sec:path-supervision,sec:temporal-composition}: recurrent target anchoring through $\mathcal L_{\mathrm{path}}$ in \cref{eq:path-loss} and explicit direct--composed agreement through $\mathcal L_{\mathrm{comp}}$ in \cref{eq:composition-loss}. All variants use the same encoder, horizon-conditioned transition $\Phi_\phi$, selective decoder $Q_\psi^\G$, EMA target encoder, training data, optimization budget, direct-route geometry supervision, and representation regularization. The variant names below therefore refer only to the temporal latent losses being ablated:
\begin{enumerate}[leftmargin=*,itemsep=1pt]
    \item \textit{Endpoint:} temporal latent objective $\mathcal L_{\mathrm{end}}$ in \cref{eq:endpoint-loss};
    \item \textit{Endpoint + Path:} $\mathcal L_{\mathrm{end}}+\lambda_{\mathrm{path}}\mathcal L_{\mathrm{path}}$;
    \item \textit{Endpoint + Composition:} $\mathcal L_{\mathrm{end}}+\lambda_{\mathrm{comp}}\mathcal L_{\mathrm{comp}}$;
    \item \textit{Endpoint + Path + Composition:} all three temporal losses are active.
\end{enumerate}

Every variant additionally receives the same \emph{direct-route} target-anchored geometry supervision from $\mathcal L_{\mathrm{geo}}$ in \cref{eq:geometry-loss}; composed geometry is not directly matched to the future target in this experiment. The Endpoint baseline therefore receives no recurrent target supervision through the decoder. Geometry-level composition $\mathcal L_{\mathrm{gcomp}}$ in \cref{eq:geometry-composition-loss} is disabled and studied separately in \cref{sec:experiment-geometry-composition}. The exact implemented objectives and weights are given in \cref{app:exp1-details}.

For calibration, we additionally report two causal analytic baselines: \emph{static persistence}, which keeps the current Gaussian centers fixed, and \emph{constant velocity}, which extrapolates the displacement between the final two context states. Both use only information available at or before the context endpoint.

\paragraph{Metrics.}

For target-relative geometry accuracy, we report the \emph{final displacement error}
\begin{equation}
\mathrm{FDE}(\widehat\G,\G^{+})
=
\frac{1}{|\Iset|}
\sum_{i\in\Iset}
\|\widehat\mu_i-\mu_i^{+}\|_2,
\label{eq:experiment-fde}
\end{equation}
using the known persistent correspondences in this experiment. We distinguish \emph{direct FDE}, obtained from $\widehat\G_{t+\Delta}^{\mathrm{dir}}$, from \emph{rollout FDE}, obtained from $\widetilde\G_{t+\Delta}^{\pi}$. The additional error associated with recursive prediction is summarized by
\begin{equation}
\Delta_{\mathrm{roll}}
=
\mathrm{FDE}_{\mathrm{roll}}
-
\mathrm{FDE}_{\mathrm{dir}}.
\label{eq:rollout-gap}
\end{equation}

For latent path consistency, we report
\begin{equation}
\mathrm{CompErr}(\Delta,\pi)
=
\dhier\!\left(
\widehat Z_{t+\Delta}^{\mathrm{dir}},
\widetilde Z_{t+\Delta}^{\pi}
\right),
\label{eq:comp-metric}
\end{equation}
together with the center-level direct--composed geometric discrepancy
\begin{equation}
\mathrm{GeoCompErr}(\Delta,\pi)
=
\frac{1}{|\Iset|}
\sum_{i\in\Iset}
\left\|
\widehat\mu_{i,t+\Delta}^{\mathrm{dir}}
-
\widetilde\mu_{i,t+\Delta}^{\pi}
\right\|_2.
\label{eq:geometry-comp-metric}
\end{equation}
Composition metrics are computed only for non-trivial paths with more than one segment; the one-segment case is excluded because the direct and composed computations coincide. We report composition metrics jointly with target-relative FDE because small path discrepancy alone does not imply correct future prediction.

We evaluate supervised horizons $\{1,2,4,8\}$ and unseen interpolation horizons $\{3,6\}$. We additionally distinguish held-out decompositions composed entirely of segment lengths observed during training from a harder setting containing unseen segment lengths. For example, at $\Delta=8$, training uses $(4,4)$ and $(2,2,4)$, while held-out seen-segment decompositions include $(4,2,2)$, $(2,4,2)$, and $(2,2,2,2)$; unseen-segment tests include $(3,5)$ and $(5,3)$.

\begin{table}[t]
\caption{Experiment~1 results averaged over 3 random seeds. Direct FDE, rollout FDE, and $\Delta_{\mathrm{roll}}$ are averaged over supervised horizons $\{1,2,4,8\}$; composition metrics are averaged over non-trivial supervised paths. Unseen FDE is rollout FDE over the held-out horizons $\{3,6\}$. Lower is better for all metrics.}
\label{tab:small-temporal-results}
\centering
\tiny
\setlength{\tabcolsep}{2.5pt}
\begin{tabular}{lcccccc}
\toprule
Method
& Direct FDE $\downarrow$
& Rollout FDE $\downarrow$
& $\Delta_{\mathrm{roll}}$ $\downarrow$
& CompErr $\downarrow$
& GeoCompErr $\downarrow$
& Unseen FDE $\downarrow$ \\
\midrule
Endpoint
& $0.03013{\pm}0.00042$
& $0.03061{\pm}0.00065$
& $0.00048{\pm}0.00023$
& $0.00244{\pm}0.00055$
& $0.00739{\pm}0.00082$
& $0.03621{\pm}0.00080$
\\
Endpoint + Path
& $0.03195{\pm}0.00243$
& $0.03228{\pm}0.00256$
& $0.00033{\pm}0.00014$
& $0.00065{\pm}0.00007$
& $\mathbf{0.00649{\pm}0.00069}$
& $0.03838{\pm}0.00304$
\\
Endpoint + Composition
& $\mathbf{0.02905{\pm}0.00095}$
& $\mathbf{0.02940{\pm}0.00101}$
& $0.00035{\pm}0.00014$
& $0.00096{\pm}0.00011$
& $0.00667{\pm}0.00074$
& $\mathbf{0.03466{\pm}0.00122}$
\\
Endpoint + Path + Composition
& $0.03206{\pm}0.00245$
& $0.03236{\pm}0.00255$
& $\mathbf{0.00029{\pm}0.00013}$
& $\mathbf{0.00058{\pm}0.00003}$
& $0.00653{\pm}0.00089$
& $0.03837{\pm}0.00283$
\\
\bottomrule
\end{tabular}
\end{table}

\paragraph{Temporal composition improves consistency without sacrificing accuracy.}

Relative to Endpoint, Endpoint + Composition reduces direct FDE from $0.03013$ to $0.02905$ and rollout FDE from $0.03061$ to $0.02940$, improvements of $3.6\%$ and $4.0\%$, respectively. CompErr in \cref{eq:comp-metric} simultaneously decreases from $0.00244$ to $0.00096$, a $60.6\%$ reduction. Thus, the improvement in path consistency does not come at the expense of target-relative accuracy.

The horizon-wise results in \cref{fig:exp1-horizon-results} separate target accuracy from path consistency. Panel~(a) asks whether the model predicts the correct future Gaussian state, whereas panel~(b) asks whether alternative temporal routes produce compatible latent states. The two quantities must be interpreted jointly: low CompErr can arise from an approximately static transition, while low FDE alone does not imply route-independent prediction.

Panel~(a) rules out a trivial near-static explanation. Endpoint + Composition reduces average seen-horizon direct FDE from $0.14802$ for static persistence to $0.02905$. Constant-velocity extrapolation is a strong baseline on this deliberately smooth benchmark, attaining lower average seen-horizon FDE ($0.02484$), particularly at short horizons. At the longest supervised horizon $\Delta=8$, however, Endpoint + Composition obtains FDE $0.05687$, compared with $0.07120$ for constant velocity. Thus, at longer horizons where nonlinear motion accumulates, the learned transition $\Phi_\phi$ captures predictive structure not represented by local constant-velocity extrapolation.

Panel~(b) shows the complementary effect on path dependence. Endpoint supervision alone leaves substantially greater direct--composed discrepancy. Adding either $\mathcal L_{\mathrm{path}}$ or $\mathcal L_{\mathrm{comp}}$ reduces CompErr, while their combination yields the smallest discrepancy across the evaluated non-trivial horizons. Improved path consistency is therefore observed together with non-trivial target-relative dynamics.

\begin{figure}[t]
\centering
\begin{minipage}[t]{0.49\linewidth}
\centering
\includegraphics[width=\linewidth]{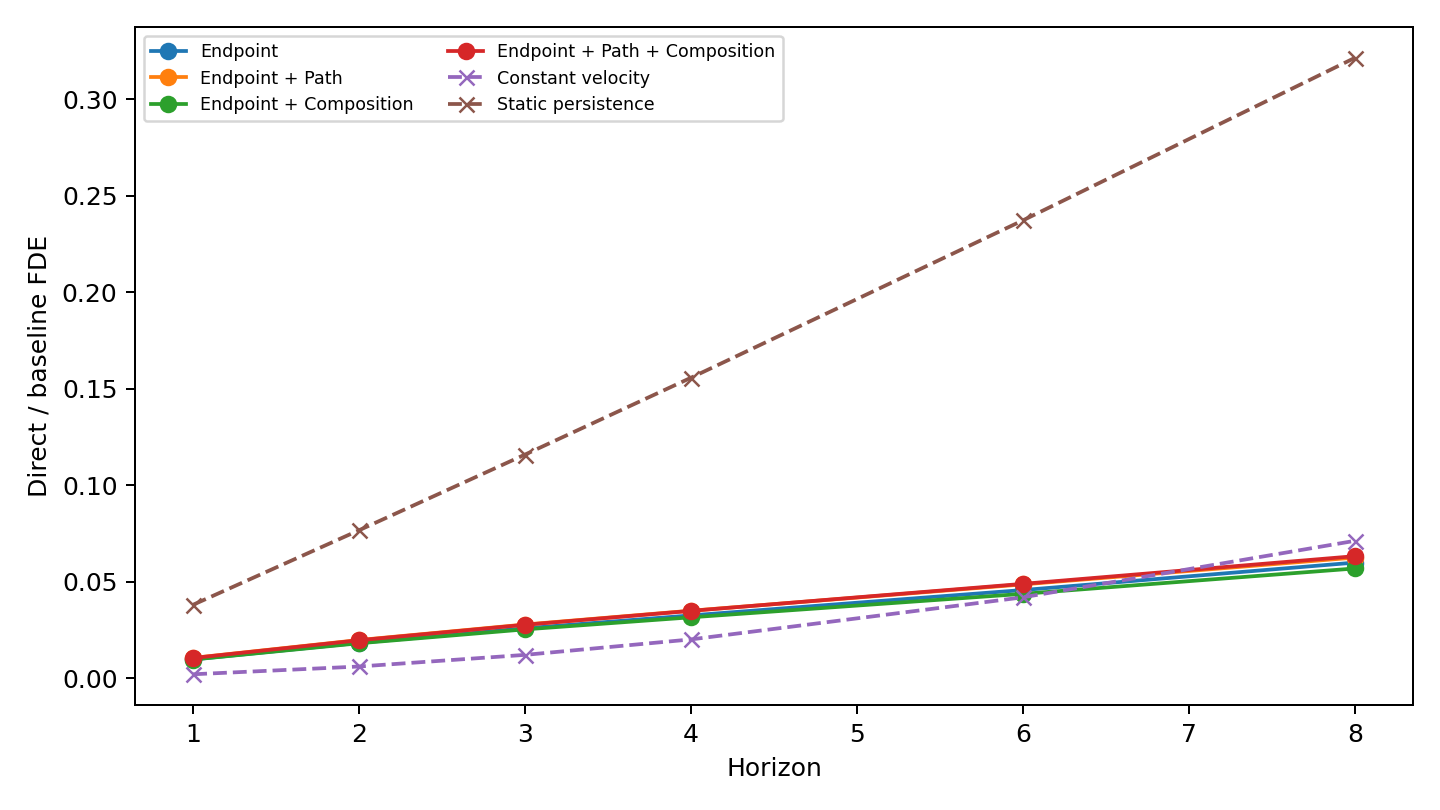}
\vspace{-1mm}

{\small (a) Target-relative direct FDE.}
\end{minipage}
\hfill
\begin{minipage}[t]{0.49\linewidth}
\centering
\includegraphics[width=\linewidth]{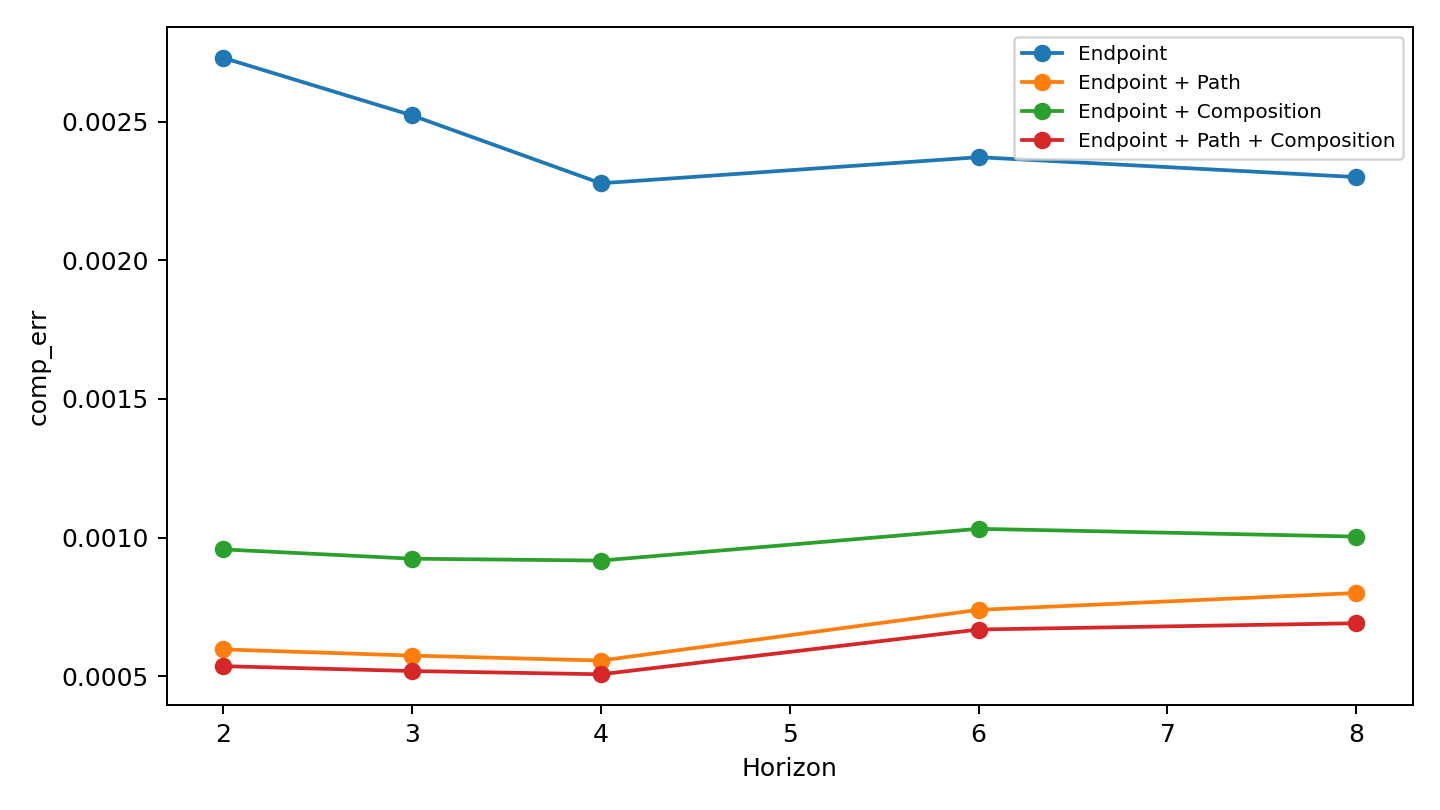}
\vspace{-1mm}

{\small (b) Direct--composed latent discrepancy.}
\end{minipage}
\caption{Prediction accuracy and temporal composition across horizons. Panel~(a) measures target-relative direct FDE: the learned models strongly outperform static persistence, while Endpoint + Composition becomes more accurate than causal constant-velocity extrapolation at the longest supervised horizon. Panel~(b) measures direct--composed latent discrepancy: endpoint supervision alone leaves substantial path dependence, whereas path supervision and explicit temporal composition reduce it, with the combined objective giving the lowest CompErr.}
\label{fig:exp1-horizon-results}
\end{figure}

\paragraph{Path supervision reduces recursive rollout drift.}

Endpoint + Path reduces CompErr from $0.00244$ to $0.00065$ even without directly minimizing the direct--composed discrepancy through $\mathcal L_{\mathrm{comp}}$. This is consistent with the two routes being independently anchored towards common future teacher states: $\mathcal L_{\mathrm{path}}$ in \cref{eq:path-loss} encourages path agreement indirectly through target consistency, whereas $\mathcal L_{\mathrm{comp}}$ enforces it explicitly.

Path supervision also reduces the rollout gap $\Delta_{\mathrm{roll}}$ in \cref{eq:rollout-gap} by about $31\%$ relative to Endpoint, although its absolute target-relative FDE is slightly higher under the present weighting. Combining path supervision and temporal composition yields the smallest rollout gap and CompErr, reducing them by about $39\%$ and $76\%$, respectively, relative to Endpoint. Under the present weights, Endpoint + Composition therefore provides the strongest accuracy--consistency trade-off, while the full objective produces the strongest direct--composed consistency.

\paragraph{Composition generalizes to unseen temporal decompositions.}

The reduction in path discrepancy persists for temporal decompositions not encountered during optimization. For held-out partitions constructed entirely from seen segment lengths, the full model reduces CompErr from $0.01134$ for Endpoint to $0.00260$, while GeoCompErr in \cref{eq:geometry-comp-metric} decreases from $0.01428$ to $0.01096$. For the harder decompositions containing unseen segment lengths, CompErr decreases from $0.00235$ to $0.00057$. These results support generalization of temporal composition beyond the particular partitions encountered during training.

Overall, Experiment~1 shows that temporal composition can reduce path dependence while preserving target-relative predictive accuracy, and that path supervision plays a complementary role by reducing recursive rollout drift. Additional per-seed results, partition-specific breakdowns, training curves, and implementation details are provided in \cref{app:exp1-details}.

\subsection{Experiment 2: does latent composition transfer to Gaussian geometry?}
\label{sec:experiment-geometry-composition}

Experiment~2 asks whether temporal consistency in latent space transfers to the explicit Gaussian motion produced by the selective decoder $Q_\psi^\G$. The question is not merely whether direct and composed latent predictions agree, but whether they induce compatible structured 3D motion after decoding. We therefore isolate the geometry-level composition objective $\mathcal L_{\mathrm{gcomp}}$ in \cref{eq:geometry-composition-loss}.

We use the same Endpoint + Composition latent objective and architecture as the common base configuration; each variant is trained independently from initialization. All variants use $\mathcal L_{\mathrm{end}}$, $\mathcal L_{\mathrm{comp}}$, $\mathcal L_{\mathrm{geo}}$, and $\mathcal L_{\mathrm{rep}}$. Both direct and composed decoded states receive identical target-anchored geometry supervision at their cumulative future endpoints. We compare:
\begin{enumerate}[leftmargin=*,itemsep=1pt]
    \item \textit{Latent composition only:} no geometry-level direct--composed regularization;
    \item \textit{+$\SE(3)$ composition:} additionally enforce the group-transformation component of $\mathcal L_{\mathrm{gcomp}}$;
    \item \textit{Full geometry composition:} additionally enforce both group-level transformation consistency and Gaussian-level $D_{\mathrm{match}}$ in \cref{eq:direct-composed-geometry-distance}.
\end{enumerate}
All other architecture, data, optimization, target supervision, horizon schedules, and temporal partitions are held fixed. The implemented objectives are given in \cref{app:exp2-details}.

Because the controlled generator provides the true motion-group transformations, we evaluate both \emph{structured-motion correctness} and \emph{path agreement}. For this controlled experiment, transformation errors use the rotation--translation surrogate $d_{\SE(3)}^{\mathrm{ctrl}}$ defined in \cref{app:exp2-details}, rather than the general Lie-group distance in \cref{eq:se3-distance}. The same controlled distance is used for training the group-level composition term and for the reported transformation metrics. Let $T_{m,t\rightarrow t+\Delta}^{+}$ denote the ground-truth relative transformation of motion group $m$. We report
\begin{equation}
E_T^{\mathrm{target}}
=
\frac{1}{|\Mset|}
\sum_{m\in\Mset}
d_{\SE(3)}^{\mathrm{ctrl}}\!\left(
\widehat T_{m,t\rightarrow t+\Delta}^{\mathrm{dir}},
T_{m,t\rightarrow t+\Delta}^{+}
\right),
\label{eq:transform-target-error}
\end{equation}
and
\begin{equation}
E_T^{\mathrm{dc}}
=
\frac{1}{|\Mset|}
\sum_{m\in\Mset}
d_{\SE(3)}^{\mathrm{ctrl}}\!\left(
\widehat T_{m,t\rightarrow t+\Delta}^{\mathrm{dir}},
\widehat T_{m,t\rightarrow t+\Delta}^{\pi}
\right).
\label{eq:transform-composition-error}
\end{equation}
At the Gaussian level, we additionally report direct target-relative center FDE from \cref{eq:experiment-fde}, direct covariance error using $\dspd$, and direct--composed geometry discrepancy $D_{\mathrm{match}}$ from \cref{eq:direct-composed-geometry-distance}. These metrics distinguish target accuracy from path consistency.

\begin{table}[t]
\caption{Geometry-level composition at the longest supervised horizon $\Delta=8$, averaged over 3 random seeds. Direct FDE, direct covariance error, and $E_T^{\mathrm{target}}$ are target-relative; $E_T^{\mathrm{dc}}$ and $D_{\mathrm{match}}$ measure direct--composed disagreement. Lower is better for all metrics.}
\label{tab:small-geometry-composition}
\centering
\tiny
\setlength{\tabcolsep}{3pt}
\begin{tabular}{lccccc}
\toprule
Method
& Direct FDE $\downarrow$
& Direct cov.\ err.\ $\downarrow$
& $E_T^{\mathrm{target}}$ $\downarrow$
& $E_T^{\mathrm{dc}}$ $\downarrow$
& $D_{\mathrm{match}}$ $\downarrow$ \\
\midrule
Latent composition only
& $\mathbf{0.06614{\pm}0.00271}$
& $0.19897{\pm}0.00387$
& $0.26124{\pm}0.01038$
& $0.00724{\pm}0.00057$
& $0.00192{\pm}0.00015$
\\
+$\SE(3)$ composition
& $0.06675{\pm}0.00408$
& $\mathbf{0.19758{\pm}0.00330}$
& $\mathbf{0.25554{\pm}0.00119}$
& $\mathbf{0.00202{\pm}0.00031}$
& $0.00138{\pm}0.00036$
\\
Full geometry composition
& $0.06728{\pm}0.00388$
& $0.19758{\pm}0.00290$
& $0.25582{\pm}0.00130$
& $0.00235{\pm}0.00062$
& $\mathbf{0.00130{\pm}0.00012}$
\\
\bottomrule
\end{tabular}
\end{table}

\paragraph{Geometry composition reduces residual path dependence.}

Latent temporal composition alone does not fully determine a path-consistent structured motion decomposition. At $\Delta=8$, the latent-composition baseline has $E_T^{\mathrm{dc}}=0.00724$. Adding explicit $\SE(3)$ composition reduces this to $0.00202$, a $72.2\%$ reduction, while the full geometry-composition objective obtains $0.00235$, a $67.6\%$ reduction. At the primitive level, $D_{\mathrm{match}}$ decreases from $0.00192$ to $0.00138$ under $\SE(3)$ composition and to $0.00130$ under the full objective, reductions of $28.4\%$ and $32.5\%$, respectively.

The two components of $\mathcal L_{\mathrm{gcomp}}$ therefore play distinguishable roles. The $\SE(3)$-only variant gives the lowest group-transform discrepancy $E_T^{\mathrm{dc}}$, whereas adding Gaussian-level matching further reduces $D_{\mathrm{match}}$ by $5.8\%$ relative to the $\SE(3)$-only variant. Group-level composition most strongly aligns the coarse structured motion, while Gaussian-level composition further aligns the primitive geometry induced by the two temporal routes.

\begin{figure}[t]
\centering
\begin{minipage}[t]{0.49\linewidth}
\centering
\includegraphics[width=\linewidth]{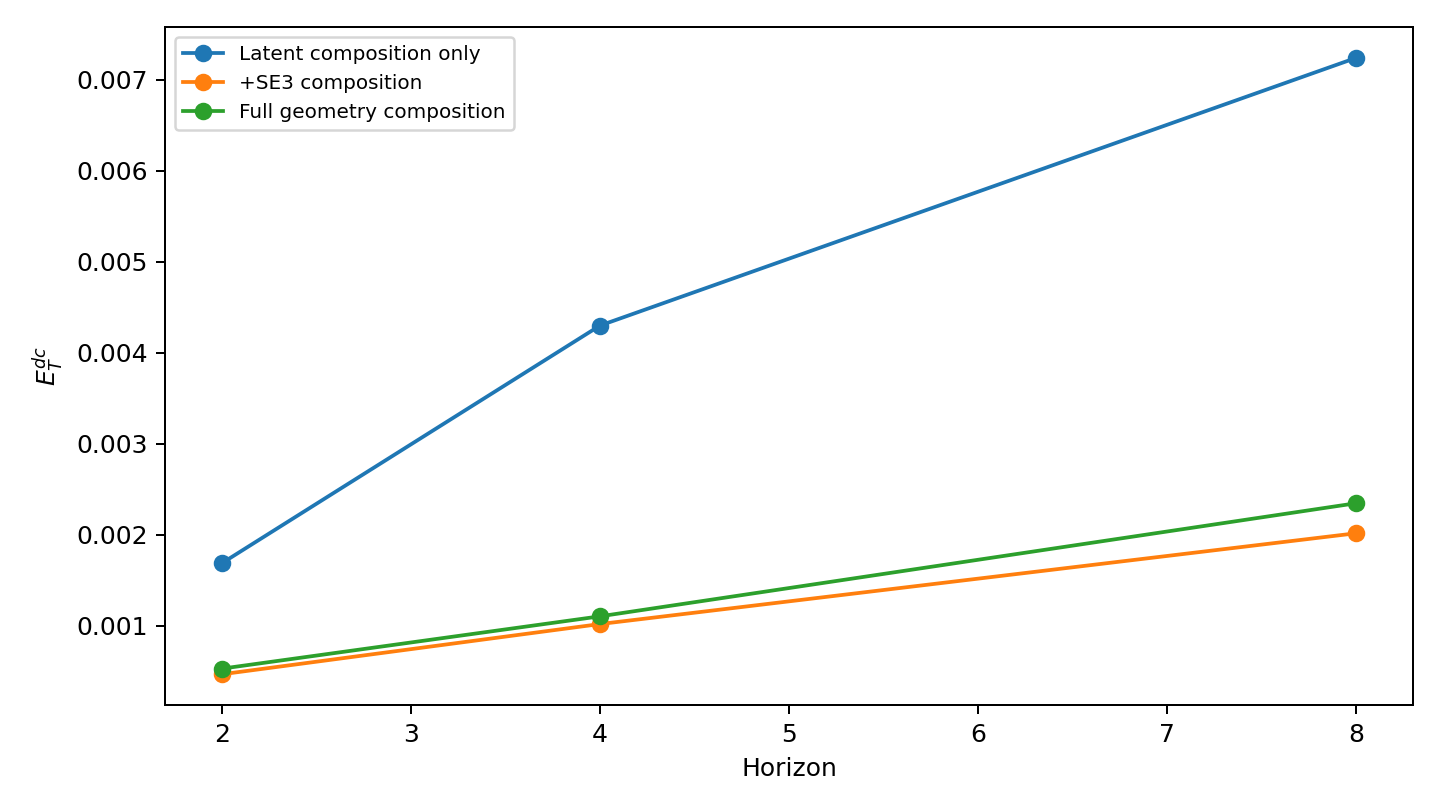}
\vspace{-1mm}

{\small (a) Group-transform path discrepancy.}
\end{minipage}
\hfill
\begin{minipage}[t]{0.49\linewidth}
\centering
\includegraphics[width=\linewidth]{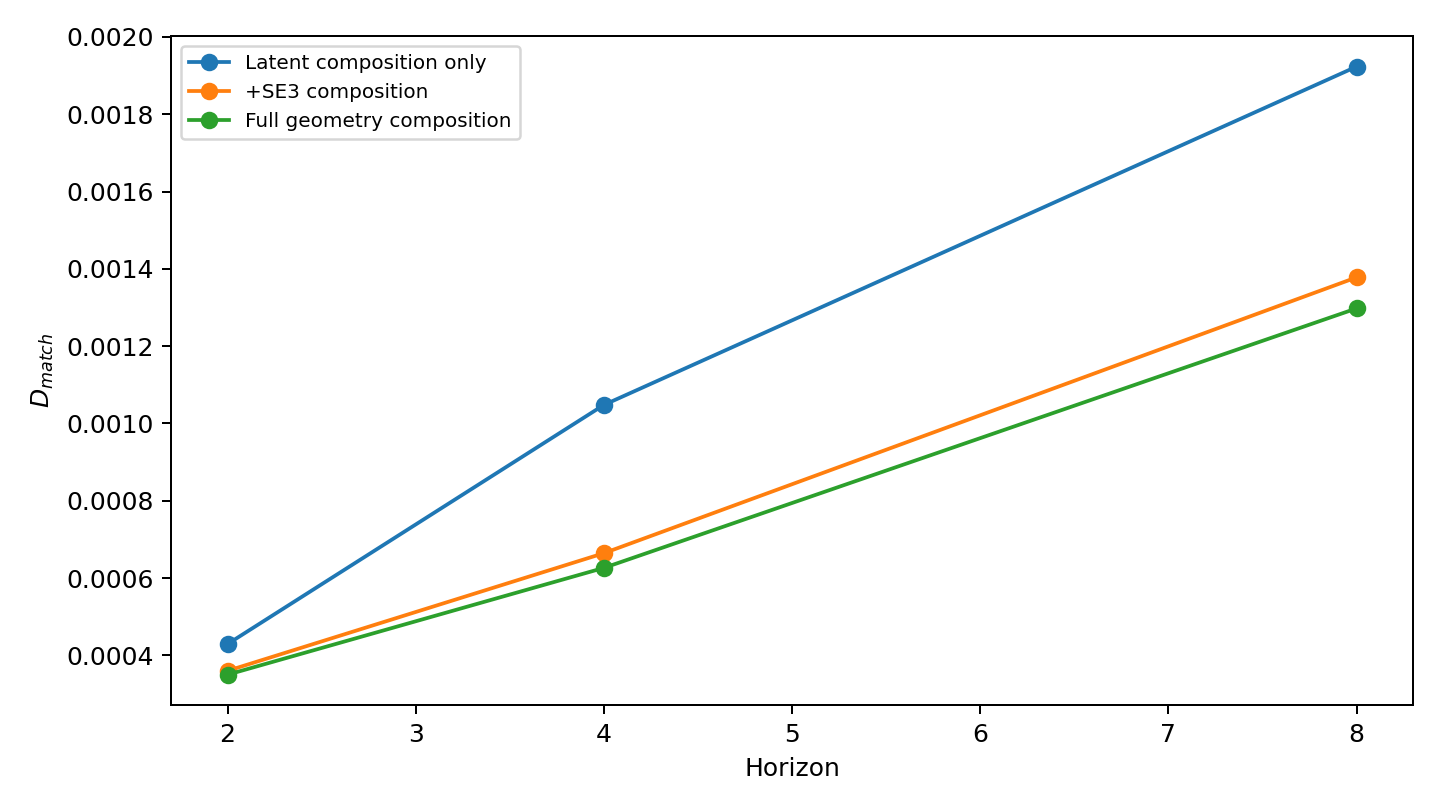}
\vspace{-1mm}

{\small (b) Gaussian-level path discrepancy.}
\end{minipage}
\caption{Explicit geometry consistency across non-trivial horizons. Panel~(a) reports the direct--composed group-transform error $E_T^{\mathrm{dc}}$ in \cref{eq:transform-composition-error}; explicit $\SE(3)$ composition substantially reduces the discrepancy at every evaluated horizon. Panel~(b) reports the Gaussian-level discrepancy $D_{\mathrm{match}}$ in \cref{eq:direct-composed-geometry-distance}; the full geometry-composition objective gives the strongest primitive-level consistency. Explicit geometry regularization therefore further reduces residual path dependence beyond latent composition and common target anchoring.}
\label{fig:exp2-geometry-composition}
\end{figure}

\paragraph{The improved consistency is not explained by trivial motion.}

Composition error must be interpreted jointly with target-relative accuracy. At $\Delta=8$, static persistence gives FDE $0.32158$ and constant-velocity extrapolation gives FDE $0.07120$, whereas all three learned variants obtain FDE between $0.06614$ and $0.06728$. Thus, every learned variant predicts substantial future motion and outperforms constant velocity at this longest horizon. Likewise, an identity group transformation gives $E_T^{\mathrm{target}}=0.34170$, compared with $0.26124$ for latent composition alone and approximately $0.256$ for the geometry-composition variants. The observed direct--composed agreement therefore cannot be explained by a decoder that predicts negligible or identity motion.

Geometry composition changes endpoint FDE only slightly. Relative to latent composition alone, the $\SE(3)$ and full variants increase direct FDE by approximately $0.9\%$ and $1.7\%$, respectively, while improving $E_T^{\mathrm{target}}$ by approximately $2.2\%$ and $2.1\%$. The gain in explicit path consistency is therefore obtained with only a small change in target-relative center accuracy and a slight improvement in structured-motion accuracy.

These results provide finite-error evidence for the role of $\mathcal L_{\mathrm{gcomp}}$ in \cref{eq:geometry-composition-loss}. If both prediction routes exactly matched the same future Gaussian state, Gaussian-level direct--composed consistency would follow automatically. In the learned finite-error regime, however, the routes can remain individually target-anchored while disagreeing with each other, and similar endpoint geometry can arise from different combinations of group-rigid motion and Gaussian-specific residual deformation. Explicit geometry composition reduces this residual ambiguity by encouraging a common structured motion explanation across temporal paths.

\subsection{Experiment 3: hybrid correspondence under identity perturbations}
\label{sec:experiment-identity}

Experiment~3 stress-tests the correspondence mechanism in \cref{sec:identity} independently of the learned transition $\Phi_\phi$, so matching quality is not confounded by forecasting error. We ask whether the Gaussian correspondence plan $\Pi^G$ in \cref{eq:hybrid-correspondence} remains reliable when raw primitive ordering or local topology changes.

\paragraph{Controlled identity stress test.}

Starting from Gaussian states with known canonical identities and lineage, we construct target sets containing random permutation, one-to-many splitting, many-to-one merging, primitive death, activation of reserved inactive slots, and mixtures of these events at increasing perturbation severity.

For example,
\begin{equation}
[A,B,C,D,E]
\quad\longrightarrow\quad
[C,A_1,A_2,E,F],
\label{eq:exp3-mixed-example}
\end{equation}
changes the target ordering, splits $A$ into descendants $A_1$ and $A_2$, removes some source geometry, and activates new geometry $F$. Because these perturbations are synthetic, ground-truth lineage such as $A\rightarrow\{A_1,A_2\}$ is retained for evaluation but hidden from the correspondence procedures.

We compare:
\begin{enumerate}[leftmargin=*,itemsep=1pt]
    \item \textit{Index matching}, which assumes unchanged raw array positions;
    \item \textit{OT only}, which ignores persistent canonical identity and solves a soft correspondence problem over the complete Gaussian sets;
    \item \textit{Hybrid}, which fixes valid persistent correspondences and applies residual OT only to unmatched or ambiguous mass according to \cref{eq:hybrid-correspondence,eq:residual-ot}.
\end{enumerate}
Index matching tests the failure mode in \cref{eq:identity-permutation-example}, OT tests identity-agnostic correspondence from geometry and auxiliary features, and Hybrid tests the principle in \cref{eq:hybrid-identity-principle}: preserve identity when reliable and fall back to soft matching when it is not.

\paragraph{Metrics.}

Because the true source--target lineage is known, let $Y_{ij}\in\{0,1\}$ indicate whether source primitive $i$ and target primitive $j$ share a valid ground-truth lineage relation. For a predicted correspondence plan $\Pi^G$, we report
\begin{equation}
\mathrm{LineageMass}
=
\frac{
\sum_{i,j}Y_{ij}\Pi^G_{ij}
}{
\sum_{i,j}\Pi^G_{ij}
+
\varepsilon_{\mathrm{match}}
},
\label{eq:lineage-mass}
\end{equation}
where the sums include only real, non-dustbin assignments. The numerator measures correspondence mass assigned to ground-truth-valid lineage edges and the denominator measures total real-to-real correspondence mass. Higher LineageMass therefore indicates greater concentration of predicted matching mass on valid ground-truth lineage relations.

We additionally report birth/death F1 for topology-changing events and the geometry discrepancy induced by the correspondence. The latter measures whether the transport plan pairs geometrically compatible primitives. The count-aware OT construction, event definitions, bootstrap confidence intervals, and additional severity and identity-availability tests are given in \cref{app:exp3-details}.

\begin{table}[t]
\caption{Controlled Gaussian-identity stress test, averaged over 3 independent repetitions. LineageMass and birth/death F1 are higher-is-better; geometry discrepancy is lower-is-better.}
\label{tab:small-identity-results}
\centering
\tiny
\setlength{\tabcolsep}{3pt}
\begin{tabular}{lcccc}
\toprule
Method
& Permutation lineage $\uparrow$
& Split/merge lineage $\uparrow$
& Birth/death F1 $\uparrow$
& Geometry discrepancy $\downarrow$ \\
\midrule
Index matching
& $0.04056{\pm}0.00507$
& $0.15563{\pm}0.00907$
& $0.61134{\pm}0.00381$
& $0.74311{\pm}0.00449$
\\
OT only
& $0.99789{\pm}0.00024$
& $0.98100{\pm}0.00099$
& $\mathbf{1.00000{\pm}0.00000}$
& $0.05010{\pm}0.00014$
\\
Hybrid
& $\mathbf{1.00000{\pm}0.00000}$
& $\mathbf{0.98715{\pm}0.00103}$
& $\mathbf{1.00000{\pm}0.00000}$
& $\mathbf{0.04906{\pm}0.00013}$
\\
\bottomrule
\end{tabular}
\end{table}

\paragraph{Raw index matching is fundamentally unsafe.}

The clearest failure in \cref{tab:small-identity-results} occurs under permutation. Although the represented scene is unchanged, Index matching obtains only $0.04056$ LineageMass, compared with $0.99789$ for OT and $1.00000$ for Hybrid. This directly supports the distinction in \cref{eq:index-vs-identity}: raw primitive array position cannot be treated as persistent canonical identity. The same failure affects geometry comparison, where Index matching has aggregate discrepancy $0.74311$, more than an order of magnitude larger than either correspondence-aware method.

\paragraph{Persistent identity remains useful beyond strong OT.}

OT alone is already a strong identity-agnostic baseline under the tested perturbations, obtaining $0.98100$ LineageMass under split/merge events. Hybrid increases this to $0.98715$ and reduces aggregate geometry discrepancy from $0.05010$ to $0.04906$. Thus, in this controlled benchmark, preserving valid canonical identity improves lineage and geometry correspondence even relative to strong residual OT. For isolated birth and death events, OT and Hybrid both attain perfect F1; the Hybrid advantage is instead concentrated in lineage preservation and the induced geometry correspondence.

\begin{figure}[t]
\centering
\begin{minipage}[t]{0.49\linewidth}
\centering
\includegraphics[width=\linewidth]{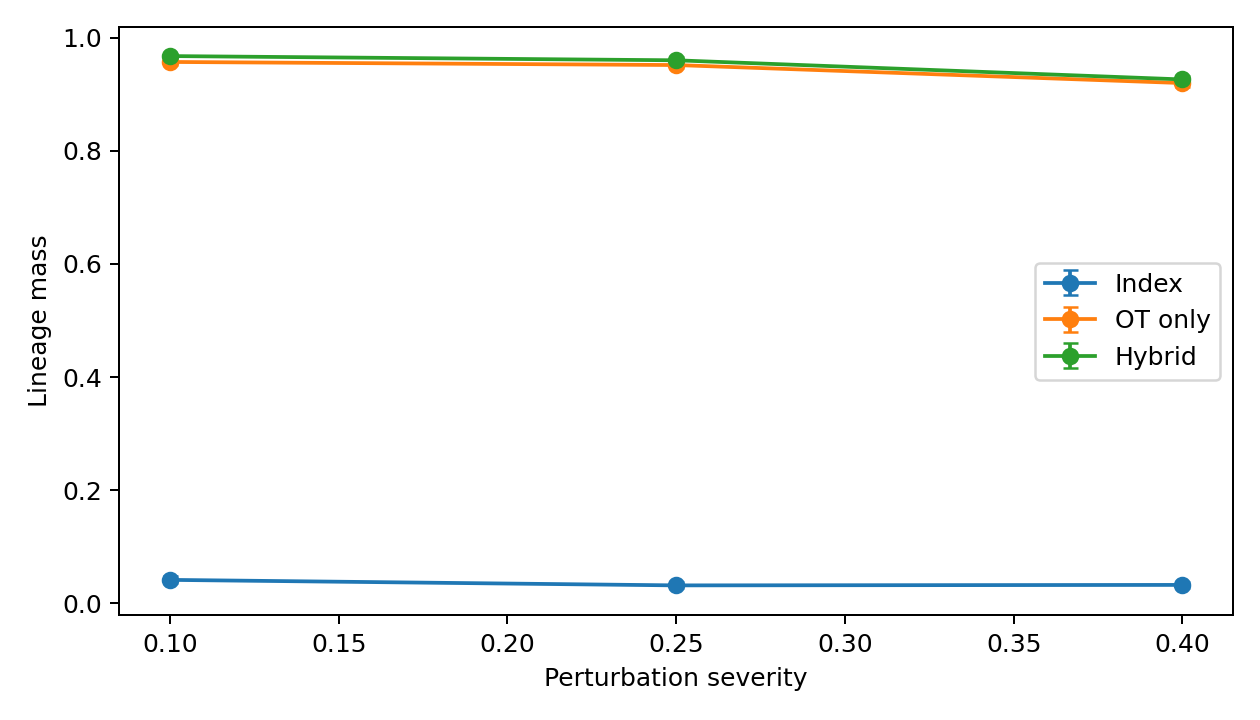}
\vspace{-1mm}

{\small (a) LineageMass under mixed perturbations.}
\end{minipage}
\hfill
\begin{minipage}[t]{0.49\linewidth}
\centering
\includegraphics[width=\linewidth]{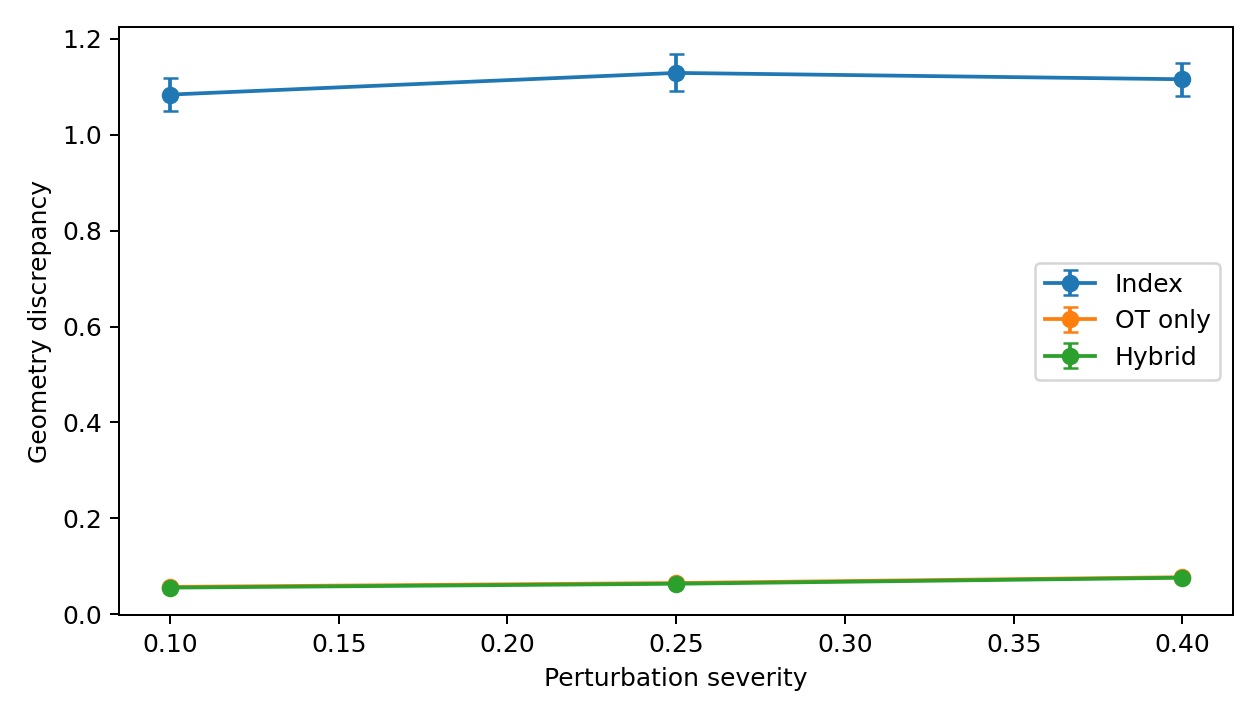}
\vspace{-1mm}

{\small (b) Induced geometry discrepancy.}
\end{minipage}
\caption{Mixed correspondence stress test as perturbation severity increases. Hybrid maintains higher LineageMass and lower induced geometry discrepancy than OT only at every evaluated severity. As more persistent canonical identities become invalid, the margin decreases, consistent with Hybrid increasingly relying on residual OT.}
\label{fig:exp3-mixed-severity}
\end{figure}

\paragraph{Hybrid degrades gracefully as identity becomes unavailable.}

\Cref{fig:exp3-mixed-severity} evaluates mixed permutation, split, merge, birth, and death perturbations at increasing severity. Hybrid obtains higher LineageMass and lower geometry discrepancy than OT at every evaluated severity. At severity $0.25$, LineageMass increases from $0.95148$ to $0.95988$, while geometry discrepancy decreases from $0.06502$ to $0.06349$. The margin narrows as perturbation severity increases, consistent with a larger fraction of persistent canonical identities becoming invalid and Hybrid relying increasingly on residual OT.

We additionally hide otherwise valid persistent canonical identities in a separate identity-availability test. As the hidden fraction increases from $0$ to $0.5$, Hybrid LineageMass decreases smoothly from $0.95988$ to $0.95649$ and geometry discrepancy increases from $0.06349$ to $0.06409$, approaching the OT-only reference of $0.95148$ and $0.06502$. This is the intended behavior of \cref{eq:hybrid-identity-principle}: Hybrid uses persistent identity when reliable but falls back towards residual OT as that information is removed. Complete identity-availability results are reported in \cref{app:exp3-id-dropout}.

Overall, Experiment~3 establishes that raw index alignment is not a valid correspondence assumption in this benchmark, that soft OT provides a strong identity-agnostic baseline under the tested reordering and topology perturbations, and that preserving reliable persistent canonical identity before matching only residual ambiguity can further improve lineage and geometry correspondence.

\subsection{Summary and scope}
\label{sec:empirical-summary}

The three experiments provide complementary mechanism-level evidence for 4DGS-JEPA. Experiment~1 shows that temporal composition reduces latent path dependence while maintaining target-relative predictive accuracy, with path supervision further reducing recursive rollout drift. Experiment~2 shows that latent consistency does not automatically imply equally consistent decoded motion, and that explicit group- and Gaussian-level composition substantially reduces this residual geometric path dependence. Experiment~3 shows that raw array alignment is unreliable under reordering and topology change, while hybrid correspondence preserves valid persistent canonical identity when available and falls back smoothly towards residual OT when it is not.

These results are controlled proof-of-concept evidence for the proposed mechanisms rather than evidence of large-scale superiority over reconstruction-based Gaussian world models or video-space predictive representations. The benchmark uses procedurally generated Gaussian dynamics so that future geometry, motion-group transformations, temporal decompositions, and lineage are known exactly. Evaluation on causally reconstructed real dynamic scenes, action-conditioned environments, larger Gaussian sets, appearance-sensitive prediction, representation probes, and downstream planning remains future work. Additional implementation details and experiment-specific results are provided in \cref{app:experiment-details}.

\section{Discussion}
\label{sec:discussion}

We introduced \method{}, a Gaussian-native joint-embedding predictive architecture for learning structured dynamics over dynamic Gaussian scene representations. The predictive state is organized hierarchically through scene-, motion-group-, and Gaussian-level tokens, and the horizon-conditioned transition $\Phi_\phi$ is closed in this predictive state space, enabling both direct prediction and recursive rollout. The central learning principle is \emph{temporal composition}: alternative chronological transition paths reaching the same future endpoint should produce compatible predictive states. Endpoint supervision $\mathcal L_{\mathrm{end}}$ anchors direct predictions to future target embeddings, multi-horizon path supervision $\mathcal L_{\mathrm{path}}$ anchors recurrent predictions along a rollout, and temporal composition $\mathcal L_{\mathrm{comp}}$ explicitly reduces residual dependence on the path used to reach the endpoint.

The theoretical and empirical results clarify the complementary roles of target anchoring and composition. Under the zero-loss and matched-support assumptions of \cref{prop:path-agreement}, path supervision together with temporal composition implies agreement between direct prediction, composed prediction, and the future target representation. The finite-error bound in \cref{eq:rollout-bound} further separates local transition error from its amplification under recursive rollout. Composition therefore constrains \emph{how consistently} a future state is reached, but does not by itself establish that the state is correct. Experiment~1 reflects this distinction: adding $\mathcal L_{\mathrm{comp}}$ reduces direct--composed latent discrepancy while preserving target-relative accuracy, whereas combining $\mathcal L_{\mathrm{path}}$ and $\mathcal L_{\mathrm{comp}}$ gives the smallest rollout gap and strongest path consistency. The reduction in path dependence also persists for held-out temporal decompositions, including decompositions containing unseen segment lengths.

Beyond latent prediction, the selective geometry decoder $Q_\psi^\G$ grounds the learned dynamics in explicit group-level $\SE(3)$ motion, Gaussian-specific residual displacement, and covariance deformation without making complete future appearance reconstruction the principal objective. The full formulation additionally permits existence dynamics for activation and disappearance of canonical slots, although these are not exercised in the present forecasting experiments. Because latent consistency does not automatically imply equally consistent decoded motion, $\mathcal L_{\mathrm{gcomp}}$ separately constrains direct and composed group transformations and Gaussian geometry. In Experiment~2, explicit $\SE(3)$ composition reduces group-transform path discrepancy by $72.2\%$ at $\Delta=8$, while the full geometry-composition objective reduces Gaussian-level direct--composed discrepancy by $32.5\%$, with only a small change in endpoint FDE. These results support a distinction between latent and geometric compositionality: latent consistency provides a useful foundation, while explicit geometric constraints further reduce finite-error path dependence after decoding.

The hybrid correspondence mechanism addresses a complementary property of Gaussian scene representations: raw array ordering does not define persistent canonical identity. \method{} therefore preserves reliable persistent canonical correspondences when available and applies residual optimal transport only to unmatched or ambiguous mass. Experiment~3 shows that raw index matching fails even under pure permutation, whereas OT provides a strong identity-agnostic baseline under the tested reordering and topology perturbations. Hybrid correspondence exactly recovers the permutation correspondence in this controlled setting and provides modest improvements over OT in lineage preservation and induced geometry correspondence under split, merge, and mixed perturbations. As persistent canonical identities are deliberately removed, Hybrid approaches OT-only behavior smoothly rather than failing abruptly, consistent with the design principle in \cref{eq:hybrid-identity-principle}.

The present evidence is intentionally mechanism-level rather than comprehensive. Experiments~1 and~2 operate directly on procedurally generated Gaussian states rather than Gaussian scenes reconstructed causally from real observations, while Experiment~3 evaluates correspondence independently of the learned transition $\Phi_\phi$. The current transition is deterministic, and the experiments do not independently establish the benefits of the full scene--motion-group--Gaussian hierarchy, structural dynamics priors, learned existence dynamics, appearance preservation, large-scale Gaussian sets, or action-conditioned planning. Reconstruction error, severe occlusion, large topology changes, finite reserved-slot capacity, and intrinsically multimodal futures also remain open challenges. More fundamentally, temporal and geometric composition constrain prediction consistency but do not establish physical correctness, causal mechanism identification, or calibrated uncertainty. Extending \method{} to causally reconstructed real scenes, probabilistic or multimodal future prediction, scalable Gaussian representations, richer interactions, and action-conditioned environments therefore provides a natural direction for future work.

\section{Conclusion}
\label{sec:conclusion}

4DGS-JEPA takes a predictive rather than purely reconstructive view of dynamic Gaussian worlds. Instead of asking only what a Gaussian scene looks like at a particular time, it learns a structured representation of how that scene evolves. The key idea is to make the learned dynamics \emph{temporally compositional}: reaching the same future through one long transition or a sequence of shorter transitions should lead to compatible predictive states. A hierarchical Gaussian-native representation captures dynamics at scene, motion-group, and primitive levels, while a \textit{selective} geometry decoder grounds the learned latent evolution in explicit 3D motion by predicting only the Gaussian variables needed for geometric dynamics, without requiring complete future appearance reconstruction.

The controlled experiments provide proof-of-concept evidence for this design. Temporal composition makes latent predictions substantially less dependent on the rollout path, explicit geometry composition transfers this consistency to group motion and Gaussian geometry, and hybrid correspondence preserves meaningful Gaussian relationships when primitive ordering or topology changes. Together, these results suggest a route from Gaussian Splatting as a representation of \emph{what a dynamic scene looked like} toward a predictive representation of \emph{how the Gaussian world evolves}. Scaling this principle to reconstructed real scenes, multimodal futures, richer physical interactions, and action-conditioned prediction is the potential direction for future work.

\bibliography{references}

@article{lecun2022path,
  author = {LeCun, Yann},
  title = {A Path Towards Autonomous Machine Intelligence Version 0.9.2, 2022-06-27},
  journal = {Open Review},
  volume = {62},
  number = {1},
  pages = {1--62},
  year = {2022}
}

@misc{assran2023ijepa,
      title={Self-Supervised Learning from Images with a Joint-Embedding Predictive Architecture}, 
      author={Mahmoud Assran and Quentin Duval and Ishan Misra and Piotr Bojanowski and Pascal Vincent and Michael Rabbat and Yann LeCun and Nicolas Ballas},
      year={2023},
      eprint={2301.08243},
      archivePrefix={arXiv},
      primaryClass={cs.CV},
      url={https://arxiv.org/abs/2301.08243}, 
}

@article{bardes2024vjepa,
title={Revisiting Feature Prediction for Learning Visual Representations from Video},
author={Adrien Bardes and Quentin Garrido and Jean Ponce and Xinlei Chen and Michael Rabbat and Yann LeCun and Mido Assran and Nicolas Ballas},
journal={Transactions on Machine Learning Research},
issn={2835-8856},
year={2024},
url={https://openreview.net/forum?id=QaCCuDfBk2},
note={Featured Certification}
}

@misc{assran2025vjepa2,
      title={V-JEPA 2: Self-Supervised Video Models Enable Understanding, Prediction and Planning}, 
      author={Mido Assran and Adrien Bardes and David Fan and Quentin Garrido and Russell Howes and Mojtaba and Komeili and Matthew Muckley and Ammar Rizvi and Claire Roberts and Koustuv Sinha and Artem Zholus and Sergio Arnaud and Abha Gejji and Ada Martin and Francois Robert Hogan and Daniel Dugas and Piotr Bojanowski and Vasil Khalidov and Patrick Labatut and Francisco Massa and Marc Szafraniec and Kapil Krishnakumar and Yong Li and Xiaodong Ma and Sarath Chandar and Franziska Meier and Yann LeCun and Michael Rabbat and Nicolas Ballas},
      year={2025},
      eprint={2506.09985},
      archivePrefix={arXiv},
      primaryClass={cs.AI},
      url={https://arxiv.org/abs/2506.09985}, 
}

@misc{murlabadia2026vjepa21,
      title={V-JEPA 2.1: Unlocking Dense Features in Video Self-Supervised Learning}, 
      author={Lorenzo Mur-Labadia and Matthew Muckley and Amir Bar and Mido Assran and Koustuv Sinha and Mike Rabbat and Yann LeCun and Nicolas Ballas and Adrien Bardes},
      year={2026},
      eprint={2603.14482},
      archivePrefix={arXiv},
      primaryClass={cs.CV},
      url={https://arxiv.org/abs/2603.14482}, 
}

@inproceedings{bardes2022vicreg,
title={{VICR}eg: Variance-Invariance-Covariance Regularization for Self-Supervised Learning},
author={Adrien Bardes and Jean Ponce and Yann LeCun},
booktitle={International Conference on Learning Representations},
year={2022},
url={https://openreview.net/forum?id=xm6YD62D1Ub}
}

@inproceedings{cuturi2013sinkhorn,
author = {Cuturi, Marco},
title = {Sinkhorn distances: lightspeed computation of optimal transport},
year = {2013},
publisher = {Curran Associates Inc.},
address = {Red Hook, NY, USA},
booktitle = {Proceedings of the 27th International Conference on Neural Information Processing Systems - Volume 2},
pages = {2292–2300},
numpages = {9},
location = {Lake Tahoe, Nevada},
series = {NIPS'13}
}

@article{kerbl2023gaussians,
author = {Kerbl, Bernhard and Kopanas, Georgios and Leimkuehler, Thomas and Drettakis, George},
title = {3D Gaussian Splatting for Real-Time Radiance Field Rendering},
year = {2023},
issue_date = {August 2023},
publisher = {Association for Computing Machinery},
address = {New York, NY, USA},
volume = {42},
number = {4},
issn = {0730-0301},
url = {https://doi.org/10.1145/3592433},
doi = {10.1145/3592433},
journal = {ACM Trans. Graph.},
month = jul,
articleno = {139},
numpages = {14}
}

@INPROCEEDINGS{wu2024fourDGS,
  author={Wu, Guanjun and Yi, Taoran and Fang, Jiemin and Xie, Lingxi and Zhang, Xiaopeng and Wei, Wei and Liu, Wenyu and Tian, Qi and Wang, Xinggang},
  booktitle={2024 IEEE/CVF Conference on Computer Vision and Pattern Recognition (CVPR)}, 
  title={4D Gaussian Splatting for Real-Time Dynamic Scene Rendering}, 
  year={2024},
  volume={},
  number={},
  pages={20310-20320},
  doi={10.1109/CVPR52733.2024.01920}}

@misc{lu2025gwm,
      title={GWM: Towards Scalable Gaussian World Models for Robotic Manipulation}, 
      author={Guanxing Lu and Baoxiong Jia and Puhao Li and Yixin Chen and Ziwei Wang and Yansong Tang and Siyuan Huang},
      year={2025},
      eprint={2508.17600},
      archivePrefix={arXiv},
      primaryClass={cs.RO},
      url={https://arxiv.org/abs/2508.17600}, 
}

@misc{chai2025gaf,
      title={GAF: Gaussian Action Field as a 4D Representation for Dynamic World Modeling in Robotic Manipulation}, 
      author={Ying Chai and Litao Deng and Ruizhi Shao and Jiajun Zhang and Kangchen Lv and Liangjun Xing and Xiang Li and Hongwen Zhang and Yebin Liu},
      year={2025},
      eprint={2506.14135},
      archivePrefix={arXiv},
      primaryClass={cs.RO},
      url={https://arxiv.org/abs/2506.14135}, 
}

@misc{lee2026mogaf,
      title={Space-Time Forecasting of Dynamic Scenes with Motion-aware Gaussian Grouping}, 
      author={Junmyeong Lee and Hoseung Choi and Minsu Cho},
      year={2026},
      eprint={2602.21668},
      archivePrefix={arXiv},
      primaryClass={cs.CV},
      url={https://arxiv.org/abs/2602.21668}, 
}

@inproceedings{li2026structured4d,
title={Structured 4D Latent Predictive Model for Robot Planning},
author={Zhiyi Li and Peilin Wu and Xiaoshen Han and Ruojin Cai and Yilun Du},
booktitle={Forty-third International Conference on Machine Learning},
year={2026},
url={https://openreview.net/forum?id=aXAgpGfHGc}
}

@misc{zhang2026gaussiandream,
      title={GaussianDream: A Feed-Forward 3D Gaussian World Model for Robotic Manipulation}, 
      author={Zijian Zhang and Yuqing Jiang and Qian Cheng and Xiaofan Li and Si Liu and Ding Zhao and Ping Luo and Weitao Zhou and Haibao Yu},
      year={2026},
      eprint={2605.20752},
      archivePrefix={arXiv},
      primaryClass={cs.RO},
      url={https://arxiv.org/abs/2605.20752}, 
}

@misc{pumarola2021dnerf,
      title={D-NeRF: Neural Radiance Fields for Dynamic Scenes}, 
      author={Albert Pumarola and Enric Corona and Gerard Pons-Moll and Francesc Moreno-Noguer},
      year={2020},
      eprint={2011.13961},
      archivePrefix={arXiv},
      primaryClass={cs.CV},
      url={https://arxiv.org/abs/2011.13961}, 
}

@article{park2021hypernerf,
author = {Park, Keunhong and Sinha, Utkarsh and Hedman, Peter and Barron, Jonathan T. and Bouaziz, Sofien and Goldman, Dan B and Martin-Brualla, Ricardo and Seitz, Steven M.},
title = {HyperNeRF: a higher-dimensional representation for topologically varying neural radiance fields},
year = {2021},
issue_date = {December 2021},
publisher = {Association for Computing Machinery},
address = {New York, NY, USA},
volume = {40},
number = {6},
issn = {0730-0301},
url = {https://doi.org/10.1145/3478513.3480487},
doi = {10.1145/3478513.3480487},
journal = {ACM Trans. Graph.},
month = dec,
articleno = {238},
numpages = {12}
}

@inproceedings{huang2026vjepa,
title={{VJEPA}: Variational Joint Embedding Predictive Architectures as Probabilistic World Models},
author={Yongchao Huang},
booktitle={Forty-third International Conference on Machine Learning},
year={2026},
url={https://openreview.net/forum?id=omJqMJt2fC}
}

@misc{huang2026bijepa,
      title={BiJEPA: Bi-directional Joint Embedding Predictive Architecture for Symmetric Representation Learning}, 
      author={Yongchao Huang},
      year={2026},
      eprint={2603.00049},
      archivePrefix={arXiv},
      primaryClass={cs.LG},
      url={https://arxiv.org/abs/2603.00049}, 
}

@inproceedings{colab2024,
author = {Edwards, Katlyn and Scalisi, Corrie and DeMars-Smith, Julianne and Lee, Key},
title = {Google Colab for Teaching CS and ML},
year = {2024},
isbn = {9798400704246},
publisher = {Association for Computing Machinery},
address = {New York, NY, USA},
url = {https://doi.org/10.1145/3626253.3635432},
doi = {10.1145/3626253.3635432},
booktitle = {Proceedings of the 55th ACM Technical Symposium on Computer Science Education V. 2},
pages = {1925},
numpages = {1},
location = {Portland, OR, USA},
series = {SIGCSE 2024}
}
\bibliographystyle{iclr2026_conference}

\appendix


\section{Temporal Composition: Analysis}
\label{app:temporal-composition}

\subsection{Intuition and operator interpretation}
\label{app:temporal-composition-intuition}

The temporal-composition objective $\mathcal L_{\mathrm{comp}}$ in \cref{eq:composition-loss} enforces a simple principle: if two admissible prediction routes describe the same scene evolution over the same elapsed time, they should produce compatible predictive states. For a cumulative horizon $s_j$, the model may reach time $t+s_j$ either directly through $\widehat Z_{t+s_j}^{\mathrm{dir}}$ in \cref{eq:direct-cumulative-prediction} or recursively through $\widetilde Z_{t+s_j}^{\pi}$ in \cref{eq:path-rollout}. Temporal composition penalizes disagreement between these routes.

For example, the same endpoint $t+6$ may be reached as
\begin{equation*}
t\longrightarrow t+6,
\qquad
t\longrightarrow t+2\longrightarrow t+6,
\qquad
t\longrightarrow t+4\longrightarrow t+6.
\end{equation*}
The intermediate states differ, but all routes describe the same total evolution. The objective therefore discourages the horizon-conditioned transition $\Phi_\phi$ in \cref{eq:transition} from learning mutually incompatible shortcuts for different temporal decompositions.

At the operator level, the controlled composition law is given in \cref{eq:controlled-composition}. In the action-free, time-homogeneous case it reduces to
\begin{equation*}
\Phi_\phi(Z_t,\delta_1+\delta_2)
\approx
\Phi_\phi\!\left(
\Phi_\phi(Z_t,\delta_1),
\delta_2
\right).
\end{equation*}
This is an approximate temporal-composition property rather than an assumption of exact semigroup structure.

The two stop-gradient directions in $\mathcal L_{\mathrm{comp}}$ implement symmetric two-way agreement: in one term the composed route serves as a fixed target while the direct route receives gradients, and in the other the roles are reversed. The stop-gradient operator therefore changes gradient flow during optimization but not the numerical discrepancy between the two predictions.

Importantly, composition consistency does not establish target correctness. Direct and composed predictions may agree while both are inaccurate. The endpoint loss $\mathcal L_{\mathrm{end}}$ in \cref{eq:endpoint-loss} and path loss $\mathcal L_{\mathrm{path}}$ in \cref{eq:path-loss} provide future-target anchoring, whereas $\mathcal L_{\mathrm{comp}}$ specifically reduces dependence on the temporal route used to reach the future state.

\subsection{Proof of Proposition~\ref{prop:path-agreement}}
\label{app:path-agreement-proof}

\begin{proof}
From $\ell_{\mathrm{path},j}^{\pi}=0$ and $\omega_j>0$, the local path-supervision contribution inherited from \cref{eq:path-loss} gives
\begin{equation*}
\dhier\!\left(
\widetilde Z_{t+s_j}^{\pi},
Z_{t+s_j}^{+}
\right)
=
0.
\end{equation*}
Hence
\begin{equation*}
\widetilde Z_{t+s_j}^{\pi}
=
Z_{t+s_j}^{+}
\end{equation*}
on the supervised matched support.

Next, $\ell_{\mathrm{comp},j}^{\pi}=0$. By its definition as the local contribution of \cref{eq:composition-loss}, both constituent terms are non-negative. Since $u_j+\nu_j>0$, at least one direct--composed distance has positive weight and must therefore vanish. Because stop-gradient changes only gradient flow and not numerical values,
\begin{equation*}
\dhier\!\left(
\widehat Z_{t+s_j}^{\mathrm{dir}},
\widetilde Z_{t+s_j}^{\pi}
\right)
=
0.
\end{equation*}
Therefore
\begin{equation*}
\widehat Z_{t+s_j}^{\mathrm{dir}}
=
\widetilde Z_{t+s_j}^{\pi}
\end{equation*}
on the supervised matched support.

Combining the two equalities proves \cref{eq:zero-loss-path-agreement}.
\end{proof}

\subsection{Finite-error rollout accumulation}
\label{app:theory}

\Cref{prop:path-agreement} describes the idealized zero-loss case. We now derive the finite-error rollout bound stated in \cref{eq:rollout-bound}.

For this analysis, identify the supervised matched components of the hierarchical predictive state with their concatenated finite-dimensional representation, and let $\|\cdot\|$ denote a fixed norm on this representation. For segment $j$ of a temporal partition $\pi=(\delta_1,\ldots,\delta_J)$, define
\begin{equation*}
\Phi_j(Z)
=
\Phi_\phi\!\left(
Z,
\delta_j,
A_{t+s_{j-1}:t+s_j-1}
\right),
\end{equation*}
with the action argument omitted in passive forecasting.

Let
\begin{equation*}
Z_0^\star,Z_1^\star,\ldots,Z_J^\star
\end{equation*}
denote a reference trajectory, with
\begin{equation*}
Z_j^\star
=
Z_{t+s_j}^{+},
\qquad
j=1,\ldots,J,
\end{equation*}
at the supervised future horizons. The initial reference state $Z_0^\star$ need not equal the realized predictive state $Z_t$.

Assume that, on the relevant state region, each segment transition $\Phi_j$ is locally $L_j$-Lipschitz,
\begin{equation*}
\|\Phi_j(Z)-\Phi_j(Z')\|
\leq
L_j\|Z-Z'\|,
\end{equation*}
and that its local approximation error when applied to the reference trajectory is bounded by
\begin{equation*}
\left\|
\Phi_j(Z_{j-1}^\star)
-
Z_j^\star
\right\|
\leq
\epsilon_j.
\end{equation*}

Define the realized recurrent prediction error after segment $j$ as
\begin{equation*}
e_j
=
\left\|
\widetilde Z_{t+s_j}^{\pi}
-
Z_j^\star
\right\|,
\end{equation*}
with initial mismatch
\begin{equation*}
e_0
=
\|Z_t-Z_0^\star\|.
\end{equation*}

Using the recursive rollout in \cref{eq:path-rollout},
\begin{align*}
e_j
&=
\left\|
\Phi_j\!\left(
\widetilde Z_{t+s_{j-1}}^{\pi}
\right)
-
Z_j^\star
\right\|
\\
&\leq
\left\|
\Phi_j\!\left(
\widetilde Z_{t+s_{j-1}}^{\pi}
\right)
-
\Phi_j(Z_{j-1}^\star)
\right\|
+
\left\|
\Phi_j(Z_{j-1}^\star)
-
Z_j^\star
\right\|
\\
&\leq
L_j e_{j-1}
+
\epsilon_j.
\end{align*}

Applying this recurrence repeatedly gives
\begin{align*}
e_J
&\leq
L_J e_{J-1}
+
\epsilon_J
\\
&\leq
L_JL_{J-1}e_{J-2}
+
L_J\epsilon_{J-1}
+
\epsilon_J
\\
&\leq
\cdots
\\
&\leq
\left(
\prod_{q=1}^{J}L_q
\right)e_0
+
\sum_{j=1}^{J}
\left(
\prod_{q=j+1}^{J}L_q
\right)
\epsilon_j.
\end{align*}
Since $Z_J^\star=Z_{t+s_J}^{+}$, this is exactly the bound in \cref{eq:rollout-bound}:
\begin{align*}
\left\|
\widetilde Z_{t+s_J}^{\pi}
-
Z_{t+s_J}^{+}
\right\|
&\leq
\left(
\prod_{q=1}^{J}L_q
\right)e_0
+
\sum_{j=1}^{J}
\left(
\prod_{q=j+1}^{J}L_q
\right)
\epsilon_j.
\end{align*}
Here, an empty product is defined as one. If the reference trajectory is initialized at the realized predictive state, $Z_0^\star=Z_t$, then $e_0=0$ and the initial-mismatch term vanishes.

The bound separates three effects. First, $e_0$ measures any mismatch between the rollout initialization and the chosen reference trajectory. Second, the local errors $\epsilon_j$ measure how accurately the learned transition propagates each reference state across one temporal segment. Third, the products of Lipschitz factors quantify how errors introduced earlier in the rollout can be amplified by subsequent applications of $\Phi_\phi$.

The path loss $\mathcal L_{\mathrm{path}}$ in \cref{eq:path-loss} does not directly minimize the local quantities $\epsilon_j$. Instead, it supervises the realized recurrent states encountered when $\Phi_\phi$ is applied to its own predictions and therefore primarily addresses recursive target accuracy. Temporal composition $\mathcal L_{\mathrm{comp}}$ in \cref{eq:composition-loss} constrains the discrepancy between those recurrent states and their corresponding direct predictions and therefore primarily addresses path dependence. Experiment~1 empirically separates these complementary effects.

\section{Training Procedure}
\label{app:training-algorithm}

Algorithm~\ref{alg:training} summarizes the generic \method{} training procedure. For clarity, one terminal horizon $\Delta$ and one temporal partition $\pi$ are sampled per minibatch; multiple samples simply give a Monte Carlo estimate of the corresponding expectations in the full objective.

\begin{algorithm}[t]
\caption{Temporally compositional \method{} training}
\label{alg:training}
\small
\begin{algorithmic}[1]

\Require Observations $X_{1:T}$; horizon set $\Kset$; causal builders $B_{\mathrm{on}},B_{\mathrm{on}}^{+}$
\Require Student encoder $E_\theta$; EMA target encoder $E_{\bar\theta}$; transition $\Phi_\phi$; decoder $Q_\psi^\G$

\For{each minibatch of context endpoints $t$}

    \State $\G_{1:t}\gets B_{\mathrm{on}}(X_{1:t})$
    \Comment{causal Gaussian prefix}

    \State $Z_t\gets E_\theta(\Hh_t)$

    \State Sample $\Delta\in\Kset$ and
    $\pi=(\delta_1,\ldots,\delta_J)$ with
    $\sum_{j=1}^{J}\delta_j=\Delta$

    \State $\G_{t+1:t+\Delta}^{+}
    \gets
    B_{\mathrm{on}}^{+}(\G_t,X_{t+1:t+\Delta})$
    \Comment{forward-only targets}

    \State $\widetilde Z_t^\pi\gets Z_t$;
    $\widetilde\G_t^\pi\gets\G_t$

    \For{$j=1,\ldots,J$}

        \State $s_j\gets\sum_{q=1}^{j}\delta_q$

        \State $Z_{t+s_j}^{+}
        \gets
        \stopgrad\!\left(E_{\bar\theta}(\Hh_{t+s_j}^{+})\right)$

        \State $\widehat Z_{t+s_j}^{\mathrm{dir}}
        \gets
        \Phi_\phi(Z_t,s_j,A_{t:t+s_j-1})$

        \State $\widetilde Z_{t+s_j}^{\pi}
        \gets
        \Phi_\phi(
        \widetilde Z_{t+s_{j-1}}^{\pi},
        \delta_j,
        A_{t+s_{j-1}:t+s_j-1})$

        \State $\widehat\G_{t+s_j}^{\mathrm{dir}}
        \gets
        Q_\psi^\G(
        Z_t,
        \widehat Z_{t+s_j}^{\mathrm{dir}},
        \G_t,
        s_j)$

        \State $\widetilde\G_{t+s_j}^{\pi}
        \gets
        Q_\psi^\G(
        \widetilde Z_{t+s_{j-1}}^\pi,
        \widetilde Z_{t+s_j}^\pi,
        \widetilde\G_{t+s_{j-1}}^\pi,
        \delta_j)$

        \State Compute direct--target, composed--target, and direct--composed correspondences
        \Comment{\cref{sec:identity}}

        \State Accumulate $\mathcal L_{\mathrm{path}}$ and $\mathcal L_{\mathrm{comp}}$
        \Comment{\cref{eq:path-loss,eq:composition-loss}}

        \State Accumulate target-anchored $\mathcal L_{\mathrm{geo}}$
        \Comment{\cref{eq:geometry-loss}}

        \State Accumulate applicable structural priors
        \Comment{\cref{eq:structural-objective}}

    \EndFor

    \State Compute terminal $\mathcal L_{\mathrm{end}}$
    \Comment{\cref{eq:endpoint-loss}}

    \State Compute terminal $\mathcal L_{\mathrm{gcomp}}$
    \Comment{\cref{eq:geometry-composition-loss}}

    \State Compute hierarchy-level $\mathcal L_{\mathrm{rep}}$
    \Comment{\cref{eq:representation-regularization}}

    \State Form $\mathcal L$ using \cref{eq:full-objective}

    \State Update $(\theta,\phi,\psi)$ by gradient descent on $\mathcal L$

    \State $\bar\theta\gets\tau\bar\theta+(1-\tau)\theta$
    \Comment{EMA target update}

\EndFor

\end{algorithmic}
\end{algorithm}

For passive forecasting, the action arguments in Algorithm~\ref{alg:training} are omitted. Correspondences used by the latent and geometric losses are constructed from persistent canonical identity and Gaussian-level matching rather than from the latent discrepancy itself. In the current controlled experiments, Experiments~1--2 use fixed persistent correspondence, while Experiment~3 evaluates residual OT independently and contains no learned transition; hence the reported experiments do not backpropagate through Sinkhorn matching.

Experiment~1 and Experiment~2 instantiate restricted versions of Algorithm~\ref{alg:training}. Experiment~1 applies target-anchored geometry only to the direct route, selectively enables $\mathcal L_{\mathrm{path}}$ and $\mathcal L_{\mathrm{comp}}$, and disables $\mathcal L_{\mathrm{gcomp}}$. Experiment~2 omits $\mathcal L_{\mathrm{path}}$, supervises both decoded routes at cumulative endpoints, and varies the group- and Gaussian-level components of $\mathcal L_{\mathrm{gcomp}}$. Experiment~3 contains no optimization and therefore does not instantiate Algorithm~\ref{alg:training}.


\section{Gaussian Splatting: Representation and Rendering}
\label{app:gaussian-splatting}

This section briefly reviews the Gaussian Splatting representation used throughout the paper and clarifies the distinction between Gaussian reconstruction, Gaussian rendering, and the predictive Gaussian dynamics studied by \method{}.

\subsection{Explicit Gaussian scene representation}

Three-dimensional Gaussian Splatting (3DGS) represents a scene as a collection of anisotropic Gaussian primitives \citep{kerbl2023gaussians},
\begin{equation*}
\G
=
\{g_i\}_{i=1}^{N},
\qquad
g_i
=
(\mu_i,\Sigma_i,\alpha_i,c_i),
\end{equation*}
where the center $\mu_i\in\R^3$ specifies position, the covariance $\Sigma_i\in\mathbb S_{++}^{3}$ specifies spatial extent and orientation, the opacity $\alpha_i$ controls visibility, and the appearance coefficients $c_i$ typically parameterize view-dependent color through spherical harmonics.

Ignoring normalization constants, the spatial contribution of primitive $i$ is
\begin{equation*}
G_i(x)
=
\exp\!\left[
-\frac{1}{2}
(x-\mu_i)^\top
\Sigma_i^{-1}
(x-\mu_i)
\right].
\end{equation*}
A common covariance parameterization is
\begin{equation*}
\Sigma_i
=
R_iS_iS_i^\top R_i^\top,
\qquad
R_i\in\SO(3),
\end{equation*}
where $S_i$ contains positive scale parameters. Each primitive therefore represents an oriented ellipsoidal density in 3D.

The Gaussian set $\G$ is the explicit \emph{scene representation}. Gaussian \emph{splatting} refers specifically to projecting these primitives into an image and compositing their image-space contributions.

\subsection{Projection and alpha compositing}

For a camera transformation $(R_{\mathrm c},t_{\mathrm c})$ and perspective projection $\pi_{\mathrm c}$, the center of Gaussian $i$ is transformed and projected as
\begin{equation*}
\mu_i^{\mathrm c}
=
R_{\mathrm c}\mu_i+t_{\mathrm c},
\qquad
\mu_i^{\mathrm{2D}}
=
\pi_{\mathrm c}(\mu_i^{\mathrm c}).
\end{equation*}
Under a local linearization of the perspective projection, its image-plane covariance is approximately
\begin{equation*}
\Sigma_i^{\mathrm{2D}}
=
J_iR_{\mathrm c}
\Sigma_i
R_{\mathrm c}^{\top}J_i^\top,
\end{equation*}
where $J_i$ is the projection Jacobian evaluated at $\mu_i^{\mathrm c}$.

The resulting elliptical image-space footprint contributes
\begin{equation*}
a_i(p)
=
\alpha_i
\exp\!\left[
-\frac{1}{2}
(p-\mu_i^{\mathrm{2D}})^\top
(\Sigma_i^{\mathrm{2D}})^{-1}
(p-\mu_i^{\mathrm{2D}})
\right]
\end{equation*}
at image coordinate $p$. After approximately depth-ordering the relevant Gaussians, front-to-back alpha compositing gives
\begin{equation*}
C(p)
=
\sum_i
T_i(p)\,
a_i(p)\,
c_i(v),
\qquad
T_i(p)
=
\prod_{j<i}
\left(
1-a_j(p)
\right),
\end{equation*}
where $c_i(v)$ is the view-dependent color and $T_i(p)$ is the accumulated transmittance before Gaussian $i$.

Because projection and compositing are differentiable, Gaussian parameters can be fitted from image observations through a rendering objective of the schematic form
\begin{equation*}
\min_{\G}
\sum_v
\ell_{\mathrm{render}}
\left(
\mathcal R(\G;\mathcal C_v),
I_v
\right),
\end{equation*}
where $\mathcal R$ denotes Gaussian rasterization and $\mathcal C_v$ specifies camera $v$.

Gaussian reconstruction and Gaussian rendering therefore correspond to two different mappings:
\begin{equation*}
\underbrace{
\{I_v,\mathcal C_v\}_{v}
\longrightarrow
\G
}_{\text{Gaussian reconstruction}},
\qquad
\underbrace{
(\G,\mathcal C)
\longrightarrow
I
}_{\text{Gaussian rendering}}.
\end{equation*}

\subsection{Dynamic Gaussian states and predictive dynamics}

The dynamic Gaussian state used by \method{} is defined in \cref{eq:gaussian-state}, where the time-indexed primitive $g_{i,t}$ contains center $\mu_{i,t}$, covariance $\Sigma_{i,t}$, opacity $\alpha_{i,t}$, appearance coefficients $c_{i,t}$, semantic features $f_{i,t}$, and existence $e_{i,t}$. Dynamic Gaussian reconstruction methods may represent such evolution directly or through time-conditioned deformation of a canonical Gaussian scene \citep{wu2024fourDGS}.

The problem studied by \method{} is different from reconstructing or querying such a representation at observed timestamps. The student--teacher predictive mapping is defined in \cref{eq:problem-overview}, the latent dynamics are given by the horizon-conditioned transition $\Phi_\phi$ in \cref{eq:transition}, and the selective decoder maps predicted latent dynamics back to explicit Gaussian geometry through the decoding chain in \cref{eq:decoder-chain}.

The primary learning target is therefore a future representation rather than a rendered future image. The selective geometry decoder $Q_\psi^\G$ provides an explicit interface to future Gaussian motion without requiring complete appearance reconstruction.

This distinction also explains the controlled protocol in Experiments~1 and~2. Those experiments begin directly from explicit Gaussian states and therefore bypass both Gaussian reconstruction and image rendering. No camera projection, rasterization, RGB reconstruction loss, PSNR, or perceptual image metric is involved. This isolates predictive-dynamics errors from errors introduced by an upstream reconstruction system or renderer.

In short, Gaussian reconstruction asks \emph{which Gaussian state explains the observations}, whereas \method{} asks \emph{how a causal Gaussian-world state evolves predictively}.


\section{Experimental Details and Additional Results}
\label{app:experiment-details}

This section provides implementation details, experiment-specific protocol definitions, and additional results for the three controlled experiments in \cref{sec:experiments}. Experiments~1 and~2 use the same Gaussian-world model architecture and procedural dynamics generator, with every experimental variant trained independently from initialization. Experiment~3 isolates the Gaussian correspondence mechanism and contains no learned model.

\paragraph{Correspondence instantiation across experiments.}

Experiments~1 and~2 isolate predictive dynamics and temporal composition from correspondence ambiguity. Persistent canonical identities remain valid, primitive correspondence is fixed, and no permutation, split, merge, birth, or death events are introduced. Gaussian-level correspondence therefore reduces to known persistent matching in these experiments.

Experiment~3 instead isolates the hybrid correspondence mechanism itself. It contains no learned transition $\Phi_\phi$ or model optimization and deliberately perturbs primitive ordering and topology through permutation, split, merge, birth, death, and mixed events. Ground-truth source--target lineage is retained only for evaluation and is hidden from the correspondence procedures.

The correspondence evaluation is invariant to arbitrary raw storage ordering while remaining sensitive to whether transport mass is assigned to valid canonical lineage. Experiment~3 operationalizes this distinction through LineageMass, birth/death detection, and correspondence-induced geometry discrepancy. It thereby tests whether the hybrid rule in \cref{eq:hybrid-identity-principle} preserves reliable persistent canonical identity when available and degrades gracefully toward residual OT as that information is removed.

\paragraph{Compute environment.}

All three experiments were executed in Google Colaboratory (Colab) \citep{colab2024} in runtimes provisioned with an NVIDIA A100 GPU. Experiments~1 and~2 use PyTorch with CUDA for model training and evaluation. Experiment~3 is implemented entirely with NumPy-based correspondence computations and does not require GPU acceleration, although it was executed in the same Colab environment. The associated runtime used Linux x86-64 with an Intel Xeon CPU at $2.20$\,GHz, exposing $6$ physical cores and $12$ logical CPUs and approximately $89.6$\,GB of system memory.

\subsection{Common controlled setup for Experiments 1 and 2}
\label{app:training-details}

Experiments~1 and~2 operate directly on procedurally generated anisotropic 3D Gaussian trajectories rather than fitting Gaussian scenes from images. Consequently, the causal builders $B_{\mathrm{on}}$ and $B_{\mathrm{on}}^{+}$ in \cref{eq:causal-context-builder,eq:causal-target-builder} reduce to causal access to the generated Gaussian states.

Each scene contains $M=3$ motion groups with $6$ Gaussians per group, giving $N=18$ active primitives. Their evolution follows the structured group-motion and local-deformation construction described in \cref{sec:controlled-gaussian-data}: group trajectories combine linear and accelerated translation and rotation with oscillatory components, while Gaussian-specific residual displacement and anisotropic scale deformation introduce smooth local non-rigidity. Half of the generated scenes contain one static motion group.

Training, validation, and test scenes are independently generated from disjoint scene seeds, so test trajectories are neither reused training scenes nor temporal continuations of them. For Experiments~1 and~2, motion-group and Gaussian identities remain fixed and all active primitives persist throughout each sequence. Consequently, the correspondence matrices in the general formulation reduce to identity alignment in these two controlled experiments. Permutation, split, merge, birth, and death are introduced only in Experiment~3 so that correspondence errors do not confound the temporal- and geometry-composition experiments.

\begin{table}[t]
\caption{Controlled Gaussian-world generator used in Experiments~1 and~2.}
\label{tab:app-generator-settings}
\centering
\small
\setlength{\tabcolsep}{4pt}
\begin{tabular}{lc}
\toprule
Quantity & Setting \\
\midrule
Motion groups & $3$ \\
Gaussians per group & $6$ \\
Sequence length & $24$ \\
Group-center range & $[-1,1]^3$ \\
Within-group center std. & $0.16$ \\
Initial Gaussian scale & $[0.04,0.10]$ \\
Angular velocity & $[-0.045,0.045]$ \\
Angular acceleration & $[-0.002,0.002]$ \\
Translation velocity & $[-0.035,0.035]^3$ \\
Translation acceleration & $[-0.0015,0.0015]^3$ \\
Oscillatory translation amplitude & $[0,0.08]^3$ \\
Oscillatory translation frequency & $[0.08,0.18]$ \\
Local residual amplitude & $[0,0.035]$ \\
Local residual frequency & $[0.10,0.35]$ \\
Log-scale deformation amplitude & $[0,0.06]^3$ \\
\bottomrule
\end{tabular}
\end{table}

\Cref{tab:app-generator-settings} reports the parameter ranges used by the generator. Spatial coordinates use a normalized coordinate system and time advances in unit discrete steps. Translation velocities and accelerations are therefore measured in normalized spatial units per step and per step squared, respectively; angular velocities and accelerations are in radians per step and radians per step squared; sinusoidal frequencies are in radians per step; and log-scale deformation amplitudes are dimensionless. The generated dynamics deliberately contain smooth nonlinear effects beyond static persistence or constant-velocity motion while remaining small enough for controlled analysis.

\subsubsection{Causal motion-aware encoder}

The context length is $L=4$. In addition to current Gaussian geometry, the encoder receives causal finite-difference motion information. We use
\begin{equation*}
v_{i,t}
=
\mu_{i,t}-\mu_{i,t-1},
\qquad
a_{i,t}
=
\mu_{i,t}-2\mu_{i,t-1}+\mu_{i,t-2}.
\end{equation*}
Let $\chi(\Sigma)\in\R^6$ denote the log-Cholesky coordinates of an SPD covariance, and let
\begin{equation*}
\bar\mu_{m,t}
=
\frac{1}{|\Iset_m|}
\sum_{i\in\Iset_m}
\mu_{i,t}
\end{equation*}
be the current center of motion group $m$. The primitive input features are
\begin{align*}
x_{i,t}
=
\Bigg[
&
\frac{\mu_{i,t}}{s_{\mathrm p}},
\frac{\mu_{i,t}-\bar\mu_{m(i),t}}{s_{\mathrm r}},
\operatorname{NormCov}(\Sigma_{i,t}),
1,
\operatorname{onehot}(m(i)),
\\
&
\left\{
\frac{\mu_{i,t-k}-\mu_{i,t}}{s_{\mathrm r}}
\right\}_{k=1}^{L-1},
\frac{v_{i,t}}{s_v},
\frac{a_{i,t}}{s_a},
\frac{
\chi(\Sigma_{i,t})-\chi(\Sigma_{i,t-1})
}{
s_\Sigma
}
\Bigg],
\end{align*}
with fixed scales
\begin{equation*}
s_{\mathrm p}=1.5,
\qquad
s_{\mathrm r}=0.20,
\qquad
s_v=0.05,
\qquad
s_a=0.01,
\qquad
s_\Sigma=0.05.
\end{equation*}
For the current covariance coordinate,
\begin{equation*}
\operatorname{NormCov}(\Sigma)
=
\left[
\frac{\chi_{1:3}(\Sigma)-\log(0.07)}{0.5},
\frac{\chi_{4:6}(\Sigma)}{0.05}
\right].
\end{equation*}
These constants are fixed from the known generator scale and are not estimated from validation or test data.

Current-geometry and motion features are processed by separate MLP branches and fused before hierarchical aggregation. Gaussian tokens are pooled within their known motion groups, and the resulting group tokens are pooled into a scene token. The encoder and horizon-conditioned transition $\Phi_\phi$ each use one Transformer block with four attention heads.

\paragraph{Controlled geometry decoder.}

Experiments~1 and~2 use a restricted instantiation of the selective decoder in \cref{sec:geometry-decoder}. For each motion group, the decoder predicts a $3$-D axis--angle rotation increment and a $3$-D translation. For each Gaussian, it predicts a $3$-D residual displacement together with three diagonal log-scale deformation parameters. The latter define
\begin{equation*}
U_i
=
\operatorname{diag}\!\left(\exp(\widehat s_i)\right),
\end{equation*}
which is used in the covariance update. Because every active primitive persists in Experiments~1 and~2, existence is fixed to one and no existence head or birth/death supervision is used. Appearance and semantic attributes are likewise absent from these two controlled forecasting experiments.

Because motion-group and Gaussian identities are fixed, the experimental hierarchical latent distance is the identity-aligned specialization of \cref{eq:hier-distance}: scene- and group-level discrepancies use cosine distance, while Gaussian-token discrepancies use a SmoothL1/Huber loss with equal hierarchy weights.

\begin{table}[t]
\caption{Common model and optimization settings for Experiments~1 and~2.}
\label{tab:app-common-training}
\centering
\small
\setlength{\tabcolsep}{4pt}
\begin{tabular}{lc}
\toprule
Setting & Value \\
\midrule
Context length $L$ & $4$ \\
Train / validation / test scenes & $512/32/48$ \\
Train / validation / test samples & $6656/416/624$ \\
Embedding dimension & $48$ \\
Decoder horizon embedding & $16$ \\
Transformer layers & $1$ \\
Attention heads & $4$ \\
Transformer dropout & $0$ \\
Batch / evaluation batch size & $48/64$ \\
Maximum epochs & $15$ \\
Optimizer & AdamW \\
Learning rate & $10^{-3}\rightarrow10^{-5}$, cosine \\
Weight decay & $10^{-5}$ \\
Gradient clipping & $1.0$ \\
EMA momentum $\tau$ & $0.99$ \\
Training horizons & $\{1,2,4,8\}$ \\
Independent runs & $3$ \\
\bottomrule
\end{tabular}
\end{table}

\Cref{tab:app-common-training} summarizes the architecture and optimization settings shared by the learned experiments.

\subsubsection{Representation regularization}
\label{app:regularization}

The controlled experiments instantiate the representation regularizer $\mathcal L_{\mathrm{rep}}$ introduced in \cref{eq:representation-regularization} separately at the scene-, motion-group-, and Gaussian-token levels.

Let $V^\ell\in\R^{B_\ell\times d}$ collect the flattened embeddings at hierarchy level $\ell\in\{S,O,G\}$. The variance component is
\begin{equation*}
\mathcal R_{\mathrm{var}}(V^\ell)
=
\frac{1}{d}
\sum_{q=1}^{d}
\left[
\max\!\left(
0,
0.5-
\sqrt{
\operatorname{Var}(V^\ell_{:,q})+10^{-4}
}
\right)
\right]^2,
\end{equation*}
and, writing $C^\ell$ for the empirical feature covariance,
\begin{equation*}
\mathcal R_{\mathrm{cov}}(V^\ell)
=
\frac{1}{d}
\sum_{p\neq q}
(C^\ell_{pq})^2.
\end{equation*}
The implementation uses population-form variance in $\mathcal R_{\mathrm{var}}$ and the usual sample-covariance normalization for $C^\ell$. The hierarchy-level regularizer is
\begin{equation*}
\mathcal R_{\mathrm{rep}}(Z)
=
\frac{1}{3}
\sum_{\ell\in\{S,O,G\}}
\left[
\mathcal R_{\mathrm{var}}(V^\ell)
+
0.05\,\mathcal R_{\mathrm{cov}}(V^\ell)
\right].
\end{equation*}
At each training step we apply it to both the current student representation and the direct predicted representation,
\begin{equation*}
\mathcal L_{\mathrm{rep}}
=
\frac{1}{2}
\left[
\mathcal R_{\mathrm{rep}}(Z_t)
+
\mathcal R_{\mathrm{rep}}(\widehat Z_{t+\Delta}^{\mathrm{dir}})
\right].
\end{equation*}
Its outer objective weight is $0.02$ in both Experiments~1 and~2.

\subsubsection{Checkpoint selection}

Every learned run is trained for at most $15$ epochs. Checkpoints are selected exclusively using target-relative validation accuracy averaged over the supervised horizons:
\begin{equation*}
S_{\mathrm{val}}
=
\frac{1}{2}
\left(
\mathrm{FDE}_{\mathrm{dir,val}}
+
\mathrm{FDE}_{\mathrm{roll,val}}
\right).
\end{equation*}
The direct and rollout terms are averaged over the validation samples and supervised horizons $\{1,2,4,8\}$. Validation rollout FDE uses the fixed primary routes
\[
(1),\qquad
(1,1),\qquad
(2,2),\qquad
(4,4)
\]
for horizons $\Delta=1,2,4,8$, respectively. No latent- or geometry-composition metric enters checkpoint selection. Models trained with composition losses therefore receive no model-selection advantage on the consistency quantities they explicitly optimize.


\subsection{Experiment 1: Temporal Composition}
\label{app:exp1-details}

Experiment~1 compares the four variants in \cref{tab:small-temporal-results}. All variants share the same data, architecture, EMA teacher, selective decoder, optimization budget, and representation regularization.

As described in the main text, geometry supervision in Experiment~1 is applied only to the direct decoded route. Let $\mathcal L_{\mathrm{geo}}^{\mathrm{dir}}$ denote this direct-route restriction of the target-anchored geometry loss in \cref{eq:geometry-loss}. The implemented mean-reduced objective is
\begin{equation*}
\mathcal L_{\mathrm{Exp1}}
=
\mathcal L_{\mathrm{end}}
+
5\mathcal L_{\mathrm{geo}}^{\mathrm{dir}}
+
\mathbf 1_{\mathrm{path}}
\mathcal L_{\mathrm{path}}
+
0.5\,
\mathbf 1_{\mathrm{comp}}
\mathcal L_{\mathrm{comp}}
+
0.02\mathcal L_{\mathrm{rep}},
\end{equation*}
where $\mathbf 1_{\mathrm{path}},\mathbf 1_{\mathrm{comp}}\in\{0,1\}$ specify the corresponding ablation. Thus, the active path-loss coefficient is $1$, while the active temporal-composition coefficient is $0.5$. Geometry-level composition $\mathcal L_{\mathrm{gcomp}}$ in \cref{eq:geometry-composition-loss} is disabled throughout Experiment~1.

The center component of $\mathcal L_{\mathrm{geo}}^{\mathrm{dir}}$ applies the Huber penalty to the Euclidean $3$-D center displacement with threshold $0.1$. Covariance supervision applies SmoothL1 loss in log-Cholesky coordinates with relative weight $0.25$.

At each optimizer step, one terminal horizon is used, cycling uniformly through $\{1,2,4,8\}$ so that all variants receive the same horizon schedule. When multiple training partitions are available for that horizon, one is sampled from the corresponding partition set. The terminal horizons $\{3,6\}$ are never used for training and are evaluated only as unseen-horizon interpolation tests.

\paragraph{Temporal partitions.}

The training decompositions are
\begin{equation*}
\begin{aligned}
\mathcal D_{\mathrm{train}}^{(1)}
&=
\{(1)\},
\\
\mathcal D_{\mathrm{train}}^{(2)}
&=
\{(1,1)\},
\\
\mathcal D_{\mathrm{train}}^{(4)}
&=
\{(2,2),(1,1,2)\},
\\
\mathcal D_{\mathrm{train}}^{(8)}
&=
\{(4,4),(2,2,4)\}.
\end{aligned}
\end{equation*}

The primary evaluation route for each terminal horizon is
\begin{equation*}
\begin{aligned}
\pi_{\mathrm{pri}}^{(1)}&=(1),
&
\pi_{\mathrm{pri}}^{(2)}&=(1,1),
&
\pi_{\mathrm{pri}}^{(3)}&=(1,2),
\\
\pi_{\mathrm{pri}}^{(4)}&=(2,2),
&
\pi_{\mathrm{pri}}^{(6)}&=(2,4),
&
\pi_{\mathrm{pri}}^{(8)}&=(4,4).
\end{aligned}
\end{equation*}
Thus, terminal horizons $3$ and $6$ are not supervised during training, although the individual segment lengths used in their primary rollouts are observed during training.

To test generalization to unseen decompositions while keeping all segment lengths familiar, we additionally use
\begin{equation*}
\begin{aligned}
\mathcal D_{\mathrm{held}}^{(4)}
&=
\{(2,1,1),(1,2,1)\},
\\
\mathcal D_{\mathrm{held}}^{(8)}
&=
\{(4,2,2),(2,4,2),(2,2,2,2)\}.
\end{aligned}
\end{equation*}
A harder evaluation introduces segment lengths $3$ and $5$, neither of which is used as a transition horizon during training:
\begin{equation*}
\begin{aligned}
\mathcal D_{\mathrm{unseen}}^{(4)}
&=
\{(1,3),(3,1)\},
\\
\mathcal D_{\mathrm{unseen}}^{(8)}
&=
\{(3,5),(5,3)\}.
\end{aligned}
\end{equation*}

\begin{table}[t]
\caption{Generalization to temporal decompositions not used during Experiment~1 training, averaged over 3 random seeds. Lower is better.}
\label{tab:app-exp1-partition-results}
\centering
\tiny
\setlength{\tabcolsep}{3pt}
\begin{tabular}{lcccc}
\toprule
Method
& \multicolumn{2}{c}{Held-out, seen segments}
& \multicolumn{2}{c}{Unseen segments} \\
\cmidrule(lr){2-3}
\cmidrule(lr){4-5}
& CompErr $\downarrow$
& GeoCompErr $\downarrow$
& CompErr $\downarrow$
& GeoCompErr $\downarrow$ \\
\midrule
Endpoint
& $0.01134{\pm}0.00244$
& $0.01428{\pm}0.00160$
& $0.00235{\pm}0.00047$
& $0.00901{\pm}0.00120$
\\
Endpoint + Path
& $0.00286{\pm}0.00013$
& $0.01105{\pm}0.00146$
& $0.00064{\pm}0.00006$
& $\mathbf{0.00778{\pm}0.00119}$
\\
Endpoint + Composition
& $0.00454{\pm}0.00049$
& $0.01204{\pm}0.00125$
& $0.00096{\pm}0.00011$
& $0.00803{\pm}0.00110$
\\
Endpoint + Path + Composition
& $\mathbf{0.00260{\pm}0.00009}$
& $\mathbf{0.01096{\pm}0.00157}$
& $\mathbf{0.00057{\pm}0.00002}$
& $0.00782{\pm}0.00141$
\\
\bottomrule
\end{tabular}
\end{table}

As shown in \cref{tab:app-exp1-partition-results}, the full path-and-composition model gives the lowest latent path discrepancy in both held-out regimes. CompErr decreases from $0.01134$ to $0.00260$ for new decompositions made entirely from familiar segment lengths and from $0.00235$ to $0.00057$ when unseen segment lengths are introduced. This complements the target-relative and composition results reported in \cref{sec:experiment-temporal-composition}.


\subsection{Experiment 2: Geometry-Level Composition}
\label{app:exp2-details}

Experiment~2 uses the same procedural generator and model architecture as Experiment~1, but each variant is trained independently from initialization. All three variants use the same Endpoint + Composition latent objective, data, optimizer, horizon schedule, and temporal-partition schedule. Path supervision $\mathcal L_{\mathrm{path}}$ is omitted so that the additional effect of geometry-level composition can be isolated.

For each sampled temporal partition, both the direct and composed decoded states are independently supervised against the same target geometry at every cumulative endpoint. In the implementation, the two route losses are averaged at each endpoint and then averaged over the path. Denoting this mean-reduced controlled version of \cref{eq:geometry-loss} by $\mathcal L_{\mathrm{geo}}^{\mathrm{ctrl}}$, the common objective is
\begin{equation*}
\mathcal L_{\mathrm{base}}^{\mathrm{Exp2}}
=
\mathcal L_{\mathrm{end}}
+
0.5\mathcal L_{\mathrm{comp}}
+
5\mathcal L_{\mathrm{geo}}^{\mathrm{ctrl}}
+
0.02\mathcal L_{\mathrm{rep}}.
\end{equation*}

Let $\mathcal L_{\mathrm{gcomp}}^{\SE(3)}$ and $\mathcal L_{\mathrm{gcomp}}^{\mathrm{match}}$ denote the controlled motion-group and Gaussian-level direct--composed terms. The three variants use
\begin{align*}
\mathcal L_{\mathrm{latent}}
&=
\mathcal L_{\mathrm{base}}^{\mathrm{Exp2}},
\\
\mathcal L_{\SE(3)}
&=
\mathcal L_{\mathrm{base}}^{\mathrm{Exp2}}
+
0.4\mathcal L_{\mathrm{gcomp}}^{\SE(3)},
\\
\mathcal L_{\mathrm{full}}
&=
\mathcal L_{\mathrm{base}}^{\mathrm{Exp2}}
+
0.4\mathcal L_{\mathrm{gcomp}}^{\SE(3)}
+
0.5\mathcal L_{\mathrm{gcomp}}^{\mathrm{match}}.
\end{align*}
The geometry-composition terms are evaluated only at the terminal endpoint of the sampled partition.

The general Method defines a Lie-group transformation discrepancy in \cref{eq:se3-distance}. Experiment~2 instead uses the controlled rotation--translation surrogate
\begin{equation*}
d_{\SE(3)}^{\mathrm{ctrl}}(T_1,T_2)
=
\sqrt{
\theta(R_1^\top R_2)^2
+
\|d_2-d_1\|_2^2
+
10^{-12}
},
\qquad
T_k=(R_k,d_k),
\end{equation*}
where $\theta(R)$ is the geodesic rotation angle on $\SO(3)$. The implemented group-level composition term is the mean of $d_{\SE(3)}^{\mathrm{ctrl}}$ over motion groups and samples. The same distance is used for the reported target-relative and direct--composed transformation errors $E_T^{\mathrm{target}}$ and $E_T^{\mathrm{dc}}$ in \cref{eq:transform-target-error,eq:transform-composition-error}.

For Gaussian-level geometry supervision and composition during training, center error uses the Euclidean-displacement Huber penalty with threshold $0.1$, while covariance discrepancy uses SmoothL1 distance in log-Cholesky coordinates with relative weight $0.25$. Evaluation instead uses the affine-invariant SPD covariance distance. In particular, the direct target-relative covariance error reported in \cref{tab:small-geometry-composition} is
\begin{equation*}
E_{\Sigma}^{\mathrm{target}}
=
\frac{1}{|\Iset|}
\sum_{i\in\Iset}
\dspd\!\left(
\widehat\Sigma_{i}^{\mathrm{dir}},
\Sigma_i^{+}
\right).
\end{equation*}

The reported direct--composed $D_{\mathrm{match}}$ uses the same Huber center penalty together with the affine-invariant SPD covariance distance, again with covariance weight $0.25$. Because persistent Gaussian identities are fixed, these comparisons are slot-aligned and require no OT. Existence is also fixed to one, so the existence component of the general $D_{\mathrm{match}}$ in \cref{eq:direct-composed-geometry-distance} is identically zero in this experiment.

For numerical stability during recursive Experiment~2 rollouts, decoded covariance matrices are explicitly symmetrized and regularized by adding $10^{-6}I$ before reuse as the source covariance for the next segment.

Because the synthetic generator exposes the absolute motion-group poses at each time, the ground-truth relative transformation can be constructed exactly. This permits both target-relative structured-motion evaluation through $E_T^{\mathrm{target}}$ and direct--composed evaluation through $E_T^{\mathrm{dc}}$.

The non-trivial direct--composed geometry evaluation uses horizons $\Delta\in\{2,4,8\}$ with primary composed paths $(1,1)$, $(2,2)$, and $(4,4)$, respectively. The one-segment horizon $\Delta=1$ is excluded from direct--composed metrics because its direct and composed computations coincide.

\begin{table}[t]
\caption{Horizon-wise direct--composed geometry consistency in Experiment~2, averaged over 3 random seeds. Lower is better.}
\label{tab:app-exp2-horizon-consistency}
\centering
\tiny
\setlength{\tabcolsep}{2.5pt}
\begin{tabular}{lcccccc}
\toprule
& \multicolumn{3}{c}{$E_T^{\mathrm{dc}}$ $\downarrow$}
& \multicolumn{3}{c}{$D_{\mathrm{match}}$ $\downarrow$} \\
\cmidrule(lr){2-4}
\cmidrule(lr){5-7}
$\Delta$
& Latent
& +$\SE(3)$
& Full
& Latent
& +$\SE(3)$
& Full \\
\midrule
$2$
& $0.001685{\pm}0.000121$
& $\mathbf{0.000466{\pm}0.000089}$
& $0.000528{\pm}0.000163$
& $0.000429{\pm}0.000056$
& $0.000360{\pm}0.000051$
& $\mathbf{0.000350{\pm}0.000056}$
\\
$4$
& $0.004297{\pm}0.000723$
& $\mathbf{0.001018{\pm}0.000259}$
& $0.001102{\pm}0.000355$
& $0.001048{\pm}0.000167$
& $0.000664{\pm}0.000153$
& $\mathbf{0.000626{\pm}0.000082}$
\\
$8$
& $0.007238{\pm}0.000573$
& $\mathbf{0.002015{\pm}0.000311}$
& $0.002345{\pm}0.000619$
& $0.001923{\pm}0.000152$
& $0.001378{\pm}0.000362$
& $\mathbf{0.001298{\pm}0.000124}$
\\
\bottomrule
\end{tabular}
\end{table}

The ordering in \cref{tab:app-exp2-horizon-consistency} is stable across $\Delta\in\{2,4,8\}$. Relative to latent composition alone, explicit $\SE(3)$ composition reduces $E_T^{\mathrm{dc}}$ by approximately $72.4\%$, $76.3\%$, and $72.2\%$, respectively. The full objective reduces $D_{\mathrm{match}}$ by approximately $18.5\%$, $40.2\%$, and $32.5\%$ at the same horizons. Thus, the motion-group term gives the strongest transform-level consistency, whereas the full objective gives the strongest primitive-level Gaussian consistency.

As a non-degeneracy check at $\Delta=8$, static persistence has FDE $0.32158$, causal constant velocity has FDE $0.07120$, and the identity group transformation has $E_T^{\mathrm{target}}=0.34170$. The three learned variants obtain FDE between $0.06614$ and $0.06728$ and target-relative structured-transform error between $0.25554$ and $0.26124$. Their low direct--composed discrepancies therefore cannot be explained by predicting negligible motion.


\subsection{Experiment 3: Correspondence Stress Test}
\label{app:exp3-details}

Experiment~3 evaluates correspondence independently of the learned transition $\Phi_\phi$ and contains no learned parameters or optimization. Each source scene contains $24$ active Gaussian primitives and $6$ reserved inactive canonical slots with four motion-group labels. Each primitive additionally carries a $6$-dimensional normalized semantic descriptor. We use three independent Monte Carlo repetitions with $50$ trials per perturbation type and severity.

For each trial, small target-side center and log-scale perturbations are first applied to the active source primitives. A topology operation then applies permutation, split, merge, death, reserved-slot activation, or a mixture of these events. Split descendants retain the ground-truth lineage of their parent but lose a valid one-to-one persistent ID; merged primitives retain the union of their parent lineages and likewise receive no one-to-one persistent ID. Death removes primitives and birth activates reserved inactive canonical slots. After the topology operation, independent semantic noise is added and each semantic descriptor is renormalized. Permutation and mixed conditions finally randomize the target array ordering.

At perturbation severity $r$, the nominal number of affected active primitives is
\begin{equation*}
n_r
=
\max\!\left(
1,
\operatorname{round}(rN_{\mathrm{active}})
\right).
\end{equation*}

\begin{table}[t]
\caption{Experiment~3 configuration.}
\label{tab:app-exp3-settings}
\centering
\small
\setlength{\tabcolsep}{4pt}
\begin{tabular}{lc}
\toprule
Setting & Value \\
\midrule
Active / reserved primitives & $24/6$ \\
Motion-group labels & $4$ \\
Semantic dimension & $6$ \\
Perturbation severities & $\{0.10,0.25,0.40\}$ \\
Position-noise std. & $0.025$ \\
Log-scale-noise std. & $0.020$ \\
Semantic-noise std. & $0.050$ \\
Sinkhorn entropy $\varepsilon_{\mathrm{OT}}$ & $0.08$ \\
Maximum Sinkhorn iterations & $200$ \\
Sinkhorn tolerance & $10^{-7}$ \\
Real--dustbin cost & $1.2$ \\
Dustbin--dustbin cost & $0$ \\
Independent Monte Carlo repetitions & $3$ \\
Trials per condition and repetition & $50$ \\
ID-dropout rates & $\{0,0.10,0.25,0.50\}$ \\
Bootstrap samples & $2000$ \\
\bottomrule
\end{tabular}
\end{table}

\subsubsection{Count-aware residual matching}
\label{app:matching}

The Method defines the generic residual OT cost $C_{ij}$ in \cref{eq:ot-cost}. In the controlled correspondence stress test, the covariance component is instantiated through log-scale coordinates, giving
\begin{equation*}
C_{ij}
=
2\|\mu_i-\mu_j^{+}\|_2^2
+
0.25\|s_i-s_j^{+}\|_2^2
+
0.8\cosdist(f_i,f_j^{+})
+
0.8\,\mathbf 1[m(i)\neq m(j)],
\end{equation*}
where $s_i$ and $s_j^{+}$ are Gaussian log-scale vectors. Because motion-group labels are known exactly in this controlled stress test, the final term is the hard-label specialization of the generic group-compatibility term in \cref{eq:ot-cost}, with
\[
\kappa^O_{mn}
=
\mathbf 1[m=n].
\]
The semantic descriptors are independently perturbed and renormalized, so semantic features cannot serve as perfectly preserved identity labels.

OT only and Hybrid use the same unit-mass convention. After any hard persistent matches have been removed, let $n_s$ and $n_t$ denote the numbers of residual source and target primitives. The dustbin-augmented residual marginals are
\begin{equation*}
\mathbf p_{\mathrm{res}}^{\mathrm{aug}}
=
\left(
\underbrace{1,\ldots,1}_{n_s},
n_t
\right),
\qquad
\mathbf q_{\mathrm{res}}^{\mathrm{aug}}
=
\left(
\underbrace{1,\ldots,1}_{n_t},
n_s
\right),
\end{equation*}
so that
\begin{equation*}
\left\|
\mathbf p_{\mathrm{res}}^{\mathrm{aug}}
\right\|_1
=
\left\|
\mathbf q_{\mathrm{res}}^{\mathrm{aug}}
\right\|_1
=
n_s+n_t.
\end{equation*}
The final source row and target column are dustbins. Their capacities equal the cardinality of the opposite residual set, allowing arbitrary source--target cardinality mismatch without imposing a fixed unmatched-mass fraction.

For OT only, no persistent correspondence is removed first, so the complete source and target sets enter this augmented OT problem. For Hybrid, every valid unique persistent canonical correspondence is first assigned one unit of hard mass,
\begin{equation*}
\Pi_{ij}^{G,\mathrm{pers}}
=
1.
\end{equation*}
Only the remaining unmatched source and target primitives enter the same count-aware OT problem. Persistent and residual correspondence therefore operate on identical mass scales. Split and merge primitives, whose one-to-one persistent identities are invalidated by construction, are handled entirely through residual OT.

The augmented cost matrix assigns cost $1.2$ to every real-source--target-dustbin and source-dustbin--real-target entry, while the dustbin--dustbin corner has cost zero. The dustbin cost is fixed and is not learned. If one residual side is empty, each remaining real primitive on the opposite side is assigned one unit of mass directly to or from the corresponding dustbin.

The residual plan $\Pi^{G,\mathrm{res}}$ is obtained from the Sinkhorn problem in \cref{eq:residual-ot}, using entropy parameter $\varepsilon_{\mathrm{OT}}=0.08$, at most $200$ Sinkhorn iterations, and marginal tolerance $10^{-7}$.

\subsubsection{Metrics and uncertainty}

The primary lineage metric is LineageMass in \cref{eq:lineage-mass}. In the controlled implementation,
\[
\varepsilon_{\mathrm{match}}
=
10^{-12}.
\]
We additionally evaluate the geometry discrepancy induced by the real--real component of the correspondence plan:
\begin{equation*}
D_{\mathrm{corr}}
=
\frac{
\sum_{i,j}
\Pi^G_{ij}
\|\mu_i-\mu_j^{+}\|_2
}{
\sum_{i,j}\Pi^G_{ij}
+
\varepsilon_{\mathrm{match}}
},
\end{equation*}
where dustbin assignments are excluded. Lower $D_{\mathrm{corr}}$ indicates that correspondence mass is concentrated on geometrically compatible source--target pairs.

Birth and death are scored separately using precision, recall, and F1. A source primitive is classified as dead when its target-dustbin mass is at least its total mass assigned to real targets. A target primitive is classified as newly born when the mass arriving from reserved source slots together with the source dustbin is at least the mass arriving from active source primitives. Event F1 is the NaN-aware mean of the applicable birth and death F1 scores; an event type absent from the ground truth is treated as not applicable rather than assigned zero.

The main results in \cref{tab:small-identity-results} report mean $\pm$ standard deviation across the three Monte Carlo repetition-level aggregates. We additionally compute $95\%$ bootstrap confidence intervals using $2000$ resamples over repetition--trial clusters. A cluster corresponds to one Monte Carlo repetition and trial index, preserving the common replicate structure while aggregating the relevant perturbation conditions. The $95\%$ bootstrap intervals for split/merge LineageMass are $[0.97916,0.98281]$ for OT only and $[0.98540,0.98889]$ for Hybrid. The corresponding intervals for aggregate geometry discrepancy are $[0.04977,0.05043]$ and $[0.04871,0.04940]$, respectively. These separated intervals are consistent with the small systematic Hybrid advantage in this controlled benchmark; they are not presented as a separate formal hypothesis test.

\subsubsection{Mixed perturbations}

\begin{table}[t]
\caption{Mixed permutation, split, merge, birth, and death perturbations as severity increases, averaged over 3 independent Monte Carlo repetitions.}
\label{tab:app-exp3-mixed}
\centering
\tiny
\setlength{\tabcolsep}{3pt}
\begin{tabular}{c l cc}
\toprule
Severity
& Method
& LineageMass $\uparrow$
& Geometry discrepancy $\downarrow$ \\
\midrule
\multirow{2}{*}{$0.10$}
& OT only
& $0.95694{\pm}0.00752$
& $0.05693{\pm}0.00157$
\\
& Hybrid
& $\mathbf{0.96735{\pm}0.00620}$
& $\mathbf{0.05541{\pm}0.00108}$
\\
\midrule
\multirow{2}{*}{$0.25$}
& OT only
& $0.95148{\pm}0.00091$
& $0.06502{\pm}0.00171$
\\
& Hybrid
& $\mathbf{0.95988{\pm}0.00082}$
& $\mathbf{0.06349{\pm}0.00147}$
\\
\midrule
\multirow{2}{*}{$0.40$}
& OT only
& $0.91962{\pm}0.00344$
& $0.07718{\pm}0.00183$
\\
& Hybrid
& $\mathbf{0.92605{\pm}0.00137}$
& $\mathbf{0.07603{\pm}0.00158}$
\\
\bottomrule
\end{tabular}
\end{table}

\Cref{tab:app-exp3-mixed} shows the same qualitative behavior at every evaluated severity: Hybrid retains more correspondence mass on valid lineage and produces lower induced geometry discrepancy than OT only. The margin decreases as severity increases because progressively more persistent canonical identities become invalid and a larger fraction of the correspondence problem is delegated to residual OT.

Birth/death identification itself is not the source of the Hybrid advantage. For isolated birth and death perturbations, both OT only and Hybrid attain F1 $=1$. Under mixed perturbations, averaged across severities, OT only obtains macro event F1 $0.8984$ and Hybrid obtains $0.8884$. The Hybrid advantage is therefore concentrated in lineage preservation and induced geometry correspondence rather than in uniformly improving every auxiliary event metric.

\subsubsection{Persistent-ID dropout}
\label{app:exp3-id-dropout}

To test whether Hybrid depends critically on canonical identity, we hide a fraction of otherwise valid persistent target IDs under mixed perturbations at severity $0.25$. Gaussian geometry and ground-truth lineage remain unchanged; only the identity information available to Hybrid is removed. OT only is therefore unchanged across dropout rates.

\begin{table}[t]
\caption{Persistent-ID dropout under mixed perturbations at severity $0.25$.}
\label{tab:app-exp3-id-dropout}
\centering
\small
\setlength{\tabcolsep}{4pt}
\begin{tabular}{ccccc}
\toprule
Hidden IDs
& OT lineage
& Hybrid lineage
& OT geometry
& Hybrid geometry \\
\midrule
$0$
& $0.95148$
& $\mathbf{0.95988}$
& $0.06502$
& $\mathbf{0.06349}$
\\
$0.10$
& $0.95148$
& $\mathbf{0.95916}$
& $0.06502$
& $\mathbf{0.06361}$
\\
$0.25$
& $0.95148$
& $\mathbf{0.95838}$
& $0.06502$
& $\mathbf{0.06376}$
\\
$0.50$
& $0.95148$
& $\mathbf{0.95649}$
& $0.06502$
& $\mathbf{0.06409}$
\\
\bottomrule
\end{tabular}
\end{table}

As shown in \cref{tab:app-exp3-id-dropout}, progressively hiding persistent canonical identity causes Hybrid to move smoothly toward the identity-agnostic OT-only solution rather than failing abruptly. This is the intended fallback behavior of the hybrid correspondence principle in \cref{eq:hybrid-identity-principle}.


\section{Additional Discussion and Limitations}
\label{app:additional-discussion}

\paragraph{Scope of the current evidence.}

The experiments in this paper are controlled mechanism-level tests rather than a comprehensive evaluation of dynamic-scene world modelling. Experiments~1 and~2 operate directly on procedurally generated Gaussian states, providing exact future geometry, motion-group transformations, and temporal decompositions while bypassing the difficulty of constructing Gaussian scenes from real observations. Experiment~3 similarly isolates correspondence using synthetic identity and topology perturbations with known lineage. The reported results therefore support the proposed temporal-composition, geometry-composition, and correspondence mechanisms, but do not establish large-scale superiority over reconstruction-based Gaussian world models or video-space predictive models.

\paragraph{Why latent prediction rather than full Gaussian generation?}

A future scene contains both predictable structure and details that may be ambiguous, unobserved, or irrelevant to future dynamics. A reconstruction objective can allocate substantial capacity to reproducing precise appearance or representation-specific details even when they are not necessary for prediction. \method{} instead predicts future representations and uses the selective geometry decoder $Q_\psi^\G$ in \cref{sec:geometry-decoder} to retain explicit geometric grounding.

This design does not imply that complete future Gaussian generation is undesirable. Rather, latent prediction and complete generation answer different questions. A generative Gaussian world model aims to produce a complete future scene, whereas \method{} learns a predictive state optimized to retain information relevant to future evolution.

\paragraph{What temporal composition does and does not guarantee.}

The temporal-composition loss $\mathcal L_{\mathrm{comp}}$ in \cref{eq:composition-loss} constrains alternative temporal routes to produce compatible latent states, while the geometry-composition loss $\mathcal L_{\mathrm{gcomp}}$ in \cref{eq:geometry-composition-loss} applies the analogous principle to structured decoded motion.

Neither objective establishes future correctness by itself. Two prediction routes may agree while both are incorrect. Target-relative supervision through $\mathcal L_{\mathrm{end}}$, $\mathcal L_{\mathrm{path}}$, and $\mathcal L_{\mathrm{geo}}$ in \cref{eq:endpoint-loss,eq:path-loss,eq:geometry-loss} remains necessary. Experiments~1 and~2 therefore evaluate composition jointly with target-relative prediction accuracy.

Temporal composition also does not imply an exact semigroup law for arbitrary dynamic scenes. Actions, absolute time, contacts, topology changes, and other non-autonomous effects may alter the relevant transition. The appropriate requirement is consistency between admissible chronological decompositions of the same evolution, as formalized by the controlled composition law in \cref{eq:controlled-composition}.

\paragraph{Dependence on the Gaussian front-end.}

The general formulation assumes that a sufficiently stable Gaussian representation can be constructed causally from observations through the online builder $B_{\mathrm{on}}$ in \cref{eq:causal-context-builder}. Errors in calibration, depth estimation, reconstruction, segmentation, motion grouping, tracking, or Gaussian initialization may therefore propagate into the predictive representation.

The present controlled experiments intentionally remove this source of error by supplying Gaussian states directly. Consequently, they test the predictive model independently of reconstruction quality but do not establish robustness to errors produced by a real online Gaussian front-end. Evaluation on causally reconstructed real dynamic scenes is required to characterize this interaction.

\paragraph{Deterministic future prediction.}

The transition $\Phi_\phi$ in \cref{eq:transition} produces a single predictive future for a fixed context, horizon, and action sequence. This is appropriate for the approximately deterministic controlled trajectories evaluated here, but many real environments admit multiple plausible futures because of unobserved forces, agent decisions, or stochastic interactions.

In such settings, a deterministic transition may average incompatible outcomes or select one trajectory without representing its uncertainty. Extending the Gaussian-world transition to probabilistic or multimodal latent prediction is therefore a natural extension of the present formulation.

\paragraph{Gaussian identity and topology change.}

Persistent canonical identity provides reliable correspondence when the same canonical slot can be tracked through time. Experiment~3 shows, however, that raw array position is not a valid substitute for persistent canonical identity and that split, merge, birth, and death events require a more flexible correspondence mechanism.

The residual OT formulation in \cref{eq:residual-ot} provides this fallback, but correspondence can remain intrinsically ambiguous when several primitives have nearly identical geometry, motion-group membership, and semantic features. Furthermore, the prefix-anchored canonical representation described in \cref{sec:identity} contains a finite pool of reserved inactive slots. Sufficiently large amounts of newly represented geometry could exhaust this capacity and require expansion or reconstruction of the canonical bank.

The present experiments also do not evaluate residual OT inside end-to-end predictive training: Experiments~1 and~2 use fixed persistent correspondence, whereas Experiment~3 studies correspondence independently of model optimization. The interaction between learned predictive states and residual correspondence under real topology-changing dynamics therefore remains to be evaluated.

\paragraph{Non-uniqueness of structured motion decomposition.}

The selective decoder represents future geometry through a combination of motion-group $\SE(3)$ transformations and Gaussian-specific residual deformation. This decomposition is generally not identifiable from endpoint geometry alone: different combinations of group motion and local residual deformation can produce similar final Gaussian states.

Experiment~2 demonstrates that the motion-group component of $\mathcal L_{\mathrm{gcomp}}$ substantially reduces path dependence in the predicted transformations, while the Gaussian-level component improves primitive-level consistency. These constraints encourage a common structured decomposition across temporal paths, but they do not imply that the resulting decomposition is uniquely or physically correct.

\paragraph{Structural priors.}

The rigidity, local smoothness, and static-region terms in \cref{eq:rigid-loss,eq:smooth-loss,eq:static-loss} encode useful geometric biases but are not universal physical laws. Strong articulation, tearing, fluids, granular motion, abrupt contacts, or major topology changes may violate them. Such terms should therefore be interpreted as optional soft priors derived from causal evidence rather than hard physical constraints.

The present three experiments do not separately evaluate these structural regularizers. Their empirical contribution should therefore not be inferred from the current results.

\paragraph{Appearance and rendering.}

The present experiments deliberately focus on Gaussian geometry rather than appearance reconstruction. Experiments~1 and~2 do not contain camera projection, rasterization, RGB supervision, PSNR, or perceptual image metrics, as described in \cref{app:training-details}. Consequently, they do not establish how much future view-dependent appearance is retained by the learned representation or how accurately predicted Gaussian states can render future observations.

This omission is intentional in the current mechanism-level evaluation: it isolates predictive geometry and temporal composition from rendering quality. A broader evaluation can combine the predictive state with the Gaussian rendering interface reviewed in \cref{app:gaussian-splatting}.

\paragraph{Computational scaling.}

The controlled experiments use small Gaussian sets and known motion groups. Real Gaussian reconstructions may contain orders of magnitude more primitives. Although hierarchical grouping and persistent identity reduce the need for unrestricted primitive-wise comparison, attention over large Gaussian sets and residual OT matching can still become computationally expensive.

The present experiments do not contain a scaling study, so no empirical claim about large-scene computational efficiency is made. Practical large-scale implementations may require Gaussian-token subsampling, local or sparse attention, spatially restricted residual matching, or multiresolution grouping.

\paragraph{Action-conditioned dynamics and planning.}

The formulation in \cref{eq:transition,eq:controlled-composition} permits future actions to condition the transition, but the three experiments reported here are passive forecasting and correspondence experiments. We therefore make no empirical claim about action-conditioned prediction, control, or planning in the present evaluation.

Testing whether temporal composition improves long-horizon action-conditioned rollout and model-based planning is an important extension, but it is separate from the mechanism-level evidence established in this paper.

\end{document}